\documentclass[10pt]{amsart}

\usepackage{amsmath,amssymb,amsthm,amsrefs,a4wide}
\usepackage{latexsym}
\usepackage{indentfirst}
\usepackage{graphicx}
\usepackage{placeins}
\usepackage{enumerate}
\usepackage{booktabs}
\usepackage{algorithm}
\usepackage{algorithmic}
\usepackage{multirow}
\usepackage{subcaption}
\usepackage{optidef}
\usepackage{graphicx,color}
\usepackage{diagbox}
\usepackage[normalem]{ulem}

\theoremstyle{plain}
\newtheorem{theorem}{Theorem}
\newtheorem{lemma}[theorem]{Lemma}

\newtheorem{cond}{Condition}
\newtheorem{cor}[theorem]{Corollary}

\theoremstyle{definition}
\newtheorem{defn}[theorem]{Definition}

\theoremstyle{remark}

\newtheoremstyle{cited}%
  {3pt}
  {3pt}
  {\itshape}
  {}
  {\bfseries}
  {.}
  {.5em}
  {\thmname{#1} \thmnumber{#2} \thmnote{\normalfont#3}}

\theoremstyle{cited}

\DeclareMathOperator{\dist}{dist}

\renewcommand{\ell}{l}

\newcommand{\bbm}{\begin{bmatrix}}
\newcommand{\ebm}{\end{bmatrix}}

\newcommand{\KL}[2]{D_{\mathrm{KL}}\left(#1 \parallel #2\right)}

\newcommand{\child}{\mathcal{C}}
\newcommand{\neighbor}{\mathcal{N}}
\newcommand{\leftside}{\mathcal{L}}
\newcommand{\rightside}{\mathcal{R}}
\newcommand{\parent}{\mathcal{P}}

\newcommand{\hp}{\hat{p}}

\newcommand{\PTS}{\hat p_{\mathrm{TS}}}

\newcommand{\bR}{\mathbb{R}}

\usepackage{amsmath}
\usepackage{mathtools}
\usepackage{etoolbox}
\DeclareFontFamily{U}{matha}{\hyphenchar\font45}
\DeclareFontShape{U}{matha}{m}{n}{
      <5> <6> <7> <8> <9> <10> gen * matha
      <10.95> matha10 <12> <14.4> <17.28> <20.74> <24.88> matha12
      }{}
\DeclareSymbolFont{matha}{U}{matha}{m}{n}
\DeclareFontSubstitution{U}{matha}{m}{n}

\DeclareFontFamily{U}{mathx}{\hyphenchar\font45}
\DeclareFontShape{U}{mathx}{m}{n}{
      <5> <6> <7> <8> <9> <10>
      <10.95> <12> <14.4> <17.28> <20.74> <24.88>
      mathx10
      }{}
\DeclareSymbolFont{mathx}{U}{mathx}{m}{n}
\DeclareFontSubstitution{U}{mathx}{m}{n}

\DeclareMathDelimiter{\vvvert}{0}{matha}{"7E}{mathx}{"17}
\DeclarePairedDelimiterX{\normi}[1]
  {\vvvert}
  {\vvvert}
  {\ifblank{#1}{\:\cdot\:}{#1}}

\title{Generative Modeling via Tree Tensor Network States}

\author[X. Tang]{Xun Tang}
\author[Y. Hur]{YoonHaeng Hur}
\author[Y. Khoo]{Yuehaw Khoo}
\author[L. Ying]{Lexing Ying}

\address{Institute for Computational and Mathematical Engineering, Stanford, CA 94305, USA. }
\email{xuntang@stanford.edu}

\address{Department of Statistics, Chicago, IL 60637, USA. }
\email{yoonhaenghur@uchicago.edu}

\address{Computational and Applied Mathematics Initiative, Department of Statistics, University of Chicago, IL 60637, USA }
\email{ykhoo@uchicago.edu}

\address{Department of Mathematics and Institute for Computational and Mathematical Engineering, Stanford, CA 94305, USA. }
\email{lexing@stanford.edu}

\begin{document}

\maketitle

\begin{abstract}
In this paper, we present a density estimation framework based on tree tensor-network states. The proposed method consists of determining the tree topology with Chow-Liu algorithm, and obtaining a linear system  of equations that defines the tensor-network components via sketching techniques. Novel choices of sketch functions are developed in order to consider graphical models that contain loops. Sample complexity guarantees are provided and further corroborated by numerical experiments.
\end{abstract}
\section{Introduction}

Generative modeling of a probability distribution is one of the most important tasks in machine learning, engineering and science. In a nutshell, the goal of generative modeling is to approximate a high-dimensional distribution without the curse of dimensionality. There are generally several properties one would like to have for a generative model: (1) Can it be stored with a low memory complexity as the dimension grows? (2) Can it be determined from the given input with a low computational complexity? (3) Can it be used to generate samples with a low computational complexity? In this paper, we propose using a tree tensor network states as a generative model that enjoys these properties. 

We focus on the problem of \emph{density estimation}. More precisely, given $N$ independent samples \[
(y_1^{(1)}, \ldots, y_d^{(1)}), \ \ldots \ , (y_1^{(N)}, \ldots , y_d^{(N)}) \sim p^\star
\]
drawn from some ground truth density $p^\star \colon \mathbb{R}^d\rightarrow \mathbb{R}$, our goal is to estimate $p^\star$ from the empirical distribution
\begin{equation}
    \hp(x_1,\ldots,x_d) = \frac{1}{N} \sum_{i = 1}^{N} \delta_{(y_1^{(i)}, \ldots, y_d^{(i)})}(x_1,\ldots,x_d),
\end{equation}
where $\delta_{(y_1,\ldots,y_d)}$ is the $\delta$-measure supported on $(y_1,\ldots,y_d) \in \mathbb{R}^d$. It is hard to give a comprehensive survey of the broad field of density estimation, for this we refer readers to \cite{silverman2018density}. Here we review several popular generative models that are related to our work. Energy based model \cites{hinton2002training, lecun2006tutorial, song2019generative} such as graphical models represents a density by parameterizing it as the Gibbs measure of some energy function. Mixture models approximate the distribution via a composition of  simple distributions. On the other hand, deep learning methods based on generative adversarial networks \cite{goodfellow2014generative}, variational auto-encoders \cite{kingma2013auto} and normalizing flows \cites{tabak2010density,rezende2015variational} have gained tremendous popularity recently. Generally, obtaining the parameter of these parameterizations in a density estimation setting involves solving optimization problems that are often non-convex. Therefore frequently theoretical consistency guarantees of a density estimator cannot be achieved in practice. Furthermore, generating new samples from optimized model could be difficult (for example in energy based model) and requires running a Markov Chain Monte-Carlo.

Very recently, tensor-network methods, in particular matrix product state/tensor train, have emerged as an alternative paradigm for generative modeling \cites{glasser2019expressive, 2202.11788, han2018unsupervised, bradley2020modeling}. Such methods represent the exponentially size $d$-dimensional tensor as a network of $d$ small tensors, achieving polynomial storage complexity in $d$. Moreover, for networks that can be contracted easily (e.g. tensor train), there exists an efficient strategy based on conditional distribution method to generate independently and identically distributed samples \cite{dolgov2020approximation}. The question is then whether one can determine the underlying tensor-network efficiently for the task of density estimation. In \cites{glasser2019expressive, han2018unsupervised, bradley2020modeling}, non-convex optimization approaches are applied to determine the tensor cores. Unlike these previous approaches, in \cite{2202.11788}, sketching is used to set up a set of parallel core determining system of equations to determine the tensor cores of a tensor train without the use of optimization. We propose several extensions that generalizes this work in terms of its practicality and reach. 

We emphasize that there is another line of works in tensor literature that constructs low-rank tensor representations from sensing the entries of an order $d$-tensor. These includes matrix completion \cite{candes2010matrix} and its generalizations to tensor completion problem (e.g. \cites{khoo2017efficient, gandy2011tensor, richard2014statistical}). Furthermore, cross-approximations \cite{oseledets2011tensor} has been applied to the cases where one gets to choose the sensing pattern. The input data considered in these works are partial observation of the entries of the $d$-dimensional function, which is different from the case of density estimation, where empirical samples of the underlying distribution is given. 

We now give a discussion which compares our method with other existing generative modelling methods. In Section \ref{sec: tt to ttns}, we compare TTNS with other potential tensor network architectures for generative modelling. In Section \ref{sec: chow-liu to ttns}, we discuss the connection between TTNS and tree-based graphical models. In Section \ref{sec: iterative method}, we give a discussion on tensor network generative models which are based on iterative training.

\subsection{Extending tensor train to tree-based tensor networks}\label{sec: tt to ttns}
For the case where the underlying density \(p^\star\) has a tensor train (TT) format, an algorithm termed Tensor Train via Recursive Sketching (TT-RS) \cite{2202.11788} has been introduced. In this paper, we start from the more general model assumption that \(p^\star\) is in a Tree Tensor Network States (TTNS) format \cite{nakatani2013efficient}, which is more suitable when the random variables have a natural graph structure that is either a tree or locally tree-like.
In terms of representation power, there is a natural hierarchy \[\mathrm{TT} \subset \mathrm{TTNS} \subset \mathrm{TNS},\]
where TNS stands for Tensor Network States, which includes more general models such as projected entangled-pair states (PEPS) \cite{verstraete2006criticality}. The generalization from TT to TTNS is an extension from a path-based tensor network to a tree-based tensor network. TTNS is a generalized model which still allows for efficient and scalable tensor contraction, a task in general intractable for the TNS ansatz. See Figure \ref{fig:tree+ttns} for an illustration of a tree and its corresponding TTNS tensor diagram.


In the current literature, there is also another notion of Tree Tensor Network (TTN) \cite{cheng2019tree}, which has been used for the task of generative modeling. In \cite{cheng2019tree}, the TTN model is obtained by solving a non-convex optimization problem, which is similar to \cite{han2018unsupervised}. In contrast, our algorithm allows the tensor components of a TTNS to be directly solved, which leads to sample complexity guarantees that are hard to obtain for iterative methods. Despite similarities, the same technique cannot easily extend to TTN. The reason is due to the structural difference between TTN and TTNS, which is that the latter does not involve internal nodes, i.e. each tensor component in a TTNS ansatz has exactly one physical index, as shown in Figure \ref{fig:tree+ttns}. In conclusion, TTN and TTNS are under quite different model assumptions for the underlying density \(p^{\star}\), and model inference of the two models involve quite different numerical tools.

\subsection{Extending model inference of tree-based graphical models to tree-based tensor networks}\label{sec: chow-liu to ttns}

Model inference for distributions with a TTNS ansatz is deeply related to the model inference problem of tree-based graphical model. The Chow-Liu algorithm \cite{chow1968approximating} efficiently compresses a target density to the best tree-based graphical model in the sense of Maximum Likelihood Estimation (MLE).
In terms of representational power, if the underlying distribution \(p^\star\) is a tree-based graphical model, then \(p^\star\) is guaranteed to have a tractable TTNS representation. On the other hand, the extra representation power of TTNS over tree-based graphical models allows the model to account for longer range interactions between variables. To the best of our knowledge, this paper is the first instance where model inference of TTNS ansatz has been implemented.

After the model inference step, in terms of downstream tasks such as likelihood computation and sampling, both TTNS and tree-based graphical models scales linearly in the dimension \(d\). Importantly, the samples produced from the TTNS ansatz have no auto-correlations and are i.i.d., which is a desirable property for generative modeling.

We call our main method Tree Tensor Network State via Sketching (TTNS-Sketch). 
By the sketching technique \cite{woodruff2014sketching}, the tensor components of the ansatz can be computed entirely with conventional linear algebraic equations, which results in a sample complexity which is quadratic in the dimension \(d\). In terms of computational complexity scaling, the cost is linear in the sample size \(N\), and at most quadratic in \(d\). The method is proven to be a consistent estimator under reasonable technical assumptions.

\subsection{Comparison between TTNS-Sketch and iterative algorithms}\label{sec: iterative method}


In contrast to a direct method such as TTNS-Sketch, iterative algorithm for tensor network methods (e.g. \cites{han2018unsupervised, glasser2019expressive,cheng2019tree}) typically optimizes for the tensor components in the sense of minimal negative log likelihood, with the the maximum likelihood estimation (MLE) estimator as the optimizer. Despite the well-known Cramer-Rao bound for MLE estimators, the training in such iterative methods is non-convex, which prevents one from establishing theoretical guarantees, e.g. consistency and sample complexity. For example, \cite{mcclean2018barren} identifies the issue of vanishing gradient in randomly initialized quantum circuits for large qubit size. Due to the representational equivalence between tensor train and quantum circuits (see \cite{glasser2019expressive} for a discussion), the same issue also faces randomly initialized tensor train. This observation is also corroborated by numerical experiment in Section \ref{sec: numerical result}, where we show training failure of iterative methods under a setting far more modest than discussed in \cite{mcclean2018barren}.

\subsection{Main contribution}\label{sec: main contribution}
We list our main contribution as follows:
\begin{enumerate}
    \item We provide a simple notational system that can work well for arbitrary tree structures for TTNS ansatz. This structural flexibility is helpful for samples with an underlying tree structure but no practical path structure.
    \item We introduce perturbative sketching, motivated by randomized SVD \cite{lin2011fast}. We show that TTNS-Sketch with perturbative sketching performs well for models with short-range non-local interactions, thus exhibiting significant improvement over the graphical model given by Chow-Liu, which is only suitable for tree-based graphical models.
    \item We derive a general upper bound on sample complexity of TTNS-Sketch. Based on the Wedin theorem and Matrix Bernstein inequality, we obtain a non-asymptotic sample complexity upper bound for TTNS-Sketch under recursive sketching functions. Up to log factors and condition numbers, the sample complexity of the method scales by \(N = O\left(\Delta(T)^2d^2\right)\), where \(d\) is the number of nodes in the tree structure \(T\), and \(\Delta(T)\) stands for the maximal degree of \(T\). This shows that TTNS-Sketch converges reasonably fast to the true model.
    \item We also identify a failure mode for iterative generative modeling methods based on the tensor train ansatz. For Born Machine (BM) \cite{han2018unsupervised} under periodic spin system, we show that the training will fail unless one significantly increases the internal bond dimension. When the tensor train is of a correctly-sized internal bond dimension, the model learned by BM closely resembles that of a non-periodic spin system. In comparison, with a simple high-order Markov sketching function, TTNS-Sketch is successful at converging to \(p^{\star}\) without over-parameterization. See Section \ref{sec: numerical result} for detail. \label{enumerate: BM}

\end{enumerate}




\subsection{Outline}
The outline for the rest of this paper is as follows. Section \ref{sec: ttns ansatz} is an introduction to notations and to the basics of TTNS ansatz. Section \ref{sec: ttns-sketch full explanation} derives the essential linear equation to be used for TTNS-Sketch. Section \ref{sec: topology finding} introduces the Chow-Liu algorithm for finding a tree structure using samples. Section \ref{sec: main algorithm} provides the main Algorithm and the condition for TTNS-Sketch to be a consistent estimator. Section \ref{sec: sketch functions} gives examples of sketch functions. Section \ref{sec:sample-comlexity} provides a sample complexity upper bound. Section \ref{sec: numerical result} gives the numerical result.

\section{Introduction to TTNS}\label{sec: ttns ansatz}


The aim in this section is to introduce the notation to describe a function with a TTNS ansatz. We first introduce some important notations frequently used. The letter \(d\) is reserved for the dimension of the joint distribution of interest, \(N\) is reserved for sample size, and \(T\) without subscript is reserved for a tree graph. For any integer \(q \in \mathbb{N}\), set \([q]  :=  \{1, \ldots, q\}\). 

For the TTNS ansatz, we use specific letters to label its indices. The letter \(x\) is reserved for the physical index (external bond) of the tensor core, and the letters \(\alpha, \beta, \gamma\) are reserved for the internal bond of the tensor core.

Crucial equations are illustrated with a tensor diagram representation for reader's convenience. To provide a concrete example, all tensor diagrams are plotted based on the specific tree structure set in Figure \ref{fig:tree+ttns}a. 
\begin{figure}
  \centering
  \begin{subfigure}[A]{\textwidth}
  	  \centering
  	\includegraphics[width=0.48\textwidth]{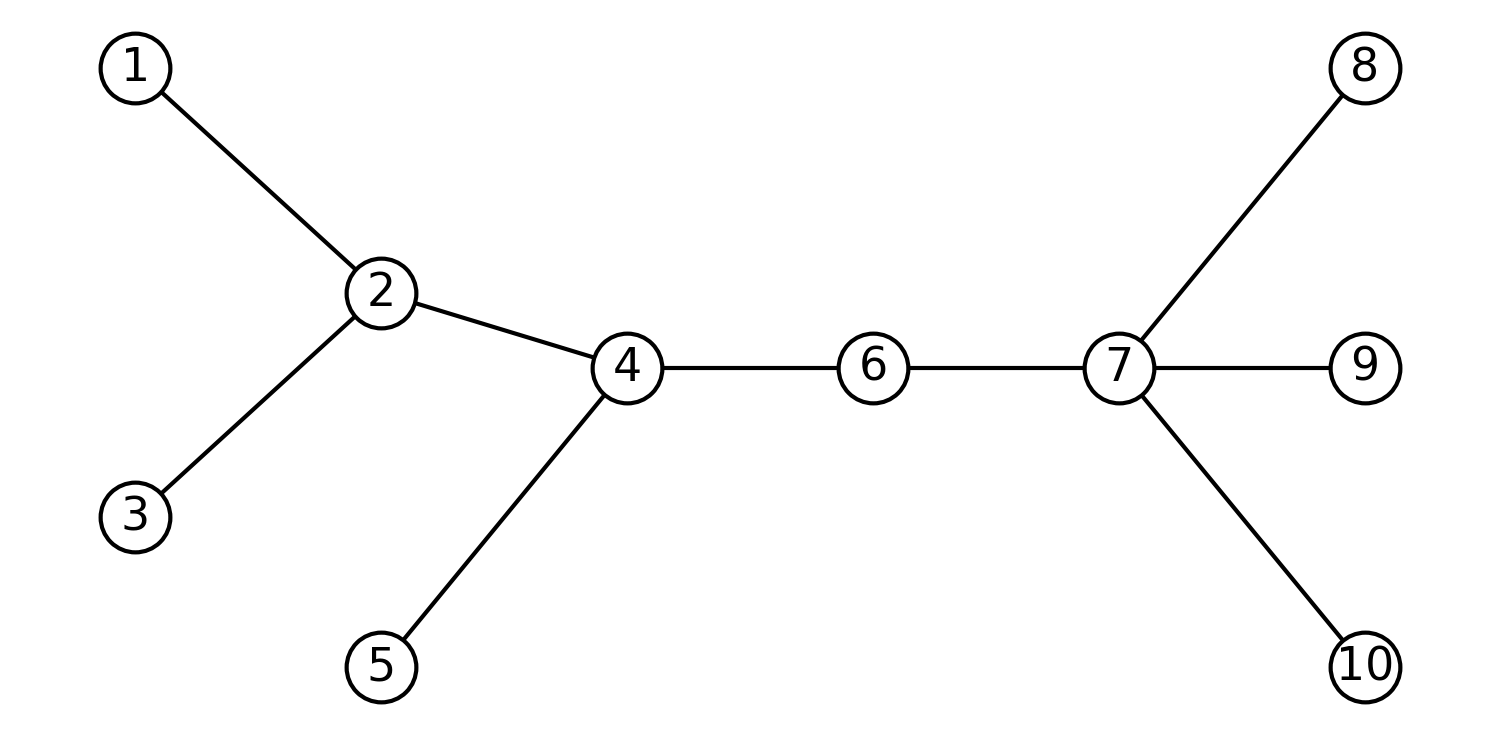}
  \end{subfigure}
  \hspace{12pt}
  \begin{subfigure}[B]{\textwidth}
  	  \centering
  	\includegraphics[width=0.48\textwidth]{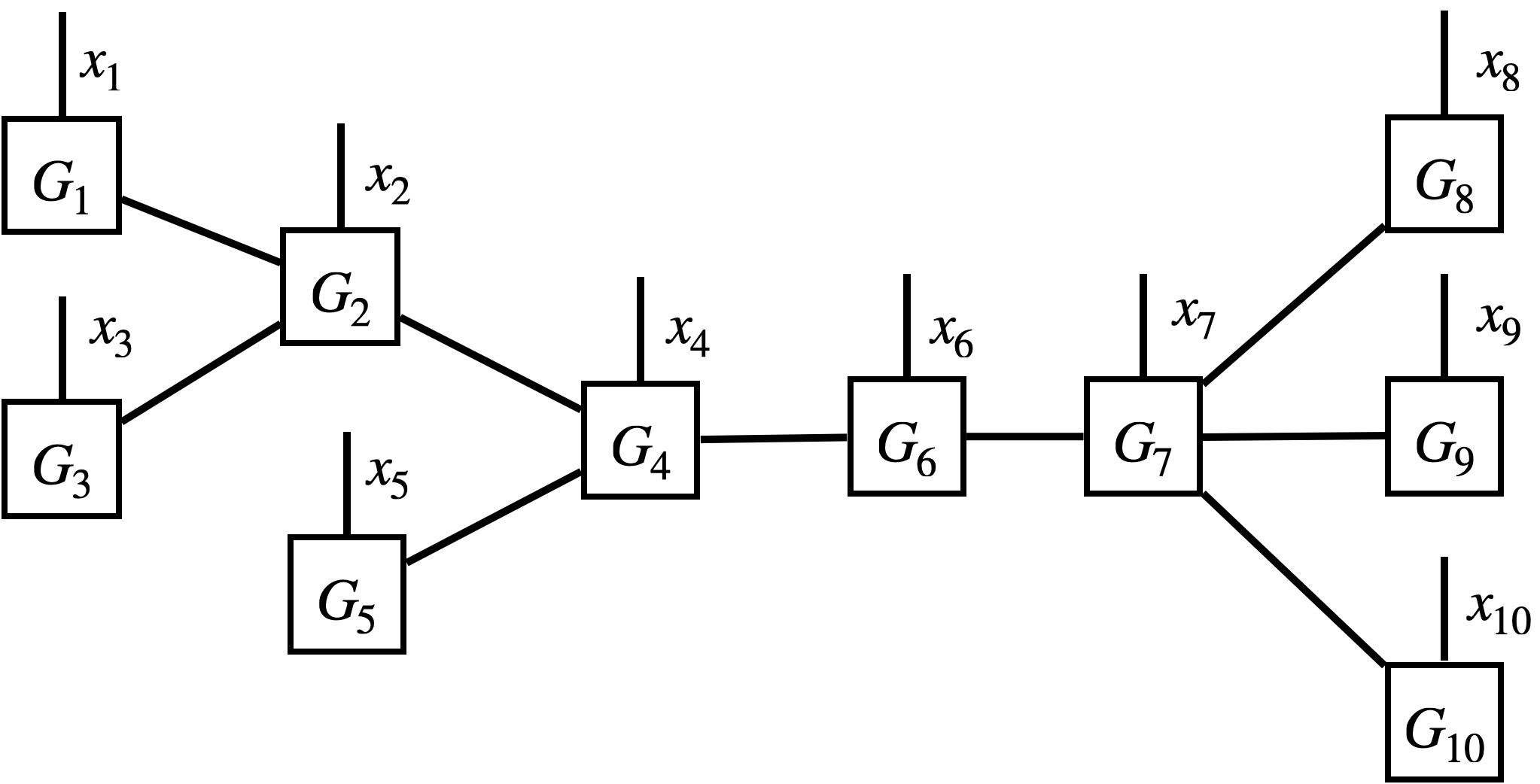}
  \end{subfigure}
 \caption{(A) A tree structure \(T = (V, E)\) with \(V = \{1, \ldots, 10\}\). (B) Tensor Diagram representation of TTNS over \(T\).}
 \label{fig:tree+ttns}
\end{figure}

\subsection{Notation for distribution}

First we introduce notations related to probabilistic distributions. 

\begin{defn}\label{def: distribution notation} (Probability distribution notation)
    Fix a generic joint distribution \(p\) on \(d\) variables. We use \(X  :=  \left(X_{1}, \ldots, X_{d}\right)\), where \(X \sim p\), to denote a random vector in \(d\) dimension. Each \(X_{k}\) is assumed to be a discrete random variable over \(\{1, \ldots, n_{k}\}\). Set \(n  :=  \max_{k \in [d]}n_k\).
\end{defn}

\begin{defn}\label{def: ttns distribution notation} (Probability distribution for TTNS-Sketch)
    Suppose the distribution of interest is for the random vector \(X  :=  \left(X_{1}, \ldots, X_{d}\right)\) in \(d\) dimension, and moreover suppose one is given samples of \(X\). The symbol \(\hp\) denotes the \emph{empirical distribution} over samples of \(X\).
    The symbol \(p^{\star}\) denotes the \emph{underlying distribution} of \(X\).
\end{defn}
  
For this paper, we only consider discrete variables. Hence a distribution function such as \(p^{\star}\) can be considered as a \(d\)-dimensional tensor.

\subsection{Notation for tree structure}

Next we introduce notations for a tree graph. A tree graph \(T = (V,E)\) is a connected undirected graph without cycles. Throughout this paper, $V = [d]$. Moreover \(T\) is specified with a root node, and vertices will have a partial topological ordering generated by the child-parent relationship. For an undirected edge \(\{w,k\}\) in \(T\), we write it interchangeably as $(w, k)$ or $(k, w)$. If $k$ is the parent of $w$, we also write \(\{w,k\}\) as $w \to k$ with the aim of signaling the child-parent hierarchy. 

The following Definition \ref{defn: tree topology information} contains the notation for the graph-theoretic concept one needs to define a TTNS. See Figure \ref{fig:tree_topology_notation_infograph} for an illustration.
\begin{defn}\label{defn: tree topology information}
(Tree topology notation)
For a rooted tree structure \(T\) with nodes $V = [d]$ and any \(k \in [d]\), define \(\child(k), \parent(k), \neighbor(k)\) respectively as the children, parent, and neighbors of \(k\). In particular, one has \(|\parent(k)| \le 1\) and \(\neighbor(k) = \child(k) \cup \parent(k)\). Moreover, define \(\mathcal{E}(k)\) as the set of edges incident to \(k\). Define \(\leftside(k), \rightside(k)\) respectively as the descendant, non-descendant of node \(k\) in \(T\). In particular, \(\leftside(k)\) and \(\rightside(k)\) are respectively called the \emph{left} and the \emph{right} of node \(k\).

\end{defn}

Importantly, in Definition \ref{def: compact shorthand}, we introduce several short-hands in order to write joint variables compactly.
\begin{defn}\label{def: compact shorthand}
    (Joint variable notation)
    For variables indexed by nodes on \(T\), we write $x_\mathcal{S}$ to denote the joint variable $(x_{i_1},\ldots,x_{i_k})$, where $\mathcal{S} = \{i_1 ,\ldots , i_k \} \subset V$. 
    
    Likewise, for variables indexed by edges on \(T\), we write \(\alpha_{\mathcal{U}}\) to denote the joint index $(\alpha_{e_{i_{1}}},\ldots,\alpha_{e_{i_{k}}})$, where \(\mathcal{U} = \{e_{i_{1}}, \ldots, e_{i_{k}}\} \subset E\).
    In particular, \(\mathcal{U} \subset E\) is typically all incident to one node \(w\), and we write $\alpha_{(w, \mathcal{S})}$ to denote the joint variable $(\alpha_{(w, i_1)},\ldots,\alpha_{(w, i_k)})$, where $\mathcal{S} = \{i_1 ,\ldots , i_k \} \subset V$. 
    For compactness, we write \(x_{\mathcal{S} \cup k}  :=  x_{\mathcal{S} \cup \{k\}}\), where the element \(k\) is used in place of the singleton set \(\{k\}\).
\end{defn}

Frequently used symbols include \(x_{\leftside(k)}\), \(x_{\rightside(k)}\), \(x_{\child(k)}\), which respectively denote the joint variable corresponding to the left, the right, and the children of \(k\). Moreover, we use \(x_{\leftside(k)\cup k}\) to denote the joint variable corresponding to nodes which are not on the right-side of \(k\). For edge-indexed variables, we use \(\alpha_{(k, \child(k))}\) to denote the joint variables corresponding to the edges between \(k\) and its children.

\begin{figure}
    \centering
    \includegraphics[width=0.48\textwidth]{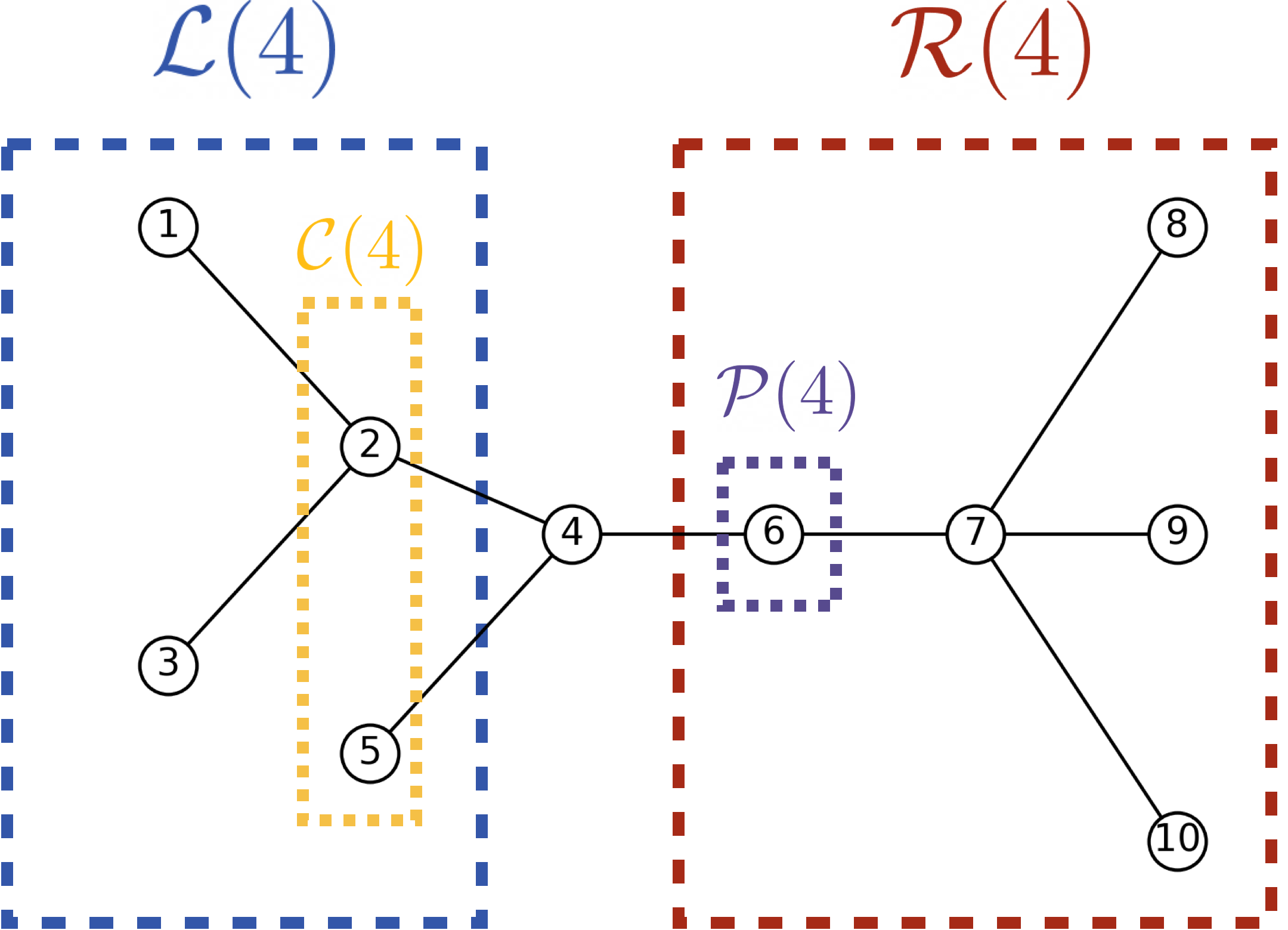}
    \caption{Illustration of tree topology notation. For the tree in Figure \ref{fig:tree+ttns}a, if 7 is the root, then $\child(4) = \{2, 5\}$, $\parent(4) = 6$, $\leftside(4) = \{1, 2, 3, 5\}$, and $\rightside(4) = \{6, 7, 8, 9, 10\}$. In this graph, one also has $\neighbor(4) = \{2, 5, 6\}$ and $\mathcal{E}(4) = \{(4, 2), (4, 5), (4, 6)\}$.}
    \label{fig:tree_topology_notation_infograph}
\end{figure}

\subsection{Notation for TTNS ansatz}
We introduce condition and notation for a generic tensor with TTNS ansatz. We will prove that having a TTNS ansatz is equivalent to satisfying the TTNS condition, i.e. having a low rank factorization structure along a tree. See Figure \ref{fig:ttns_condition_illustration} for an illustration.
\begin{cond}(TTNS condition)\label{cond: TTNS ansatz condition}
    Let \(T = (V, E)\) be a rooted tree graph, and let \(\{r_{e}\}_{e \in E}\) be a collection of positive integers, where \(r_{(w,k)}\) denotes internal bond rank at the edge \((w, k)\). 
    
    A function \(p \colon \prod_{k = 1}^{d}[n_{k}] \rightarrow \bR\) is said to satisfy the TTNS ansatz condition if for every edge \((w,k) \in E\), there exists a rank \(r_{(w,k)}\) decomposition $\Phi_{w \to k} \colon  \prod_{i \in \leftside(w) \cup w}  [n_i] \times  [r_{(w, k)}] \to \bR$ and $\Psi_{w \to k} \colon [r_{(w, k)}] \times \prod_{i \in \rightside(w)} [n_i] \to \bR$ such that 
    \begin{equation}
    \label{eqn: TTNS ansatz condition}
        p(x_{1}, \ldots, x_{d}) = \sum_{\alpha_{(w, k)} = 1}^{r_{(w,k)}} \Phi_{w \to k}(x_{\leftside(w) \cup w}, \alpha_{(w, k)}) \Psi_{w \to k}(\alpha_{(w, k)}, x_{\rightside(w)}).
    \end{equation}
\end{cond}
    
\begin{figure}
    \centering
    \includegraphics[width=0.38\textwidth]{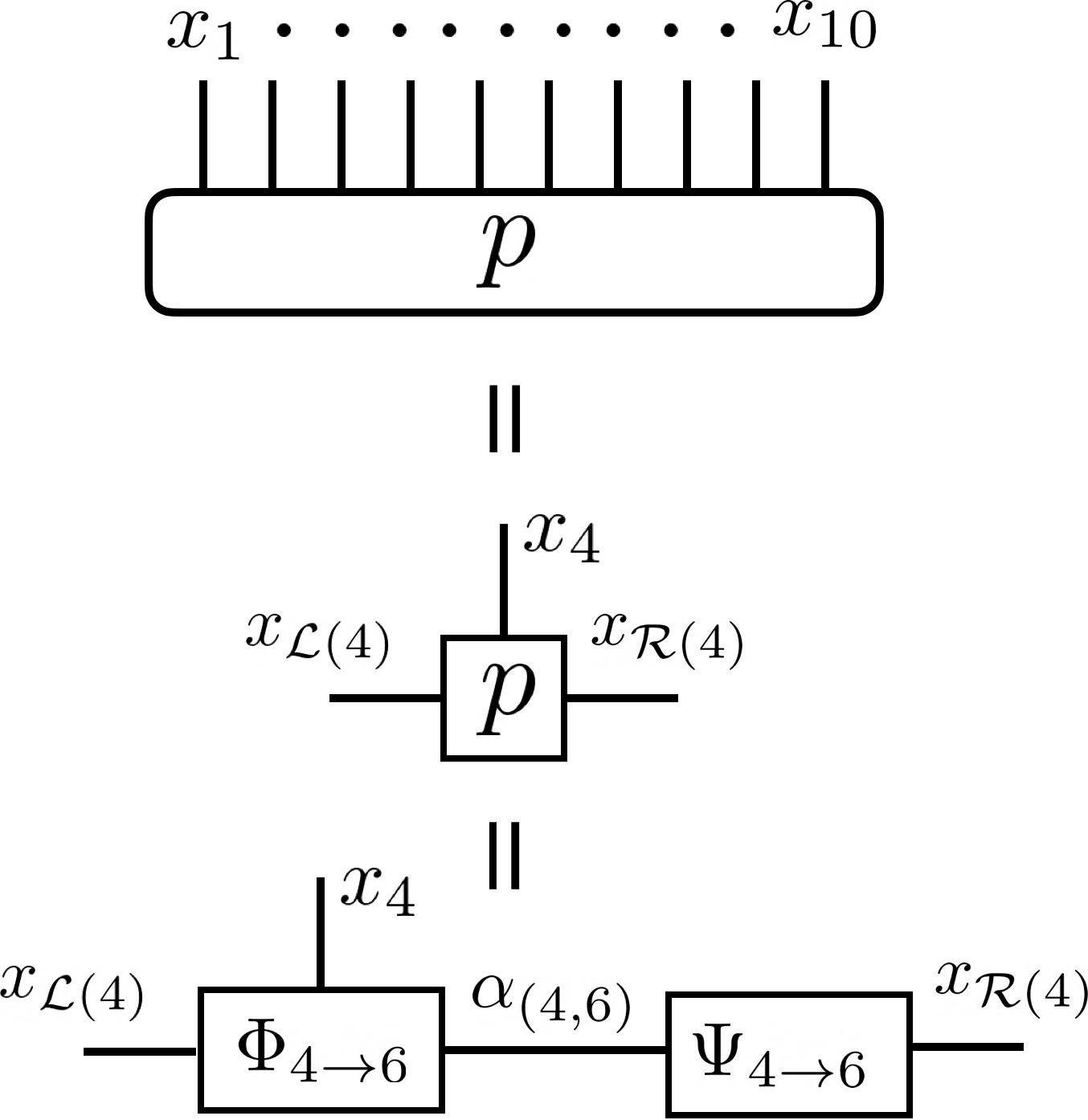}
    \caption{Tensor diagram representation of the TTNS ansatz assumption in Condition \ref{cond: TTNS ansatz condition}. The illustration is over the rooted tree in Figure \ref{fig:tree_topology_notation_infograph} with \((w,k) = (4,6)\). The symbol \(x_{\leftside(4)} = (x_{1}, x_{2}, x_{3}, x_{5})\) and \(x_{\rightside(4)} = (x_{6}, x_{7}, \ldots, x_{10})\) is a short-hand for joint variables as defined in Definition \ref{def: compact shorthand}.}
    \label{fig:ttns_condition_illustration}
\end{figure}

More explicitly, a TTNS ansatz can be defined in terms of tensor cores. Definition \ref{def: TTNS short def} shows a construction in terms of tensor cores. For illustration, see Figure \ref{fig:tree+ttns}b.
\begin{defn}\label{def: TTNS short def}
    (TTNS tensor core and TTNS ansatz notation)
    Given a tree structure $T = ([d], E)$ and corresponding ranks $\{r_e : e \in E\}$.
    The \emph{TTNS tensor core} at \(k\) is denoted by \(G_{k}\). Let \(d_{k}\) stand for the degree of \(k\) in \(T\), and then \(G_{k}\) is defined as an \((d_{k} + 1)\)-tensor of the following shape:
    \begin{equation*}\label{eqn: core size constraint}
        G_{k} : [n_{k}] \times \prod_{e \in \mathcal{E}(k)}  [r_e] \rightarrow \bR.
    \end{equation*}
    
    We say that a function $p$ admits a TTNS ansatz over tensor cores $\{G_i\}_{i=1}^{d}$ over \(k = 1,\ldots,d\) if
    \begin{equation}
        \label{eq:ttn-contraction}
        p(x_1, \ldots, x_d) = \sum_{\substack{\alpha_e \in [r_{e}] \\ e \in E }} \left(\prod_{k = 1}^{d} G_k\left(x_k, \alpha_{\left(k, \neighbor(k)\right)}\right)\right).
    \end{equation}

\end{defn}
For an example, when \(T = (V, E)\) is a chain with \(E = \{(k, k+1)\}_{k = 1}^{d-1}\), a TTNS ansatz is a tensor train ansatz. 
In Figure \ref{fig:tree+ttns}a, we show a tree structure over \(10\) vertices, and the corresponding tensor diagram for TTNS is given in Figure \ref{fig:tree+ttns}b. For instance, $G_4 \colon [n_4] \times [r_{(4, 2)}] \times [r_{(4, 5)}] \times [r_{(4, 6)}] \to \bR$ in Figure \ref{fig:tree+ttns}b, and the tensor network defines a \(d\)-dimensional function after the contraction of internal bonds.

Importantly, when working with high-dimensional functions, it is often convenient to group the variables into two subsets and think of the resulting object as a matrix. The notion is referred to as unfolding matrix, and defined as follows: 
\begin{defn}
    \label{def: unfolding matrix}
    (Unfolding matrix notation)
    For a generic \(D\)-dimensional tensor $f \colon [n_1] \times \cdots \times [n_D] \to \bR$ and for two disjoint subsets \(\mathcal{U}, \mathcal{V}\) with $\mathcal{U} \cup \mathcal{V} = [D]$, we define the corresponding \emph{unfolding matrix} by $f(x_{\mathcal{U}}; x_{\mathcal{V}})$. Namely, group the variables indexed by \(\mathcal{U}\) and the ones indexed by \(\mathcal{V}\) to form rows and columns, respectively. The matrix $f(x_{\mathcal{U}}; x_{\mathcal{V}})$ is of size \(\left(\prod_{i \in \mathcal{U}}n_{i}\right) \times \left(\prod_{j \in \mathcal{V}}n_{j}\right)\).
\end{defn}
As an example, for a function \(p\) satisfying TTNS assumption in Condition \ref{cond: TTNS ansatz condition}, define the unfolding matrix of \(p\) at the edge $w \to k \in E$ as $p(x_{\leftside(w) \cup w}; x_{\rightside(w)})$, which is of size $\left(\prod_{i \in \leftside(w) \cup w} n_i\right) \times \left(\prod_{j \in \rightside(w)} n_j\right)$. Viewed in this context, Condition \ref{cond: TTNS ansatz condition} exactly means that the unfolding matrix of \(p\) at any edge $(w,k) \in E$ is a matrix of rank $r_{(w,k)}$.

\subsection{Equation for TTNS ansatz}

We now show that Condition \ref{cond: TTNS ansatz condition} implies the existence of a TTNS ansatz in the sense of Definition \ref{def: TTNS short def}. With the information for every \(\Phi_{w \to k}\) in Condition \ref{cond: TTNS ansatz condition}, there exists an equation for obtaining the TTNS tensor cores exactly. We summarize this result in Theorem \ref{thm:ttns-existence}, which shows that one can obtain cores of a TTNS by solving a recursive system of linear equations. See Figure \ref{fig:full_CDE_illustration} for an illustration.
\begin{theorem} 
    \label{thm:ttns-existence}
    Suppose Condition \ref{cond: TTNS ansatz condition} holds for a rooted tree structure $T = ([d], E)$ and bond information $\{r_{e}\}_{e \in E}$. For non-leaf $k$, define
    \begin{equation*}
        \Phi_{\child(k) \to k} = \bigotimes_{w \in \child(k)} \Phi_{w \to k},
    \end{equation*}
    and in terms of entries one has $\Phi_{\child(k) \to k} \colon \prod_{w \in \leftside(k)} [n_w] \times \prod_{w \in \child(k)} [r_{(w, k)}] \to \bR$, and
    \begin{equation}
    \label{eqn: phi C_k to k}
        \Phi_{\child(k) \to k}(x_{\leftside(k)}, \alpha_{(k, \child(k))}) = \prod_{w \in \child(k)}\Phi_{w \to k}(x_{\leftside(w)}, \alpha_{(k, w)}).
    \end{equation}
    
    Then $G_k \colon [n_k] \times \prod_{w \in \neighbor(k)} [r_{(w, k)}] \to \bR$ satisfies the following linear Core Determining Equations (CDE) for $k = 1, \ldots, d$:
    \begin{equation}
    \label{eq:ttn-determining}
        \begin{aligned}
            G_k & = \Phi_{k \to \parent(k)} \quad \text{if} ~ k ~ \text{is a leaf}, \\
            \sum_{\alpha_{(k, \child(k))}} \Phi_{\child(k) \to k}(x_{\leftside(k)}, \alpha_{(k, \child(k))}) G_k(x_k, \alpha_{(k, \neighbor(k))}) & = p(x_1, \ldots, x_d) \quad \text{if} ~ k ~ \text{is the root}, \\
            \sum_{\alpha_{(k, \child(k))}} \Phi_{\child(k) \to k}(x_{\leftside(k)}, \alpha_{(k, \child(k))}) G_k(x_k, \alpha_{(k, \neighbor(k))}) & = \Phi_{k \to \parent(k)}(x_{\leftside(k) \cup k}, \alpha_{(k, \parent(k))}) \quad \text{otherwise}.
        \end{aligned}
    \end{equation}
    
    Then, each equation of \eqref{eq:ttn-determining} has a unique solution, and \(p\) has a TTNS ansatz over the cores $\{G_i\}_{i=1}^{d}$ in the sense of Definition \ref{def: TTNS short def}.
\end{theorem}

The proof is deferred to the Appendix, but we will give a rough idea on why \(p\) admits a TTNS ansatz over $\{G_i\}_{i=1}^{d}$. Equation \eqref{eq:ttn-determining} for when \(k\) is not root essentially shows that each \(\Phi_{w \to k}\) can be represented by tensor contractions of cores in $\{G_i\}_{i \in \leftside(w)\cup w}$, and the proof is based on simple mathematical induction. From this observation, one can work with Equation \eqref{eq:ttn-determining} for when \(k\) is the root, and replace all of the \(\Phi_{w \to k}\) terms by $\{G_i\}_{i \not = \text{root}}$, and the obtained equation will be exactly \eqref{eq:ttn-contraction} in Definition \ref{def: TTNS short def}.

In summary, Theorem \ref{thm:ttns-existence} shows how Condition \ref{cond: TTNS ansatz condition} leads to existence of a TTNS ansatz, and the our previous remark on the construction of \(\Phi_{w \to k}\) from $\{G_i\}_{i \in \leftside(w)\cup w}$ also shows a TTNS ansatz also leads to Condition \ref{cond: TTNS ansatz condition}. However, from a computational point of view, the linear system \eqref{eq:ttn-determining} in Theorem \ref{thm:ttns-existence} is intractable and we shall address this issue using sketching in the next section.

\section{Main idea of TTNS-Sketch}\label{sec: ttns-sketch full explanation}
In the setting of this section, \(p^{\star}\) admits a TTNS ansatz over \(T\) and \(\{r_{e}\}_{e \in E}\) in the sense of Definition \ref{def: TTNS short def}. We show the derivation of the linear equation which are used to solve for the TTNS tensor cores in TTNS-Sketch. However, obtaining terms in the derived linear system rely on access to \(p^{\star}\), an assumption which we will later relax by sample approximation. To emphasize this point, all of the intermediate terms from this algorithm will be labelled with the upper-index \(\star\) if it assumes access to or is derived from \(p^{\star}\). 




\subsection{Gauge degree of freedom for a TTNS ansatz}\label{sec: gauge equivalence}


By Theorem \ref{thm:ttns-existence}, a function \(p^{\star}\) having a TTNS ansatz is equivalent to the condition that its unfolding matrix along each edge of a tree has a low rank structure. Moreover, the ansatz is non-unique. This notion is typically called the gauge degree of freedom, which we will introduce here.

Let us view \(p^{\star}\) by the unfolding matrix \(p^{\star}(x_{\leftside(w) \cup w}; x_{\rightside(w)})\).
For any edge $w \to k$, the TTNS condition assumes that there exists $\Phi^{\star}_{w \to k} \colon  \prod_{i \in \leftside(w) \cup w}  [n_i] \times [r_{(w, k)}]  \to \bR$ and $\Psi^{\star}_{w \to k} \colon [r_{(w, k)}] \times \prod_{i \in \rightside(w)} [n_i] \to \bR$ such that 
\begin{equation*}
    p^{\star}(x_{\leftside(w) \cup w}, x_{\rightside(w)}) = \sum_{\alpha_{(w, k)}} \Phi^{\star}_{w \to k}(x_{\leftside(w) \cup w}, \alpha_{(w, k)}) \Psi^{\star}_{w \to k}(\alpha_{(w, k)}, x_{\rightside(w)}).
\end{equation*}

One can view \(p^\star\) as the unfolding matrix structure \(p^{\star}(x_{\leftside(w) \cup w}; x_{\rightside(w)})\). Likewise, \(\Phi^{\star}_{w \to k}\) as \(\Phi^{\star}_{w \to k}(x_{\leftside(w) \cup w}; \alpha_{(w, k)})\) and \(\Psi^{\star}_{w \to k}\) as \(\Psi^{\star}_{w \to k}(\alpha_{(w, k)}; x_{\rightside(w)})\). 
Then, by using the usual matrix product notation, the TTNS assumption along the edge \(w \to k\) is \(p^{\star} = \Phi^{\star}_{w \to k}\Psi^{\star}_{w \to k}.\) Then, for any \(R\) being a nonsingular \(r_{(w, k)} \times r_{(w, k)}\) matrix, one has\[p^{\star} = \Phi^{\star}_{w \to k}\Psi^{\star}_{w \to k} = (\Phi^{\star}_{w \to k}R)(R^{-1}\Psi^{\star}_{w \to k}).\]

Given the information of \(\{ \Phi^{\star}_{w \to k}\}_{w \to k \in E}\), solving for the tensor core \(G^{\star}_{k}\) follows from \eqref{eq:ttn-determining} in Theorem \ref{thm:ttns-existence}. Multiplying any \(\Phi^{\star}_{w \to k}\) by a matrix \(R\) will thus result in a different TTNS ansatz for \(p^\star\). In summary, a gauge degree of freedom in the low-rank decomposition of \(p^\star\) leads to a gauge degree of freedom in the TTNS ansatz of \(p^\star\).

The collection \(\{ \Phi^{\star}_{w \to k}, \Psi^{\star}_{w \to k}\}_{w \to k \in E}\) will later be chosen to have an explicit gauge, but currently it suffices to understand gauge as fixed. The desired TTNS ansatz \(\{G^{\star}_{i}\}_{i= 1}^d\) as solution to \eqref{eq:ttn-determining} is also fixed.

\subsection{Sketching down core determining equation}

Without loss of generality, we consider the equation for \(G^{\star}_{k}\) in Theorem \ref{thm:ttns-existence} where \(k\) is neither a root nor a leaf node. We can rewrite the corresponding equation for \(G^{\star}_{k}\) by substituting \(\Phi^{\star}_{\child(k) \rightarrow k}\) according to definition:
\begin{equation}\label{eqn: unsketched equation for core}
\sum_{\substack{\alpha_{(w, k)} \\ w \in \child(k) }} \left(\prod_{w \in \child(k)}\Phi^{\star}_{w \to k}(x_{\leftside(w) \cup w}, \alpha_{(w, k)}) \right) G^{\star}_k(x_k, \alpha_{(k, \child(k))},  \alpha_{(k, \parent(k))})  = \Phi^{\star}_{k \to \parent(k)}(x_{\leftside(k) \cup k}, \alpha_{(k, \parent(k))}),
\end{equation}
which is an over-determined linear system on \(G_{k}^\star\), and the number of linear equations for \(G_{k}^\star\) grows exponentially in \(d\). Hence the above equation is not tractable.

The TTNS-Sketch algorithm applies the sketching operation to \eqref{eqn: unsketched equation for core} and projects tensors of the form \(\Phi_{w \to k}\) in \eqref{eqn: unsketched equation for core} to a tensor of tractable size, which makes the equation tractable. 
In TTNS-Sketch, for each edge \(w \to k\), we define a series of linear projection operators of the form \[S_{w \to k} \colon [l_{(w, k)}] \times \prod_{v \in \leftside(w) \cup w} [n_{v}] \rightarrow \mathbb{R},\]
and they globally form an function which we call the \emph{left-sketch function} \(S_{k}\) of the form
\[S_k \colon \prod_{w \in \child(k)} [l_{(w, k)}] \times \prod_{i \in \leftside(k)} [n_i] \to \bR.\]
The definition of \(S_{k}\) is by the simple formula \(S_{k} = \bigotimes_{w \in \child(k)}S_{w \to k}\), or equivalently 
\begin{equation}
  \label{eqn: S_k splits}
  S_{k}(\beta_{(k,\child(k))}, x_{\leftside(k)}) = \prod_{w \in \child(k)}S_{w \to k}(\beta_{(w,k)}, x_{\leftside(w) \cup w}).
\end{equation}

We remark that the factorization structure of \(S_{k}\) in \eqref{eqn: S_k splits} depends on a simple topological fact on trees, which is that \(\left(\leftside(w) \cup \{w\}\right)_{w \in \child(k)}\) are pairwise disjoint and their union forms \(\leftside(k)\).

Now, we can apply the usual projection to \eqref{eqn: unsketched equation for core} using \(S_{k}\), i.e. multiplying both sides by \(S_{k}\) and summing over \(x_{\leftside(k)}\). For the RHS of \eqref{eqn: unsketched equation for core} after projection, we define a tensor \(B^{\star}_{k}(\beta_{(k,\child(k))}, x_{k}, \alpha_{(k, \parent(k))})\) to represent this result, i.e.
\begin{equation}\label{eqn: def of B_k}
B^{\star}_{k}(\beta_{(k,\child(k))}, x_{k}, \alpha_{(k, \parent(k))})  :=  \sum_{x_{\leftside(k)}} S_{k}(\beta_{(k,\child(k))}, x_{\leftside(k)}) \Phi^{\star}_{k \to \parent(k)}(x_{\leftside(k) \cup k}, \alpha_{(k, \parent(k))}).
\end{equation}

For the LHS of of \eqref{eqn: unsketched equation for core}, we define a tensor \(A^{\star}_{k}(\beta_{(k, \child(k))}, \alpha_{(k, \child(k))})\) to represent the coefficient term for \(G^{\star}_{k}\) under projection:
\begin{equation}\label{eqn: def of A_k}
    A^{\star}_{k}(\beta_{(k, \child(k))}, \alpha_{(k, \child(k))})  :=  \sum_{x_{\leftside(k)}} S_{k}(\beta_{(k,\child(k))}, x_{\leftside(k)}) \prod_{w \in \child(k)}\Phi^{\star}_{w \to k}(x_{\leftside(w) \cup w}, \alpha_{(w, k)})
\end{equation}

The equation \eqref{eqn: unsketched equation for core} then projects to the linear equation:
\begin{equation}
\label{eqn: noiseless equation for G_k^star}
\sum_{\alpha_{(k, \child(k))}}A^{\star}_{k}(\beta_{(k, \child(k))}, \alpha_{(k, \child(k))})  G^{\star}_k(x_k, \alpha_{(k, \child(k))},  \alpha_{(k, \parent(k))})  = B^{\star}_{k}(\beta_{(k,\child(k))}, x_{k}, \alpha_{(k, \parent(k))}),
\end{equation}
which is an equation of the simple form of \(A^{\star}_{k}G^{\star}_{k} = B^{\star}_{k}\) when viewing each tensor by appropriate unfolding matrix structures. This linear equation is illustrated in Figure \ref{fig:sketched_down_CDE_illustration}.

In the sketched down linear system, the number of linear equations for \(G^{\star}_{k}\) no longer scales with \(d\), and one can check that it is tractable. Moreover, due to the factorization structure of \(S_{k}\) using \(S_{w \to k}\), it follows that \(A^\star_{k}\) simplifies to
\begin{equation}\label{eqn: def of A_k simplified}
    A^{\star}_{k}(\beta_{(k, \child(k))}, \alpha_{(k, \child(k))}) = 
     \prod_{w \in \child(k)} \sum_{x_{\leftside(w)\cup w}}S_{w \to k}(\beta_{(w,k)}, x_{\leftside(w) \cup w})  \Phi^{\star}_{w \to k}(x_{\leftside(w) \cup w}, \alpha_{(w, k)}),
\end{equation}
which can be readily seen from the diagram in Figure \ref{fig:sketched_down_CDE_illustration}.

\subsection{Derivation of \(A^{\star}_{k}\) and \(B^{\star}_{k}\) in TTNS-Sketch}

For the time being, \(A^{\star}_{k}\) and \(B^{\star}_{k}\) is defined from \(\Phi^{\star}_{w \to k}\), and we now show how one can lift this requirement. We define the \emph{right-sketch function} \(T_{k}\), which is a linear operator of the form
\[
T_k \colon \prod_{i \in \rightside(k)} [n_i] \times [m_{(k, \parent(k))}] \to \bR.
\]

Using \(T_{k}\) and \(S_{k}\), one can jointly form a linear projection of \(p^{\star}\), with the result referred to as \(Z^{\star}_{k}\), as follows:
\begin{equation}
    \label{eqn: def of Z k} 
    Z^{\star}_{k}(\beta_{(k,\child(k))}, x_{k}, \gamma_{(k, \parent(k))} )   =    \sum_{x_{\leftside(k) \cup \rightside(k)}} S_{k}(\beta_{(k,\child(k))}, x_{\leftside(k)}) 
 p^{\star}(x_{\leftside(k)}, x_{k}, x_{\rightside(k)}) T_{k}(x_{\rightside(k)}, \gamma_{(k, \parent(k))}).
\end{equation}



One then performs singular value decomposition (SVD) to \(Z^{\star}_{k}\) according to the unfolding \(Z^{\star}_{k}(\beta_{(k,\child(k))}, x_{k}; \gamma_{(k, \parent(k))} )\). Due to the low rank structure of \(p^{\star}\) at \(k \to \parent(k)\), the following rank \(r_{(k, \parent(k))}\) decomposition is exact:
\begin{equation}\label{eqn: svd of Z star}
    Z^{\star}_{k}(\beta_{(k,\child(k))}, x_{k}; \gamma_{(k, \parent(k))} ) = U^{\star}_{k}\Sigma^{\star}_{k}\left( V^{\star}_{k}\right)^\top.
\end{equation}

\begin{figure}
     \centering
     \includegraphics[width=0.4\textwidth]{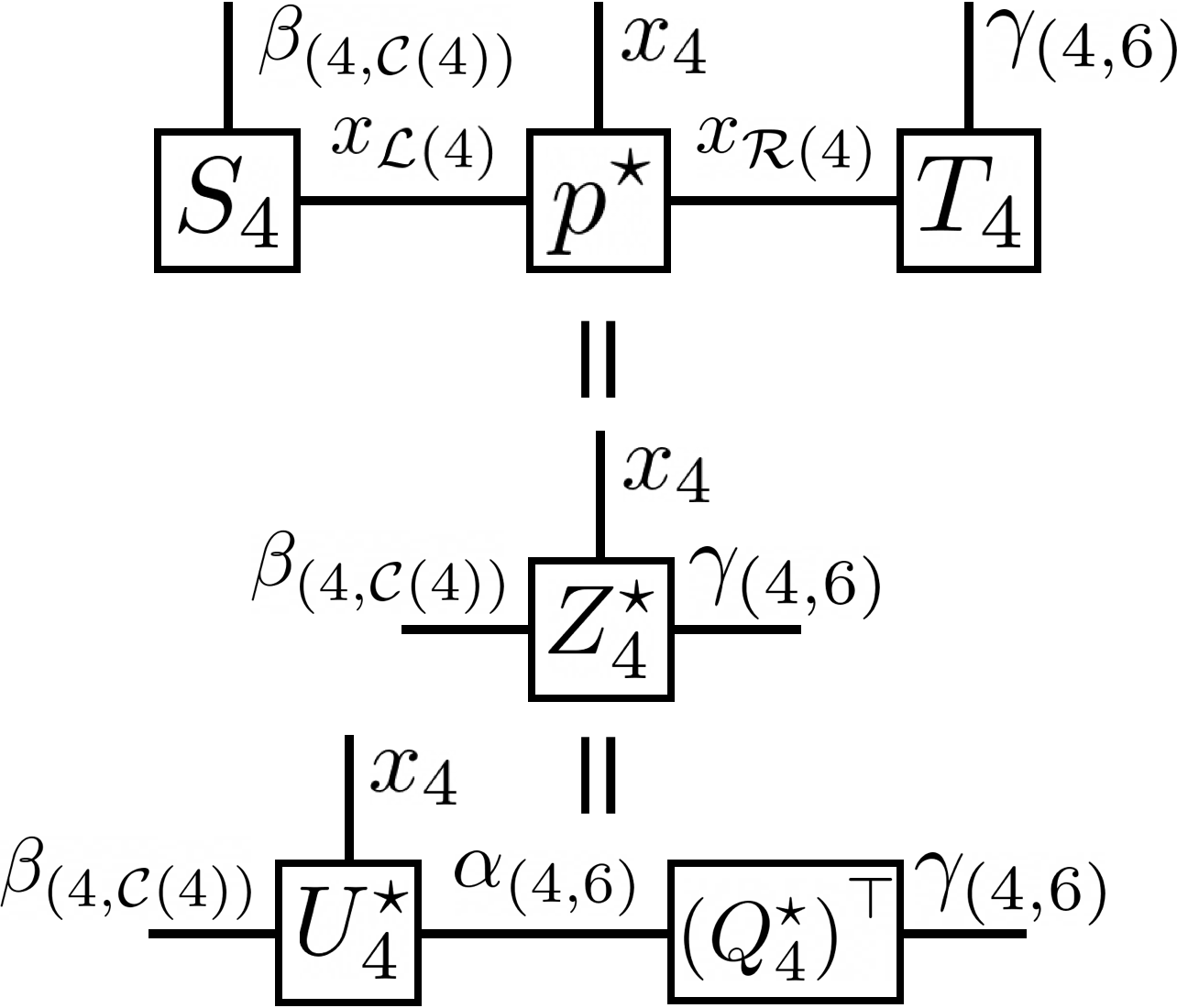}
    \caption{Tensor diagram of the sketching step in \eqref{eqn: def of Z k} and the SVD step in \eqref{eqn: svd of Z star}. The illustration is over the rooted tree in Figure \ref{fig:tree_topology_notation_infograph}.}
    \label{fig:sketch_plus_svd_of_sketch}
\end{figure}

Set \(Q^{\star}_{k} = V^{\star}_{k}\left(\Sigma^{\star}_{k}\right)^{-1}\) and \(\left(Q_{k}^{\star}\right)^{\top} = \Sigma^{\star}_{k}\left(V_{k}^{\star}\right)^{\top}\). Note that \(\left(Q_{k}^{\star}\right)^{\top}\) is the pseudo-inverse of \(Q^{\star}_{k}\). In particular, one has
 \[
     Z^{\star}_{k}(\beta_{(k,\child(k))}, x_{k}; \gamma_{(k, \parent(k))} ) = U^{\star}_{k}(\beta_{(k,\child(k))}, x_{k}; \alpha_{(k, \parent(k))})\left(Q^{\star}_{k}\right)^{\top}(\alpha_{(k, \parent(k))}; \gamma_{(k, \parent(k))}).
 \]
As a summary of \(Z_{k}^\star, U_{k}^\star, Q_{k}^\star\), see illustration in Figure \ref{fig:sketch_plus_svd_of_sketch}.

One can naturally shape \( U^{\star}_{k}(\beta_{(k,\child(k))}, x_{k}; \alpha_{(k, \parent(k))})\) as a tensor of the index \(U^{\star}_{k}(\beta_{(k,\child(k))}, x_{k}, \alpha_{(k, \parent(k))})\), i.e. \(U^{\star}_{k} \colon \prod_{w \in \child(k)} [l_{(w, k)}] \times [n_{k}] \times [m_{(k, \parent(k))}] \rightarrow \bR\). We now write out our choice of gauge and its consequences in Condition \ref{cond: ttns gauge choice}:
\begin{cond}\label{cond: ttns gauge choice}
(TTNS-Sketch gauge choice)
    The gauge for \(\Phi^{\star}_{k \to \parent(k)}\) is chosen so that the following holds:
    \begin{equation}\label{eqn: gauge choice}
        U^{\star}_{k}(\beta_{(k,\child(k))}, x_{k}, \alpha_{(k, \parent(k))}) = \sum_{x_{\leftside(k)}} S_{k}(\beta_{(k,\child(k))}, x_{\leftside(k)}) \Phi^{\star}_{k \to \parent(k)}(x_{\leftside(k) \cup k}, \alpha_{(k, \parent(k))}).
    \end{equation}
    As a consequence of \eqref{eqn: gauge choice} and \eqref{eqn: def of B_k}, one has:
    \begin{equation}\label{eqn: gauge def of B_k}
        B^{\star}_{k} = U^{\star}_{k}.
    \end{equation}
    As a consequence of \eqref{eqn: svd of Z star}, the matrix \(\left(Q^{\star}_{k}\right)^{\top}\) is a projection of \(\Psi^{\star}_{k \to \parent(k)}\) by \(T_{k}\), i.e.
    \begin{equation}
    \label{eqn: def of Q_k^T star}
        \left(Q^{\star}_{k}\right)^{\top}(\alpha_{(k, \parent(k))}, \gamma_{(k, \parent(k))}) = \sum_{x_{\rightside(k)}} \Psi^{\star}_{k \to \parent(k)}(\alpha_{(k, \parent(k))}, x_{\rightside(k)})T_{k}(x_{\rightside(k)}, \gamma_{(k, \parent(k))}).
    \end{equation}
\end{cond}


Likewise, we now show how \(A^{\star}_{k}\) can be obtained. By the choice of gauge in Condition \ref{cond: ttns gauge choice}, one forms Corollary \ref{cor: def of A_k star}. In Figure \ref{fig:structure_of_coefficient_tensor}, we include an short proof using tensor diagram. As a consequence of Corollary \ref{cor: def of A_k star}, one can form a linear system for \(G_{k}^\star\) completely in terms of the sketches $\{Z^{\star}_{w \to k}\}_{w \to k \in E}, \{Z^{\star}_i\}_{i = 1}^{d}$, which can be reasonably approximated by samples. As an illustration, one can rewrite the tensor diagram in Figure \ref{fig:sketched_down_CDE_illustration} as the new diagram illustrated in Figure \ref{fig:CDE_from_sketch}.

\begin{figure}
     \centering
     \includegraphics[width=0.6\textwidth]{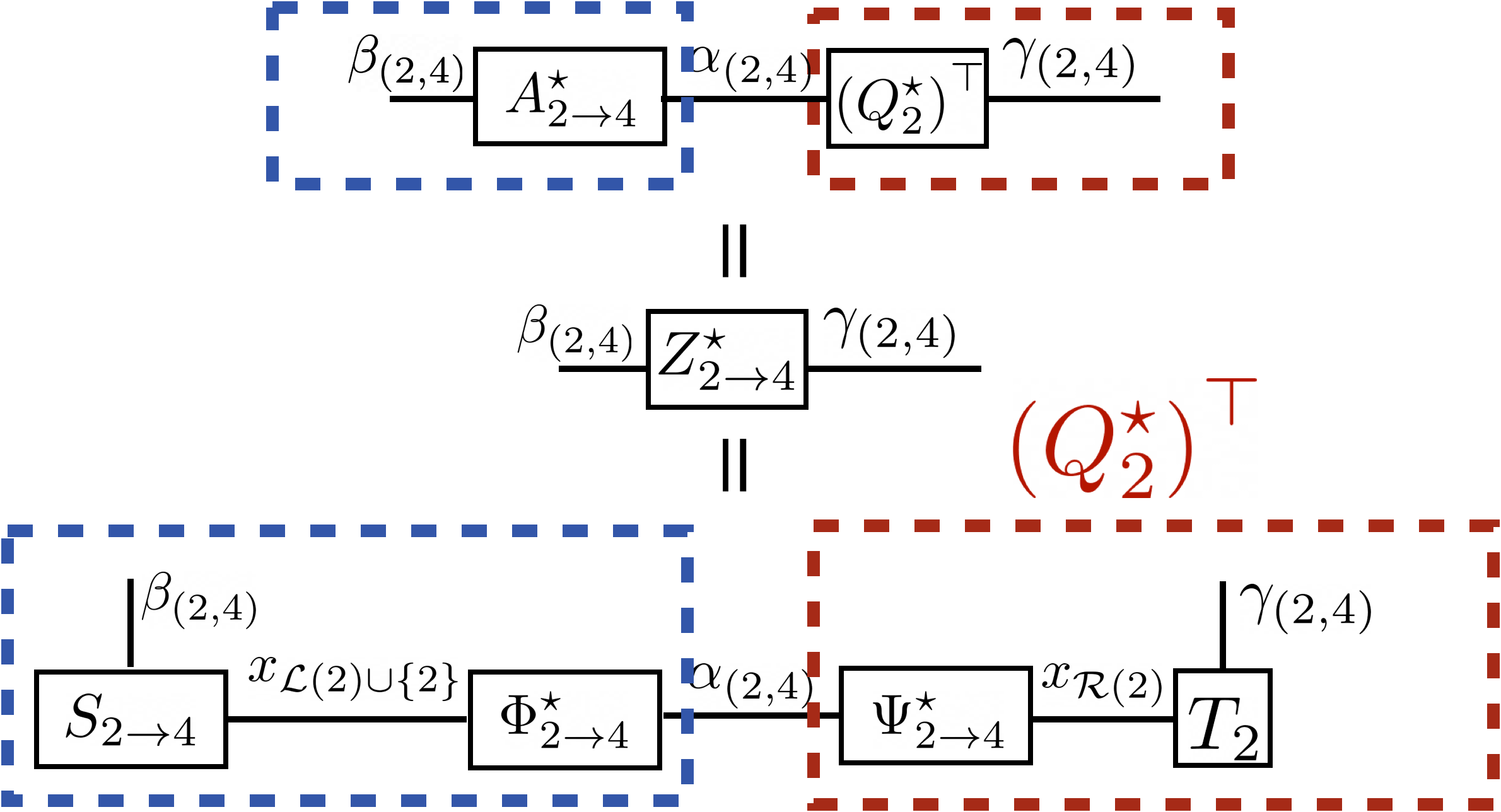}
     \caption{Proof of Corollary \ref{cor: def of A_k star} in terms of tensor diagram. Both equality holds due to \eqref{eqn: A w to k and phi w to k def}. Then, the tensors enclosed in red box coincide due to \eqref{eqn: def of Q_k^T star}, and so the tensors enclosed in blue box coincide, which is what we need to show.}
     \label{fig:structure_of_coefficient_tensor}
\end{figure}

\begin{cor}
\label{cor: def of A_k star}

Define the intermediate terms \(Z^{\star}_{w \to k}\) and \(A^{\star}_{w \to k}\) as follows:
\begin{equation}
\label{eqn: A w to k and phi w to k def}
    \begin{aligned}
    Z^{\star}_{w \to k}(\beta_{(w,k)}, \gamma_{(w, k)}) &= 
    \sum_{x_{[d]}} S_{w \to k}(\beta_{(w, k)}, x_{\leftside(w) \cup w})
    p^{\star}(x_{\leftside(w) \cup w}, x_{\rightside(w)})T_{w}(x_{\rightside(w)}, \gamma_{(w, k)})  
    \\
    A^{\star}_{w \to k}(\beta_{(w,k)}, \alpha_{(w, k)}) &= \sum_{\gamma_{(w, k)}} Z^{\star}_{w \to k}(\beta_{(w,k)}, \gamma_{(w, k)})Q^{\star}_{w}(\gamma_{(w, k)}, \alpha_{(w, k)}),
\end{aligned}
\end{equation}

Then \(A^{\star}_{k}\) satisfies the following equation
\begin{equation}
A^{\star}_{k}(\beta_{(k,\child(k))}, \alpha_{(k, \child(k))}) = 
\prod_{w \in \child(k)} A^{\star}_{w \to k}(\beta_{(w,k)}, \alpha_{(w, k)}).
\end{equation}
\end{cor}
\begin{proof}

    By the factorization structure of \(A^{\star}_{k}\), it suffices to show that
    \[
    A^{\star}_{w \to k}(\beta_{(w,k)}, \alpha_{(w, k)}) = \sum_{x_{\leftside(w)\cup w}}S_{w \to k}(\beta_{(w,k)}, x_{\leftside(w) \cup w})  \Phi^{\star}_{w \to k}(x_{\leftside(w) \cup w}, \alpha_{(w, k)}).
    \]

    We use a unfolding matrix structure for \(Z^{\star}_{w \to k}, S_{w \to k}, T_{w}\) in the following rewrite of \eqref{eqn: A w to k and phi w to k def}:
    \[
    Z^{\star}_{w \to k}(\beta_{(w,k)}; \gamma_{(w, k)}) = 
    S_{w \to k}(\beta_{(w, k)}; x_{\leftside(w) \cup w})
    p^{\star}(x_{\leftside(w) \cup w}; x_{\rightside(w)})T_{w}(x_{\rightside(w)}; \gamma_{(w, k)})  
    \]
    
    Likewise, we use the unfolding matrix structure of \(p^{\star}, \Phi^{\star}_{w \to k}, \Psi^{\star}_{w \to k}\) in 
    \[
    p^{\star}(x_{\leftside(w) \cup w}; x_{\rightside(w)}) = \Phi^{\star}_{w \to k}(x_{\leftside(w) \cup w}; \alpha_{(w, k)}) \Psi^{\star}_{w \to k}(\alpha_{(w, k)}; x_{\rightside(w)}).
    \]
    
    Using the unfolding matrix structure as just suggested, it suffices to prove
    \begin{equation*}
        A^{\star}_{w \to k} = S_{w \to k}\Phi^{\star}_{w \to k}.
    \end{equation*}
    
    The definition for the intermediate term \(Z^{\star}_{w \to k}\) simplifies to \(Z^{\star}_{w \to k} = S_{w \to k}p^{\star}T_{w}\) and more importantly, \[A^{\star}_{w \to k} = Z^{\star}_{w \to k}Q^{\star}_{w} = S_{w \to k}p^{\star}T_{w}Q^{\star}_{w}.\]
    
    Note that one can expand according to \(p^{\star} = \Phi_{w\to k}^{\star}\Psi^{\star}_{w\to k}\) and get
    \[
     A^{\star}_{w \to k} =  S_{w \to k}\Phi_{w\to k}^{\star}\Psi^{\star}_{w\to k}T_{w}Q^{\star}_{w} = S_{w \to k}\Phi_{w\to k}^{\star}\left(Q^{\star}_{w}\right)^{\top}Q^{\star}_{w} = S_{w \to k}\Phi_{w\to k}^{\star},
    \]
    where the second equality uses \eqref{eqn: def of Q_k^T star} and last equality uses that \(\left(Q^{\star}_{w}\right)^{\top}Q^{\star}_{w}\) is an identity matrix.
\end{proof}

\subsection{Sample estimation of \(A^{\star}_{k}\) and \(B^{\star}_{k}\) in TTNS-Sketch}
Practically, one only has access to the empirical distribution \(\hp\) via samples $\{(y_1^{(i)}, \ldots, y_d^{(i)})\}_{i = 1}^{N}$. A finite sample approximation of \(Z^{\star}_{k}\) is tractable and can be obtained by function evaluations of \(T_{k}\) and \(S_{k}\). One has
\begin{equation}
\label{eqn: def of tilde phi k}
    \hat Z_{k}(\beta_{(k,\child(k))}, x_{k}, \gamma_{(k, \parent(k))} )   =    \sum_{i = 1}^{N} S_{k}(\beta_{(k,\child(k))}, y_{\leftside(k)}^{(i)})
 \mathbf{1}(y_k^{(i)} = x_{k}) T_{k}(y_{\rightside(k)}^{(i)}, \gamma_{(k, \parent(k))}).
\end{equation}

Similarly, \(Z^{\star}_{w \to k}\) can be approximated by
\begin{equation}
\label{eqn: def of tilde phi w to k}
    \hat Z_{w \to k}(\beta_{(w,k)}, \gamma_{(w, k)})   =    \sum_{i = 1}^{N} S_{w \to k}(\beta_{(w, k)}, y_{\leftside(w) \cup w}^{(i)})T_{w}(y_{\rightside(w)}^{(i)}, \gamma_{(w, k)}).
\end{equation}

One can then form \(\hat U_{k}, \hat Q_{k}\) by replacing \(Z^{\star}_{k}\) with \(\hat Z_{k}\) in \eqref{eqn: svd of Z star}, noting that in this case the rank \(r_{(k, \parent(k))}\) SVD decomposition is not exact due to the presence of noise. We now explain why this algorithm is practical in the sample case. It suffices to see the linear equation in Figure \ref{fig:CDE_from_sketch}. If one replaces every tensor block by a finite sample approximation in the sense discussed above (e.g. replace \(Z^{\star}_{w \to k}\) by \(\hat Z_{w \to k}\)), then one can indeed form a linear equation, with accuracy increasing with sample size. The full algorithm for empirical distributions will be described in Section \ref{sec: main algorithm}.

\begin{figure}
    \centering
\begin{subfigure}[b]{\textwidth}
     \centering
     \includegraphics[width=\textwidth]{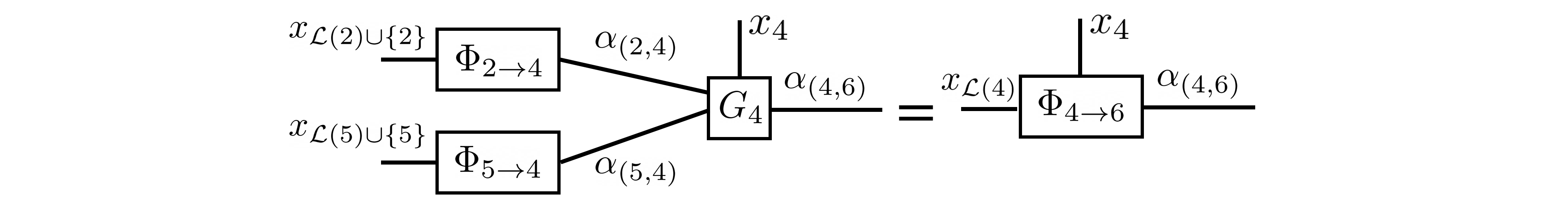}
    \caption{Tensor diagram representation of the Core Determining Equation  \eqref{eq:ttn-determining} for \(k = 4\) in Theorem \ref{thm:ttns-existence}. The left side uses \(\Phi_{2 \to 4}\) and \(\Phi_{5 \to 4}\) instead of \(\Phi_{\child(4) \to 4}\), which is allowed due to \eqref{eqn: phi C_k to k}.}  
    \label{fig:full_CDE_illustration}
\end{subfigure}
\begin{subfigure}[b]{\textwidth}
     \centering
     \includegraphics[width=\textwidth]{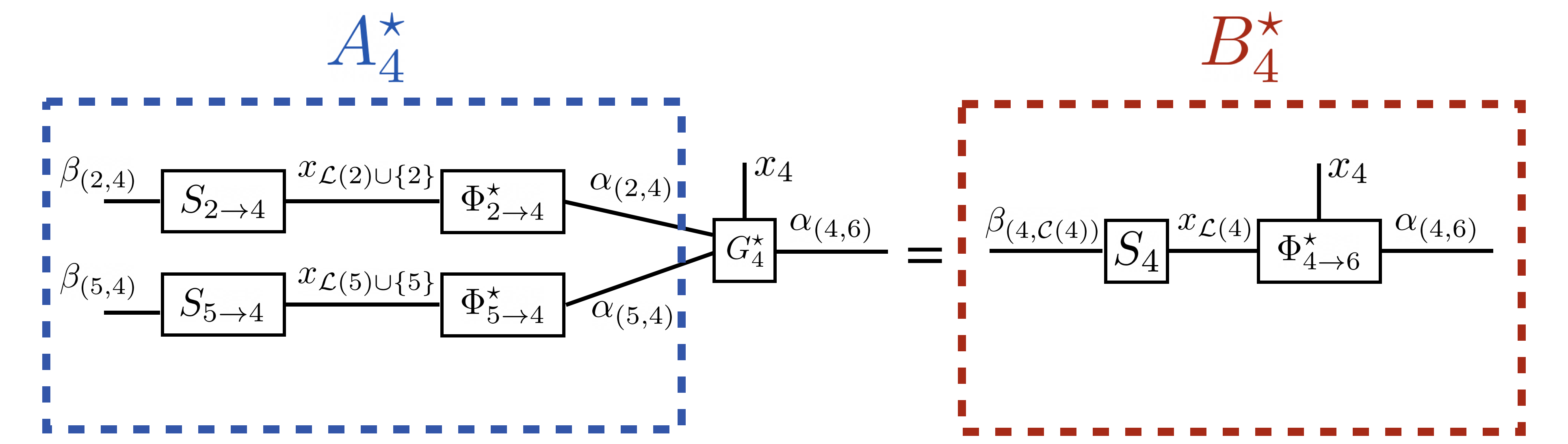}
    \caption{Tensor diagram representation of the sketched-down Core Determining Equation \eqref{eqn: noiseless equation for G_k^star} for \(k = 4\). The diagram can be thought of as a linear equation for \(G_{4}^{\star}\). Note that the tensors enclosed in blue box can be thought of as the coefficient, and the tensors enclosed in red box can be thought of as the right hand side. One can see that the tensor diagram equation is obtained from that of Figure \ref{fig:full_CDE_illustration} by a contraction according to \(S_{4}\) to both sides. The left side of the diagram uses blocks corresponding to \(S_{2 \to 4}, S_{5 \to 4}\), which is a consequence of \eqref{eqn: S_k splits}. }
    \label{fig:sketched_down_CDE_illustration}
\end{subfigure}
\begin{subfigure}[b]{\textwidth}
     \centering
     \includegraphics[width=\textwidth]{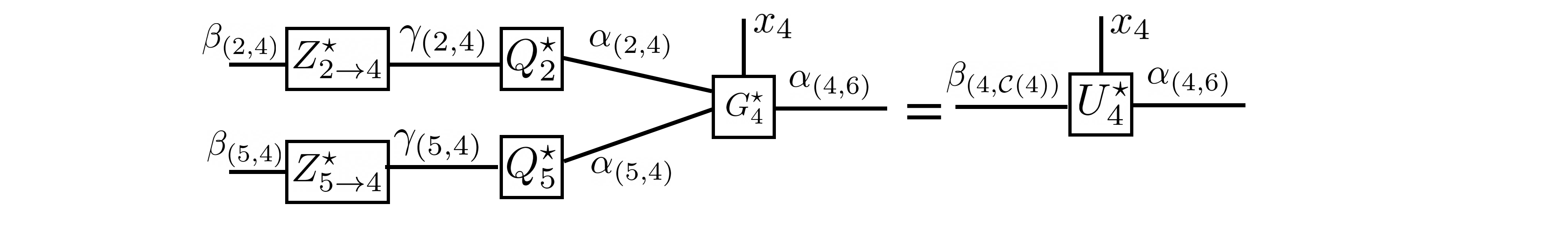}
    \caption{With $\{Z^{\star}_{w \to k}\}_{w \to k \in E}, \{Z^{\star}_i\}_{i = 1}^{d}$ from the sketch step and  $\{U^{\star}_i, Q^{\star}_i\}_{i = 1}^{d}$ from the SVD step, one can rewrite the tensor diagram in Figure \ref{fig:sketched_down_CDE_illustration} into a tractable linear equation for \(G^\star_{k}\). Note that the diagram is equivalent to that in Figure \ref{fig:sketched_down_CDE_illustration} due to Corollary \ref{cor: def of A_k star}.}
    \label{fig:CDE_from_sketch}
\end{subfigure}
    \label{fig:TTNS-Sketch summary}
    \caption{Tensor diagram representation of equations for TTNS tensor cores in TTNS-Sketch. The illustration is over the rooted tree in Figure \ref{fig:tree_topology_notation_infograph}. The equation in Figure \ref{fig:full_CDE_illustration} implies the rest. The diagram in Figure \ref{fig:CDE_from_sketch} is derived from the diagram in Figure \ref{fig:sketched_down_CDE_illustration} with a specific choice of gauge according to Condition \ref{cond: ttns gauge choice}. Importantly, the equation in Figure \ref{fig:CDE_from_sketch} allows for finite sample approximation.}
\end{figure}

\section{Topology finding}\label{sec: topology finding}

\newcommand{\TCL}{T_{\mathrm{CL}}}
\newcommand{\PCL}{p_{\mathrm{CL}}}

\subsection{The Chow-Liu algorithm for topology finding}
For an arbitrary distribution \(p^{\star}\), we discuss the tree topology specification problem. It could happen that one has no access to a pre-selected candidate tree topology, but one wishes to test if there is a tree topology to reasonably capture the structure for \(p^{\star}\). Our main algorithm for tree topology specification is the Chow-Liu algorithm, for which we now give a brief description. 

The input is a \(d\)-dimensional distribution \(p\) with its random vector denoted by \(X  :=  \left(X_{1}, \ldots, X_{d}\right)\). Typically, \(p\) is the empirical distribution \(\hp\) over given samples. In the first step, one computes the pairwise mutual information \(I(X_{i}, X_{j})\) over any distinct pair of \((i,j) \in [d] \times [d]\). In the second step, one forms a graph \(G\), which is a complete graph on \(V  :=  [d]\) with edge weight given by \(I(X_{i}, X_{j})\). In the third step, Kruskal's algorithm is used to obtain the maximal spanning tree on \(G\), i.e. a spanning tree over \(d\) nodes that maximizes the sum of mutual information over all of its edges. This maximal mutual information spanning tree is the Chow-Liu tree \(\TCL\). 

The Chow-Liu algorithm also outputs a graphical model over \(\TCL\). After specifying certain marginal and conditional probabilities to match that of \(p\), one uniquely determines the Chow-Liu model \(\PCL\), which is a graphical model based on \(\TCL\). 

We will discuss the rationale for using \(\TCL\) for our TTNS algorithm in three cases.

In the first case, if \(p^{\star}\) is indeed a graphical model over a tree \(T^{\star}\), then the Chow-Liu tree \(\TCL\) will be \(T^{\star}\) with high probability (see below). Moreover, it is well-known that \(\PCL\) is the tree graphical model with minimum KL divergence to the input distribution. With mild constraint on bond dimension, the class of functions representable by a TTNS format strictly covers density representable by a graphical model, and the performance of TTNS-Sketching with \(\TCL\) has a performance which is on par with \(\PCL\).

In the second case, if \(p^{\star}\) has a TTNS ansatz over a tree \(T^{\star}\), then it is typically true that farther-away nodes in \(T^{\star}\) are less correlated. By the maximal spanning tree procedure in Chow-Liu, variables that are far away in \(\TCL\) are also typically less correlated. If \(\TCL\) and \(T^{\star}\) differs locally, then the TTNS-sketching algorithm still performs well empirically. For a important example in this category, suppose \(p^{\star}\) is given by a graphical model over a graph with loops. In this case, \(\TCL\) will converge to a spanning tree of the graph, and one can form a TTNS ansatz with \(\TCL\) by choosing appropriately large bond dimension. 

In the third case, it may happen that \(p^{\star}\) cannot be represented by a TTNS anstaz. In this case, one can quickly reject the TTNS model assumption by looking at the mutual information \(I(X_{i}, X_{j})\) used in the Chow-Liu algorithm. For any node \(k \in [d]\), removing \(k\) will separate \(\TCL\) into two connected components. If one sees several pairs of strong correlation between nodes separated by \(k\), then the density \(p^{\star}\) most likely fails the TTNS model assumption, and more general tensor networks might be more applicable.

\subsection{Sample complexity for successful tree topology recovery}

We discuss the amount of samples required for Chow-Liu to pick the ``correct" tree topology. By a correct tree topology, we will mean that the Chow-Liu Tree \(\TCL\) equals to the tree one would have obtained if one forms the maximal spanning tree based on the exact mutual information. If \(p^{\star}\) is a graphical model over \(T^{\star}\), then this notion of correctness coincides with the intuitive notion of \(\TCL = T^{\star}\).

There has been considerable recent work in the past few years on the sample complexity of the Chow-Liu algorithm to infer the correct tree topology in the sample case. \cite{bhattacharyya2021near} shows that the sample complexity is bounded by \(O(\frac{n^3 d}{\epsilon} \log{1/\delta})\) to ensure a \(1-\delta\) success rate, where \(\epsilon\) is the gap in the sum of mutual information between the two best tree models. For tree-based Ising model with zero external field, \cite{bresler2020learning} proves an upper bound that is \(O(\log{(d/\delta)})\).




\section{TTNS-Sketch for empirical distributions}\label{sec: main algorithm}

We now give the main Algorithm \ref{alg:1}-\ref{alg:1-3} for the TTNS ansatz with empirical distribution as input, which is the main use case. We include it separately from Section \ref{sec: ttns-sketch full explanation} due to corner cases such as when \(k\) is a leaf or root node. To emphasize the sample estimation procedure, all of the intermediate terms will be labelled with \(\hat{}\) if it assumes access to \(\hp\). Importantly, Algorithm \ref{alg:1-3} only takes in the sketches $ \{Z_{w \to k}\}_{w \to k \in E}, \{Z_i\}_{i = 1}^{d}$ as input, which can be either noiseless or estimated, and so we do not label terms inside this subroutine.

\begin{figure}
    \centering
    \includegraphics[width = 0.6\textwidth]{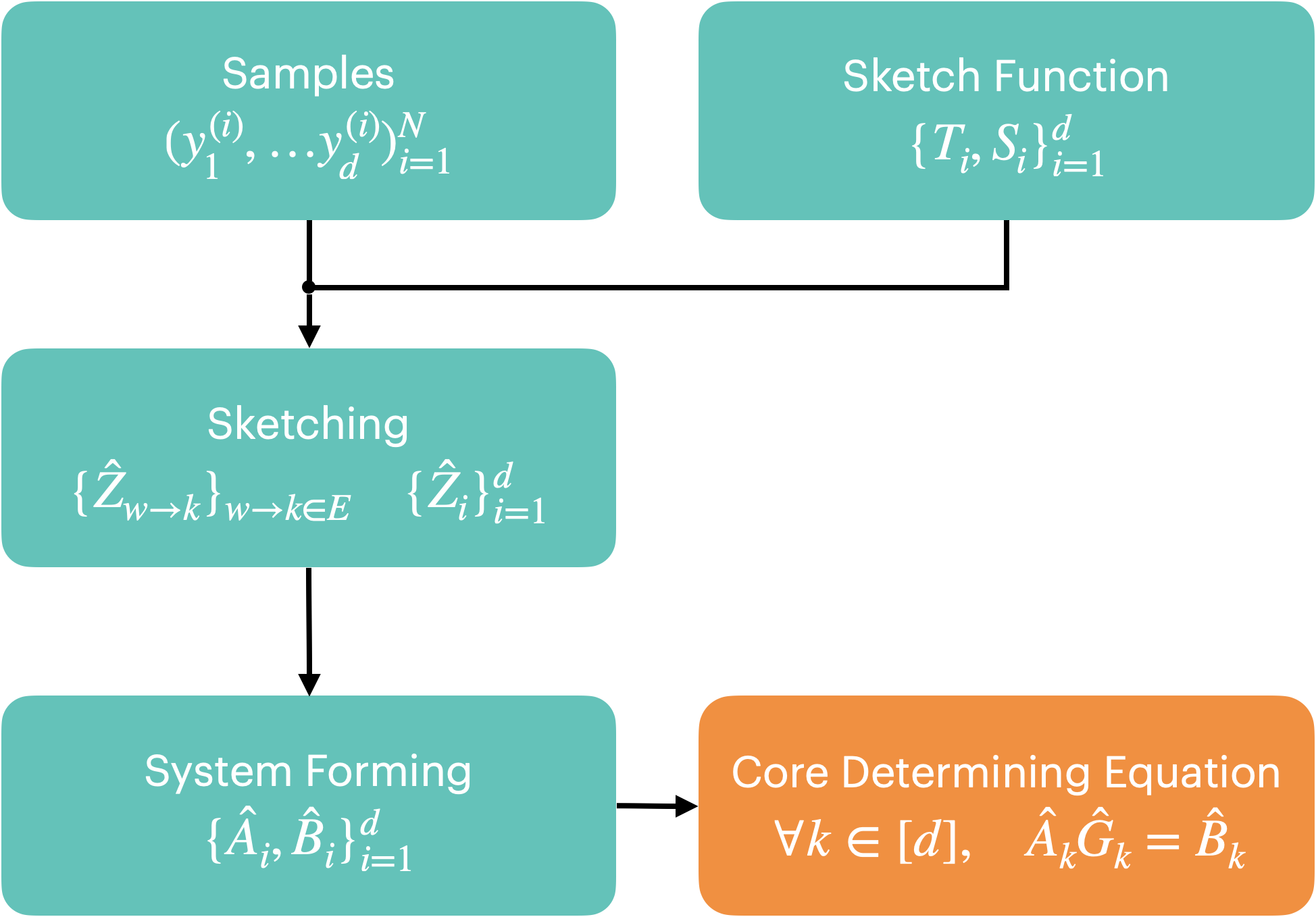}
    \caption{Schematic of Algorithm \ref{alg:1}}
    \label{fig:schematic_TTNS}
\end{figure}


\begin{algorithm}
\caption{TTNS-Sketch for empirical distribution $\hp$.}
\label{alg:1}
\begin{algorithmic}[1]
\REQUIRE Empirical distribution \(\hp\) formed by samples $\{(y_1^{(i)}, \ldots, y_d^{(i)})\}_{i = 1}^{N}$.
\REQUIRE A rooted tree structure \(T = ([d], E)\), and $\child, \parent, \leftside, \rightside$ as in Definition \ref{defn: tree topology information}.
\REQUIRE Target ranks $\{r_e : e \in E\} \subset \mathbb{N}$.
\REQUIRE $T_k \colon \prod_{i \in \rightside(k)} [n_i] \times [m_{(k, \parent(k))}] \to \bR$ for all non-root $k \in [d]$ and $T_k = 1$ if $k$ is the root.
\REQUIRE $S_1, \ldots, S_d$ formed by $S_{w \to k}$'s as in \eqref{eqn: S_k splits}.
\STATE $ \{\hat Z_{w \to k}\}_{w \to k \in E}, \{\hat Z_i\}_{i = 1}^{d}\leftarrow \textsc{Sketching}(\hp, T_{1}, \ldots, T_d, S_1, \ldots, S_{d})$.
\STATE $\{\hat{A}_i, \hat B_i\}_{i = 1}^{d} \leftarrow \textsc{SystemForming}(\{\hat Z_{w \to k}\}_{w \to k \in E}, \{\hat Z_i\}_{i = 1}^{d}).$

\STATE Solve the following $d$ equations via least-squares for the variables $\hat G_1, \ldots, \hat G_d$:
\begin{equation}
    \label{eq:alg-CDEs}
    \begin{aligned}
        \hat G_{k} & = \hat B_{k} \quad \text{if} ~ k ~ \text{is a leaf}, \\
        \textstyle\sum_{\alpha_{(k, \child(k))}} \hat A_{k}(\beta_{(k, \child(k))}, \alpha_{(k, \child(k))}) \hat G_{k}(x_k, \alpha_{(k, \neighbor(k))}) & =\hat B_{k} \quad \text{otherwise}, 
    \end{aligned}
\end{equation}
where $\hat G_k \colon [n_k] \times \prod_{w \in \neighbor(k)} [r_{(w, k)}] \to \mathbb{R}$. 
\RETURN $\hat G_1, \ldots, \hat G_d$ 
\end{algorithmic}
\end{algorithm}


\begin{algorithm}
\caption{\textsc{Sketching}.}
\label{alg:1-2}
\begin{algorithmic}
\REQUIRE $\hat p$, $T_{1}, \ldots, T_{d}$, and $S_1, \ldots, S_{d}$ as given in Algorithm \ref{alg:1}.
\REQUIRE \(T = ([d], E)\) as given in Algorithm \ref{alg:1}.
\FOR{$k = 1$ to $d$}
    \IF{$k$ is root}
    \STATE Define $\hat Z_{k} \colon \prod_{w \in \child(k)} [l_{(w, k)}] \times [n_{k}] \to \bR$ as
    \begin{equation*}
        \hat Z_{k}(\beta_{(k,\child(k))}, x_{k})
        =
        \sum_{i = 1}^{N} S_{k}(\beta_{(k,\child(k))}, y_{\leftside(k)}^{(i)})\mathbf{1}(y_k^{(i)} = x_{k}).
    \end{equation*}
    \ELSIF{$k$ is leaf}
    \STATE Define $\hat Z_{k} \colon  [n_{k}] \times [m_{(k, \parent(k))}] \to \bR$ as
    \begin{equation*}
        \hat Z_{k}(x_{k}, \gamma_{(k, \parent(k))})
        =
        \sum_{i = 1}^{N} \mathbf{1}(y_k^{(i)} = x_{k})T(y_{\rightside(k)}^{(i)}, \gamma_{(k, \parent(k))}).
    \end{equation*} 
    \STATE
    Define $\hat Z_{k \to \parent(k)}$ according to \eqref{eqn: def of tilde phi w to k}
    \ELSE
    \STATE Define $\hat Z_{k}\colon  \prod_{w \in \child(k)} [l_{(w, k)}] \times [n_{k}] \times [m_{(k, \parent(k))}] \to \bR$ according to \eqref{eqn: def of tilde phi k}.
    \STATE Define $\hat Z_{k \to \parent(k)} \colon [l_{(k,\parent(k))}] \times [m_{(k, \parent(k))}] \to \bR$ according to \eqref{eqn: def of tilde phi w to k}.
    \ENDIF
\ENDFOR
\RETURN $ \{\hat Z_{w \to k}\}_{w \to k \in E}, \{\hat Z_i\}_{i = 1}^{d}$.
\end{algorithmic}
\end{algorithm}


\begin{algorithm}
\caption{\textsc{SystemForming}.}
\label{alg:1-3}
\begin{algorithmic}
\REQUIRE Sketches $ \{Z_{w \to k}\}_{w \to k \in E}, \{Z_i\}_{i = 1}^{d}$. 
\REQUIRE Tree structure \(T = ([d], E)\) as given in Algorithm \ref{alg:1}.
\REQUIRE Target rank $r$ as given in Algorithm \ref{alg:1}.
\FOR{$k = 1$ to $d$}
    \STATE \(\beta_{k} \gets \beta_{(k,\child(k))}\).
    \STATE \(\gamma_{k} \gets \gamma_{\parent(k)}\).
    \STATE \(r_{k} \gets r_{(\parent(k), k)}\).
    \STATE \(l_{k} \gets \prod_{w \in \child(k)}l_{(w, k)}\).
    \STATE \(m_{k} \gets m_{(k , \parent(k))}\).
    \IF{$k$ is a leaf}
        \STATE Let $U_k \Sigma_k V^\top_k$, where $U_k \in \mathbb{R}^{n_k \times r_k}$, $V_k \in \mathbb{R}^{m_k \times r_k}$, $\Sigma_k \in \mathbb{R}^{r_k \times r_k}$, be the best rank-$r_k$ approximation to the matrix $Z_{k}(x_k; \gamma_k)$ via SVD. Define $B_k \colon [n_k] \times [r_k] \to \bR$ where $B_k(x_k, \alpha_k) = U_k(x_k; \alpha_k)$. Set \(A_{k} = 1\).
    \ELSIF{$k$ is root}
        \STATE Let $B_{k} = Z_{k}$. Set $Q_k = 1$.
    \ELSE
        \STATE Let $U_k \Sigma_k V^\top_k$, where $U_k \in \mathbb{R}^{l_{k} n_k \times r_k}$, $V_k \in \mathbb{R}^{m_k \times r_k}$, $\Sigma_k \in \mathbb{R}^{r_k \times r_k}$, be the best rank-$r_k$ approximation to the matrix $Z_{k}(\beta_{k}, x_{k}; \gamma_{k})$ via SVD. Define $B_{k} \colon \prod_{w \in \child(k)}[l_{(w, k)}] \times [n_k] \times [r_k] \to \bR$ where $B_{k}(\beta_{k}, x_{k}, \alpha_{k}) = U_{k}(\beta_{k}, x_{k}; \alpha_{k})$.
    \ENDIF
    \IF{$k$ is non-root}
    \STATE Let $Q_{k} = V_k \Sigma_k^{-1}$.
    \ENDIF
    \ENDFOR
    \FOR{$k = 1$ to $d$}
    \IF{$k$ is non-leaf}
        \STATE Compute $A_{k} \colon \prod_{w \in \child(k)} [l_{(w, k)}] \times \prod_{w \in \child(k)}[r_{(w,k)}] \to \bR$:
        \STATE \begin{equation*}
            A_{k}(\beta_{(k,\child(k))}, \alpha_{(k, \child(k))}) = 
            \prod_{w \in \child(k)}\sum_{\gamma_{(w, k)}} Z_{w \to k}(\beta_{(w,k)}, \gamma_{(w, k)})Q_{w}(\gamma_{(w, k)}, \alpha_{(w, k)}).
        \end{equation*}
    \ENDIF
\ENDFOR
\RETURN $\{A_i, B_i\}_{i = 1}^{d}$.
\end{algorithmic}
\end{algorithm}




\subsection{Condition for consistency of TTNS-Sketch}
We introduce Condition \ref{cond: Condition for sketch functions} on the choice of sketch functions. Essentially, Theorem \ref{thm:ttns-existence} shows that \eqref{eq:ttn-determining} is an over-determined linear system with a unique exact solution, and one needs the ``sketched-down" version of \eqref{eq:ttn-determining} to still have a unique solution:
\begin{cond}
    \label{cond: Condition for sketch functions}
    Let $p^{\star}$ be a function which satisfies Condition \ref{cond: TTNS ansatz condition}. Moreover, let $\{\Phi^{\Delta}_{(w, k)}, \Psi^{\Delta}_{(w, k)}\}_{(w, k) \in E}$ be an arbitrary collection of tensors forming the low-rank decomposition of \(p^{\star}\) in the sense of Condition \ref{cond: TTNS ansatz condition}, with gauge chosen arbitrarily. Let \(\{T_{i}, S_{i}\}_{i = 1}^{d}\) be the sketch functions in Algorithm \ref{alg:1}.
    Define two intermediate terms \(A^{\Delta}_{k} \colon \prod_{w \in \child(k)}[l_{w \to k}] \times \prod_{w \in \child(k)}[r_{w \to k}] \to \bR\) and \(B^{\Delta}_k \colon \prod_{w \in \child(k)}[l_{w \to k}] \times [n_{k}] \times [\gamma_{k \to \parent(k)}] \to \bR\) by
    \begin{equation}
    \label{eqn: def of A^ast and B^ast}
    \begin{aligned}
        &A^{\Delta}_{k}(\beta_{(k, \child(k))}, \alpha_{(k, \child(k))})  :=  
        \prod_{w \in \child(k)} \sum_{x_{\leftside(w)\cup w}}S_{w \to k}(\beta_{(w,k)}, x_{\leftside(w) \cup w})  \Phi^{\Delta}_{w \to k}(x_{\leftside(w) \cup w}, \alpha_{(w, k)}),\\
        &B^{\Delta}_{k}(\beta_{(k,\child(k))}, x_{k}, \alpha_{(k, \parent(k))})  :=  \sum_{x_{\leftside(k)}} S_{k}(\beta_{(k,\child(k))}, x_{\leftside(k)}) \Phi^{\Delta}_{k \to \parent(k)}(x_{\leftside(k) \cup k}, \alpha_{(k, \parent(k))}).
    \end{aligned}
    \end{equation}
    Moreover, define an intermediate term $\bar{\Phi}^{\star}_k \colon \prod_{i \in \leftside(k) \cup k} [n_{i}] \times [m_{(k, \parent(k))}] \to \bR$ by 
    \[
    \bar{\Phi}^{\star}_k(x_{\leftside(k) \cup k}, \gamma_{(k, \parent(k))}) = \sum_{x_{\rightside(k)}} p^{\star}(x_{1}, \ldots, x_{d}) T_{k}(x_{\rightside(k)}, \gamma_{(k, \parent(k))}).
    \]
    
    Then, \(\{T_{i}, S_{i}\}_{i = 1}^{d}\) in Algorithm \ref{alg:1} is chosen to be such that the following conditions hold:
    \begin{itemize}
        \item[(i)] $\bar{\Phi}^{\star}_k(x_{\leftside(k) \cup k}; \gamma_{(k, \parent(k))})$ and $\Phi^{\Delta}_{k \to \parent(k)}(x_{\leftside(k) \cup k}; \alpha_{(k, \parent(k))})$ have the same $r_{(k, \parent(k))}$-dimensional column space for every non-root $k$.
        \item[(ii)]  $Z^{\star}_k(\beta_{(k, \child(k))}, x_k; \gamma_{(k, \parent(k))})$ defined in \eqref{eqn: def of Z k} is of rank $r_{(k, \parent(k))}$ for every non-leaf and non-root $k$.
        \item[(iii)] $A^{\Delta}_k(\beta_{(k, \child(k))}; \alpha_{(k, \child(k))})$ has full column rank for every non-leaf $k$.
    \end{itemize}

\end{cond}

If one has oracle access access to \(p^{\star}\), and moreover suppose that computing \(Z^{\star}_{w \to k}\) and \(Z^{\star}_{k}\) are tractable, then one can directly start with the \textsc{SystemForming} step in Algorithm \ref{alg:1} with the noiseless terms \(Z^{\star}_{w \to k}\) and \(Z^{\star}_{k}\) as input. Similar to Theorem \ref{thm:ttns-existence}, under a technical condition for the given sketch functions, one can solve for the tensor cores exactly. We give out the condition in Theorem \ref{thm:ttns-algorithm-recovery}. 

\begin{theorem}
    \label{thm:ttns-algorithm-recovery}
    (Exact recovery of TTNS-Sketch)
    For an underlying distribution $p^{\star}$ satisfying Condition \ref{cond: TTNS ansatz condition}, suppose its tree structure \(T\) and internal bond \(r\) coincides with the input for Algorithm \ref{alg:1}. Let \(\{T_{i}, S_{i}\}_{i = 1}^{d}\) satisfy Condition \ref{cond: Condition for sketch functions}. Then, in Algorithm \ref{alg:1}, assume one has oracle access to $\{Z^{\star}_{w \to k}\}_{w \to k \in E}, \{Z^{\star}_i\}_{i = 1}^{d}$. Then, let
    \[\{A^{\star}_i, B^{\star}_i\}_{i = 1}^{d} \leftarrow \textsc{SystemForming}(\{Z^{\star}_{w \to k}\}_{w \to k \in E}, \{Z^{\star}_i\}_{i = 1}^{d}),\]
    and let the $\{G_{i}^\star\}_{i = 1}^{d}$ be the least-squares solution the following linear system for the variables \(G_{1}, \ldots, G_{d}\):
    \begin{equation}
        \label{eq:alg-CDEs noiseless}
        \begin{aligned}
            G_{k} & =  B^{\star}_{k} \quad \text{if} ~ k ~ \text{is a leaf}, \\
            \textstyle\sum_{\alpha_{(k, \child(k))}}  A^{\star}_{k}(\beta_{(k, \child(k))}, \alpha_{(k, \child(k))})  G_{k}(x_k, \alpha_{(k, \neighbor(k))}) & = B^{\star}_{k} \quad \text{otherwise}, 
        \end{aligned}
    \end{equation}

    Then linear system \eqref{eq:alg-CDEs noiseless} has a unique exact solution $\{G^{\star}_{i}\}_{i = 1}^{d}$, which forms the cores of the TTNS ansatz of \(p^{\star}\) given $T = ([d], E)$ and $\{r_e : e \in E\}$. As a consequence, the output of Algorithm \ref{alg:1} given an empirical distribution \(\hp\) as input satisfies \(\lim_{N \to \infty} \hat G_{k} = G^{\star}_{k}\), i.e. \(\hat G_{k}\) is a consistent estimator of \(G^{\star}_{k}\).
\end{theorem}

\subsection{Introduction to recursive sketch functions}
Storing generic \(S_{i}\) is not possible as its number of possible inputs is exponential in \(d\). As an example, one can consider cases where sketch functions \(\{S_{i}, T_{i}\}_{i =1}^{d}\) have explicit formula, which makes its evaluation easy. Alternatively, one can also use sketch functions which are themselves derived by a TTNS ansatz. We introduce a nice special case called \emph{recursive sketch functions}, which allows for sketch functions more general than those with an analytic formula, but is still tractable for computation.

For recursive sketch functions, there is a natural notion of subgraph TTNS function. We will define it here:
\begin{defn}\label{defn: TTNS on subgraphs}
    (Subgraph TTNS function)
    Suppose that \(f\) is a function satisfying Condition \ref{cond: TTNS ansatz condition} with tree \(T\), and moreover suppose \(f\) admits tensor cores \(\{s_{i}\}_{i = 1}^{d}\) for its TTNS ansatz. Let \(\mathcal{S} \subset V\) and let \(T_{\mathcal{S}} = (\mathcal{S}, E_{\mathcal{S}})\) be the subgraph of \(T\) with vertex set \(\mathcal{S}\). 
    Then \(\{s_{i}\}_{i = 1}^{d}\) and \(T_{\mathcal{S}}\) jointly defines the \emph{subgraph TTNS function} \(f_{\mathcal{S}} \colon \prod_{i \in \mathcal{S}}[n_{i}] \times \prod_{e \in \partial \mathcal{S}}[r_{e}] \to \bR \) by
    \[
    f_{\mathcal{S}}(x_{\mathcal{S}}, \alpha_{\partial \mathcal{S}}) = \sum_{\substack{\alpha_{e} \\ e \in E_{\mathcal{S}}}} \prod_{k \in \mathcal{S}} G_{k}\left(x_{k}, \alpha_{(k, \neighbor(k))}\right).
    \]
    where \(\partial\mathcal{S} :=  \{(v,w) \in E \mid v \in \mathcal{S}, w \not \in \mathcal{S}\}\). 
\end{defn}
In the recursive sketching regime, each sketch function \(T_{k}\) and \(S_{k}\) has a TTNS structure, and is recursively defined by the collection of sketch cores \(\{(t_{i},s_{i})\}_{i=1}^{d}\). We specify such a relationship. Given the convention in section \ref{sec: main algorithm}, one can easily define as follows:
\begin{cond}
    \label{cond: recurisve sketching}
    (Recursive sketching condition)
    Let \(T = (V,E)\) be a tree. Assume \(f\) (resp. \(g\)) are two functions with a TTNS ansatz over \(T\) in the sense of Condition \ref{cond: TTNS ansatz condition}, with internal bond \(l\) (resp. \(m\)) and sketch cores  \(\{s_{i}\}_{i=1}^{d}\) (resp. \(\{t_{i}\}_{i=1}^{d}\)). Then \(S_{k} = f_{\leftside(k)}\) and \(T_{k} = g_{\rightside(k)}\). 
    
    In particular, \(S_{k}\) satisfies a recursive relation
    \begin{equation}
    \label{eqn: S_k under recursive sketching}
        S_{k}(\beta_{(k, \child(k))}, x_{\leftside(k)}) = \prod_{w \in \leftside(k)}\sum_{\beta_{(w, \child(w))}}s_{w}(x_{w}, \beta_{(w, k)}, \beta_{(w, \child(w))})S_{w}(\beta_{(w, \child(w))}, x_{\leftside(w)}),
    \end{equation}
    and so \(S_{k}\) satisfies the factorization structure \eqref{eqn: S_k splits} with its \(S_{w \to k}\) defined by
    \[
    S_{w \to k}(\beta_{(w, k)}, x_{\leftside(w)}) = \sum_{\beta_{(w, \child(w))}}s_{w}(x_{w}, \beta_{(w, k)}, \beta_{(w, \child(w))})S_{w}(\beta_{(w, \child(w))}, x_{\leftside(w)}).
    \]
\end{cond}

Moreover, the recursive definition over the left sketch functions leads to a simplified equation for \(\hat A_{k}\). As a consequence of \eqref{eqn: S_k under recursive sketching}, the terms \(\hat Z_{w \to k}\) and \(\hat Z_{w}\) are now connected:
\begin{equation}
\label{eqn: phi w to k under recursive sketching}
    \hat Z_{w \to k}(\beta_{(w,k)}, \gamma_{(w, k)}) 
    =
    \sum_{x_{w}}\sum_{\substack{ \gamma_{(w, k)}}}
    s_{w}(\beta_{(w,k)}; \beta_{(w, \child(w))},x_{w} ) 
    \hat Z_{w}(\beta_{(w, \child(w))}, x_{w}; \gamma_{(w,k)}).
\end{equation}

Define an intermediate term \(\hat A_{w \to k}\) similar to \eqref{eqn: A w to k and phi w to k def}. As a consequence of \eqref{eqn: phi w to k under recursive sketching},
\begin{align*}
    \hat A_{w \to k}(\beta_{(w, k)}, \alpha_{(w, k)}) 
    = &
    \sum_{\beta_{(w, \child(w))}} \sum_{x_{w}} \sum_{\substack{ \gamma_{(w, k)}}}
    s_{w}(\beta_{(w,k)}; \beta_{(w, \child(w))},x_{w} ) 
    \hat Z_{w}(\beta_{(w, \child(w))}, x_{w}; \gamma_{(w,k)}) \hat Q_{w}(\gamma_{(w,k)}, \alpha_{(w,k)})\\
    = &\sum_{x_{w}}\sum_{ \beta_{(w, \child(w))}  } s_{w}(\beta_{(w,k)}; \beta_{(w, \child(w))},x_{w} ) 
    \hat B_{w}(\beta_{(w, \child(w))}, x_{w};\alpha_{(w,k)}),
\end{align*}
where the second equality is a consequence of Algorithm \ref{alg:1-3}.

As a consequence, we conclude that for recursive sketching one can compute the left-hand side by
\begin{equation}
\label{eqn: A_k under recursive sketching}
    \hat A_{k}(\beta_{(k,\child(k))}, \alpha_{(k, \child(k))})
    =  \prod_{w \in \child(k)}  \sum_{(\beta_{(w, \child(w))}, x_{w})} s_{w}(\beta_{(w,k)}; \beta_{(w, \child(w))},x_{w} ) 
    \hat B_{w}(\beta_{(w, \child(w))}, x_{w};\alpha_{(w,k)}).
\end{equation}

Note that the above result is solely due to the property of the left sketch function. Hence, \eqref{eqn: A_k under recursive sketching} holds true if one replaces \(\hat{A}_{k}, \hat{B}_{k}\) with \(A^{\star}_k, B^{\star}_k\). Moreover, the equation \eqref{eqn: A_k under recursive sketching} is mathematically equivalent to the formula for \(\hat A_{k}\) in Algorithm \ref{alg:1-3}. Nevertheless, it will simplify our subsequent error analysis for Markov sketch function, as one only needs to account for the error in \(\hat B_{k}\).


\subsection{Estimation of target rank \(r\)}
If one has access to a tree \(T\) but not a target rank \(r\), then one can define a noise threshold \(\delta\), and determine \(r_{(k, \parent(k))}\) from the SVD result in Algorithm \ref{alg:1-3} directly. Namely, one checks the singular values in \(\Sigma_{k}\), and sets \(r_{(k, \parent(k))}\) to be the number of singular values above the level set by \(\delta\). 

Hence, for samples from a distribution \(p^{\star}\) with no known structure, one can set \(T = \TCL\) as in Section \ref{sec: topology finding}, and set the internal bond rank by thresholding with \(\delta\) in the sense described above. The result should be reasonable if \(p^{\star}\) satisfies Condition \ref{cond: TTNS ansatz condition} with tractable bond dimension, or if \(p^{\star}\) can be well approximated by a TTNS ansatz. Analysis with approximated target rank is beyond the scope of this paper.

\section{Choice of sketch function}\label{sec: sketch functions}
We begin with Section \ref{sec: main idea of TTNS-Sketch}, which explains that the differences in sketch functions conceptually leads to matching different statistical moments. The rest of the subsections give concrete examples of sketch functions.

\subsection{Connection of sketching to moment matching} \label{sec: main idea of TTNS-Sketch}

In this subsection, we give the sketched-down linear equation another interpretation as enforcing a match between the statistical moment of the output TTNS ansatz to that of the empirical distribution.
Essentially, different sketch functions lead to different statistical moments to match, which conceptually leads to different optimization objectives. 

First, any algorithm for solving for tensor components of a TTNS boils down to attempting to fit the following equation over \(\{G_{i}\}_{i = 1}^{d}\) to the sample, with the end goal of approximating \(p^{\star}\). In terms of an equation, one can write
\begin{equation}\label{eqn: motivational equation for solving TTNS component}
    \sum_{\substack{\alpha_e \\ e \in E }} \prod_{i=1}^{d} G_{i}(x_{i}, \alpha_{(i, \neighbor(i))})
    \approx
      \hp(x_{1}, \ldots, x_{d})
    \approx
     p^{\star}(x_{1}, \ldots, x_{d}),\quad \forall (x_{1}, \ldots, x_{d}).
\end{equation}

In practical settings, \(\hp \approx p^{\star}\) only weakly. Moreover, the above equation relates to the equivalence of two tensors of  \(n^{d}\) entries, but the number of unknown parameters involved is only \(O(d)\). 
To apply sketching, one defines a sketch function \(f(\mu, x_{1}, \ldots, x_{d})\). One then multiplies \eqref{eqn: motivational equation for solving TTNS component} by \(f\) and then sum over the joint \(x_{[d]}\) variable. The sketched down equation becomes
\begin{equation}\label{eqn: sketched equation for solving TTNS component}
    \sum_{x_{[d]} } f(\mu, x_{1}, \ldots, x_{d})\left(\sum_{\substack{\alpha_e \\ e \in E }} \prod_{i=1}^{d} G_{i}(x_{i}, \alpha_{(i, \neighbor(i))})\right)
    =
    \sum_{i = 1}^{N} f\left(\mu, y_1^{(i)}, \ldots, y_d^{(i)}\right)
    \approx
    \mathbb{E}_{X \sim p^{\star}} \left[f\left(\mu, X\right)\right], \quad \forall \mu.
\end{equation}
where the first approximation sign in \eqref{eqn: motivational equation for solving TTNS component} is swapped with an equality sign, which is meant to signal that one then solves for the tensor cores using this equation. 
As for the approximate sign in \eqref{eqn: sketched equation for solving TTNS component}, the design of \(f\) will be typically such that the variance \( f\left(\mu, X\right)\) is \(O(1)\) or grows slowly with \(d\), which is why the approximation will be reasonable according to the law of large numbers. 

Moreover, one can let \(\theta\) stand for the TTNS tensor cores \(\{G_{i}\}_{i =1}^{d}\), and let \(p_{\theta}\) stand for the probability distribution obtained from the TTNS ansatz under such cores. The equation \eqref{eqn: sketched equation for solving TTNS component} is equivalent to
\begin{equation}\label{eqn: Sketching is moment matching}
    \mathbb{E}_{X \sim p_{\theta}} \left[f\left(\mu, X\right)\right]
    =
    \mathbb{E}_{X \sim \hp} \left[f\left(\mu, X\right)\right]
    \approx
    \mathbb{E}_{X \sim p^{\star}} \left[f\left(\mu, X\right)\right], \quad \forall \mu.
\end{equation}

In summary, the solution is such that \(p_{\theta}\) is close to \(\hp\) in terms of statistical moments \(\mathbb{E}_{X} \left[f\left(\mu, X\right)\right]\).
For the connection to TTNS-Sketch, we consider a simple case where one has already solved for \(\{G_{i}\}_{i \not = k}\). Consider a sketch function \(f_{k}\) by first defining its corresponding joint variable \(\mu\) by \(\mu = (\beta_{(k, \child(k))}, \gamma_{(k, \parent(k))}, \iota)\). Then, let
\[
f_{k}(\beta_{(k, \child(k))}, \gamma_{(k, \parent(k))}, \iota, x_{[d]}) = S_{k}(\beta_{(k,\child(k))}, x_{\leftside(k)})\mathbf{1}\left[\iota = x_{k}\right]T_{k}(x_{\rightside(k)}, \gamma_{(k, \parent(k))}).
\]

As a result, one has
\[
    \mathbb{E}_{X \sim \hp} \left[f_{k}\left(\mu, X\right)\right] = \hat{Z}_{k}(\beta_{(k, \child(k))}, \iota, \gamma_{(k, \parent(k))}),
\]
and so the sketched-down equation \eqref{eqn: Sketching is moment matching} tries to enforce a match between \(\mathbb{E}_{X \sim p_\theta} \left[f_k\left(\mu, X\right)\right]\) and \(Z^\star_k\). Different sketch functions thus leads to different sketches \(Z^\star_k\) to match.



\subsection{Markov sketch function}\label{section:markov sketch}
A Markov sketch function allows the TTNS procedure to essentially solve for marginal distribution information for each node \(k\) around its neighbors, which we will show with the definition of its sketch function.

For the Markov sketch function, one has \(l_{k} := n_{\parent(k)}\), and the right-sketch function \(T_{k}\) is defined by
\begin{equation}
\label{eq:markov-right-sketch}
    T_{k}(x_{\rightside(k)}, \gamma_{(k, \parent(k))}) = \mathbf{1}\left[x_{\parent(k)} =  \gamma_{(k, \parent(k))}\right].
\end{equation}

As for the left-sketch function, the form is similar, but it is complicated by the fact that a tree node can have multiple child nodes. Likewise, \(m_{(w,k)} := n_{w}\), and
\begin{equation}
\label{eq:markov-left-sketch}
    S_{k}(\beta_{(k, \child(k))}, x_{\leftside(k)}) = \prod_{w \in \child(k)} \mathbf{1}\left[x_{w} = \beta_{(w,k)}\right].
\end{equation}

By applying left-sketch and right sketch, one has
\begin{equation}\label{eq:tildepdef_markov}
    Z^{\star}_{k}(\beta_{(k, \child(k))}, x_{k}, \gamma_{(k, \parent(k))}) = \mathbb{P}_{X \sim p^{\star}}\left[X_{w} = \beta_{(w, k)} \forall w \in \child(k), X_{k} = x_{k}, X_{\parent(k)} = \gamma_{k}\right],
\end{equation}
which is then the marginal distribution on the subset of nodes \(\neighbor(k) \cup \{k\} = \{v \in V \mid \mathrm{dist}(v, k) \leq 1\}\). 

We also use
\begin{equation}
    (\mathcal{M}_{\mathcal{S}}p)(x_\mathcal{S}) := \mathbb{P}_{X \sim p}\left[X_{v} = x_{v}, \forall v \in \mathcal{S}\right]
\end{equation}
to denote the marginalization of $p$ to the variables given by the index set $\mathcal{S}$, which is a $\vert \mathcal{S} \vert$-dimensional function.

Due to the construction, the joint variables \((\beta_{(k, \child(k))}, \gamma_{(k, \parent(k))})\) each has a natural correspondence to a node neighboring \(k\). By identifying \(\beta_{(w,k)} = x_{w}\) and \(\gamma_{(k,\parent(k))} = x_{\parent(k)}\), one has the following entry-wise equality
\begin{equation*}
    Z^{\star}_{k}(\beta_{(k, \child(k))}, x_{k}, \gamma_{(k, \parent(k))}) = \mathcal{M}_{\mathcal{S}_{k}}p^{\star}(x_{\child(k)}, x_{k}, x_{\parent(k)}),
\end{equation*}
where \(\mathcal{S}_{k} = \neighbor(k)\cup \{k\}\). 

By the description in Section \ref{sec: main idea of TTNS-Sketch}, with Markov sketch function, TTNS-Sketch essentially tries to fit \(\mathcal{M}_{\neighbor(k)\cup k}\hp\) over all \(k \in V\). 
One can show that the Markov sketch function allows exact recovery for tree-based graphical models. We summarize the result in Lemma \ref{lem: ttns-markov-exact}.
\begin{lemma}
\label{lem: ttns-markov-exact}
    Assume that $p^{\star}$ is a graphical model over the tree structure $T = ([d], E)$. Then \(p^{\star}\) satisfies Condition \ref{cond: TTNS ansatz condition} with the tree structure \(T\) and \(r_{e} = n\) for any \(e \in E\). Moreover, sketches in \eqref{eq:markov-right-sketch} and \eqref{eq:markov-left-sketch} satisfies the condition in Theorem \ref{thm:ttns-algorithm-recovery}.
\end{lemma}



\subsection{Higher order Markov sketch function}\label{sec: high-order markov sketch}
One can use high-order Markov sketch functions to solve for marginal probability information over more nodes than in the previous cases. Let \(\mathcal{S}_{k} \subset V\) be a choice of nodes of interest to \(k\). By a suitable change of the sketch functions $T_k$ and $S_{k}$ in Section \ref{section:markov sketch}, one can make it so that one has the following entry-wise equality
\begin{align*}
     Z_{k}^{\star}
    =
    \mathcal{M}_{\mathcal{S}_{k}}  p^{\star}; \quad  Z_{w \to k}^{\star} = \mathcal{M}_{(\mathcal{S}_{k} \cup \leftside(k)) \cup (\mathcal{S}_{w} \cap \rightside(w))}  p^{\star}
\end{align*}

If one wishes to ensure the high-order Markov sketching function satisfies recursive sketching, one needs the following constraint:
\begin{equation}\label{eq: nbhd recursiveness constraint}
    \{k\} \subset \mathcal{S}_{k} \subset \cup_{w \in \child(k)} \mathcal{S}_{w}.
\end{equation}

For an example, for any integer \(L \geq 1\), one sets \(L\) as a distance cutoff and let
\begin{equation}\label{eq: nbhd with cutoff}
     \mathcal{S}_{k} = \{v \in V \mid \mathrm{dist}(v, k) \leq L \}.
\end{equation}

By triangle inequality, the construction in \eqref{eq: nbhd with cutoff} satisfies \eqref{eq: nbhd recursiveness constraint}, and one obtains the Markov sketching function when \(L = 1\). For the choice of neighborhood in \eqref{eq: nbhd with cutoff}, we refer to the corresponding sketch function as \(L\)-Markov sketch function. In particular, \(2\)-Markov sketch function will play an important role in numerical experiments.

\subsection{Perturbative sketching}
The idea of perturbative sketching is to use recursive random projection to form the sketch functions. In Condition \ref{cond: perturbative sketching symmetry}-\ref{cond: perturbative sketching cores}, we define the structural assumption of perturbative sketching. In Theorem \ref{thm: perturbative sketching form}, we give a structural theorem for perturbative sketching, which shows that the sketch \(Z_k^\star\) is a power series of tensors. Moreover, each term in the power series corresponds to a random projection of a marginal distribution of \(p^{\star}\), i.e. tensor of the form \(\mathcal{M}_{\mathcal{S}}  p^{\star}\) for a subset \(\mathcal{S} \subset [d]\).

In Condition \ref{cond: perturbative sketching symmetry}, we make a significant simplification to recursive sketching by unifying left and right sketching to an equal footing, which allows for a cleaner structural analysis. In Condition \ref{cond: perturbative sketching cores}, we make the assumption that each sketch core is made up of an all one tensor plus a perturbation term. 
\begin{cond}
\label{cond: perturbative sketching symmetry}
    (Directional symmetry for perturbative sketching)
    Assume the sketch functions \(\{T_{i}, S_{i}\}\) satisfies Condition \ref{cond: recurisve sketching}. Furthermore, one lets \(l_{e} = m_{e}\) for any \(e \in E\) and \(t_{i} = s_{i}\) for any \(i \in [d]\).
\end{cond}
\begin{cond}
    (perturbative structure for sketch cores)
    \label{cond: perturbative sketching cores}
    Fix a constant \(\epsilon > 0\) as the \emph{perturbative scale}. The tensor core \(s_{k}\) is defined by tensors \(O_{k}\) and \(\Delta_{k}\) of the form:
    \begin{equation}\label{eq: decomposition of perturbative core}
        s_{k}(x_{k}, \beta_{(k, \neighbor(k))}) = O_{k}(x_{k}, \beta_{(k, \neighbor(k))}) + \epsilon \Delta_{k}(x_{k}, \beta_{(k, \neighbor(k))}),
    \end{equation}
    and moreover \(O_{k}\) is the all one tensor satisfying
    \[
    O_{k}(x_{k}, \beta_{(k, \neighbor(k))}) = 1.
    \]
\end{cond}
For a concrete example, if the perturbation is formed by entry-wise i.i.d. random variable, one can use the following line in MATLAB to define a perturbative sketch core:
\begin{verbatim}
    s_k = ones(size(s_k)) + epsilon*rand(size(s_k)).
\end{verbatim}

Due to Condition \ref{cond: perturbative sketching symmetry}, it follows that \( Z^\star_{k}\) only depends on the sketch cores \(\{s_{i}\}_{i \not = k}\). Therefore, one can identify \(\gamma_{(k, \parent(k))}\) with \(\beta_{(k, \parent(k))}\). By simple algebra, one can derive the following result:
\begin{theorem}\label{thm: perturbative sketching form}
    (Structure theorem of perturbative sketching)
    Assume that Condition \ref{cond: perturbative sketching symmetry}-\ref{cond: perturbative sketching cores} are satisfied for the chosen sketch function. 
    For \(k \in [d]\) and \(\mathcal{S} \subset [d] - \{k\}\), let \(T_{\mathcal{S}} = (\mathcal{S}, E_{\mathcal{S}})\) be the subgraph of \(T\) with vertex set \(\mathcal{S}\). As in Definition \ref{defn: TTNS on subgraphs}, define \(\Delta_{\mathcal{S}}\) by
    \begin{equation}
    \label{eqn: def of Delta S}
        \Delta_{\mathcal{S}}(x_{\mathcal{S}}, \beta_{\partial \mathcal{S}})  :=  \sum_{\substack{\beta_{e} \\ e \in E_{\mathcal{S}}}}\prod_{i \in \mathcal{S} }\Delta_{i}(x_{i}, \beta_{(i, \neighbor(i))}).
    \end{equation}

    Then the following equation holds for \(Z_{k}^{\star}\):
    \begin{equation}\label{eq: structural form}
        Z^{\star}_{k}(x_{k}, \beta_{(k, \neighbor(k))}) = \sum_{l = 0}^{d-1} \epsilon^{l}\sum_{\mathcal{S} \subset [d] - \{k\}, |\mathcal{S}| = l}Z^{\star}_{k; \mathcal{S}}(x_{k}, \beta_{(k, \neighbor(k))}),
    \end{equation}
    where
    \begin{equation}
    \label{eq: structural form of  Z star kS}
        Z^{\star}_{k; \mathcal{S}}(x_{k}, \beta_{(k, \neighbor(k))}) = \sum_{\beta_{e}, k \not \in e}\left(\sum_{x_{\mathcal{S}}}\mathcal{M}_{\mathcal{S} \cup \{k\}}  p^{\star}(x_{k},x_{\mathcal{S}})\Delta_{\mathcal{S}}(x_{\mathcal{S}}, \beta_{\partial \mathcal{S}})\right).
    \end{equation}

\end{theorem}

The proof is simple and left in the Appendix. There are a few consequence of Theorem \ref{thm: perturbative sketching form}. First, each \(Z^{\star}_{k; \mathcal{S}}\) is a projection of the marginal distribution tensor \(\mathcal{M}_{\mathcal{S} \cup \{k\}} p^{\star}\). Second, in \eqref{eq: structural form}, terms corresponding to \(\mathcal{S}\) is scaled by \(\epsilon^{|\mathcal{S}|}\), which itself means that contribution of \(\mathcal{S}\) with large cardinality is insignificant. By the description in Section \ref{sec: main idea of TTNS-Sketch}, with perturbative sketching function, TTNS-Sketch essentially tries to fit over all \(\mathcal{M}_{\mathcal{S}}p^{\star}\), and a \(\epsilon^{|\mathcal{S}|}\) factor is placed to ensure \(\hat Z_{k}\) stabilizes quickly.

Moreover, if \(\mathcal{S} \cup \{k\}\) is not a connected component of \(T\), then according to \eqref{eq: structural form of  Z star kS}, \(Z^{\star}_{k; \mathcal{S}}\) does not vary with \(x_{s}\) or \(\beta_{(k, \neighbor(k))}\). Such \(Z^{\star}_{k; \mathcal{S}}\) has no contribution to \(Z^{\star}_{k}\) except for on a linear subspace spanned by an all-one tensor.
Hence, the subsets \(\mathcal{S}\) with nontrivial contribution has a one-to-one correspondence with connected components of \(T\) which contains \(k\). 

Importantly, the number of connected components of \(T\) both containing \(k\) and having a small cardinality only depends on the local topology of \(T\) around \(k\). In the numerical examples, one can see that perturbative sketching performs quite well if the interaction is local, and it is more adaptable than Markov sketch functions or high-order Markov sketch functions. 

In our numerical experiments, we keep a fixed design on the the perturbative scale \(\epsilon\), but there is slight numerical benefit in tuning \(\epsilon\). In general, the parameter \(\epsilon\) should decrease with sample size. The decay rate in \(N\) depends on how much marginal distribution information is needed to determine \(G_{k}^{\star}\). In practice, one can simply choose a decay rate of \(\epsilon  :=  cN^{-f}\), with the parameter \(c,f\) determined by cross validation.

As a remark, theoretically by applying Section \ref{sec:sample-comlexity}, one would be able to derive the dependence of sample complexity on \(\epsilon\), and tune \(\epsilon\) accordingly. To capture the dependence on \(\epsilon\) accurately, the reader is advised to use the tighter original version of the Matrix Bernstein inequality (cf. Theorem 6.1.1 in \cite{tropp2015introduction}). The rigorous account on the choice of \(\epsilon\) will be left for future works.

\section{Sample complexity bound of TTNS-Sketch}\label{sec:sample-comlexity}

This section gives an upper bound for the sample complexity of TTNS-Sketch. Within this section, we only consider the case where the sketch functions satisfies recursive sketching in the sense of Condition \ref{cond: recurisve sketching}. In this simplified case, one can use the alternative definition of \(A_{k}\) in \eqref{eqn: A_k under recursive sketching}. As an application, we obtain a sample complexity bound to the simple case where \(p^{\star}\) is a graphical model over a tree \(T\), and the sketching function is the Markov sketch function. Here is an informal version of the obtained generalized sample complexity theorem:
\begin{theorem}(Informal statement of Theorem \ref{thm:sample-complexity})

Let $p^\star \colon [n_1] \times \cdots \times [n_d] \to \mathbb{R}$ satisfy the TTNS assumption in Condition \ref{cond: TTNS ansatz condition}. Let sketch function $\{T_i, S_i\}_{i=1}^{d}$ which satisfy recursive sketching in Condition \ref{cond: recurisve sketching}. Let \(\{\hat G_{i}\}_{i=1}^{d}\) be the output of Algorithm \ref{alg:1}, and let \(\PTS\) denote the TTNS ansatz formed by \(\{\hat G_{i}\}_{i=1}^{d}\). There exists problem-dependent constant \(\zeta, L\) such that, for $\eta \in (0, 1)$ and $\epsilon \in (0, 1)$, if
\begin{equation*}
    N \ge \frac{18L^2d^2 + 4L\epsilon \zeta d}{\zeta^2 \epsilon^2}\log{\left(\frac{(ln + m)d}{\eta}\right)},
\end{equation*}
and the constants are defined as follows:
\begin{itemize}
    \item $l = \max_{k \in [d]} l_k$, where $l_k = \prod_{w \in \child(k)}l_{(w, k)}$
    \item $m = \max_{k \in [d]} m_k$, where $m_{k} = m_{(k , \parent(k))}$.
    \item $n = \max_{k \in [d]} n_k$.
\end{itemize}
Then with probability at least $1 - \eta$ one has
\begin{equation*}
    \frac{\|\PTS - p^{\star}\|_\infty}{\normi{G_1^\star} \cdots \normi{G_d^\star}} \le \epsilon,
\end{equation*}
where the notation \(\normi{G^\star_{i}}\) is to be defined later. Moreover, the sample complexity upper bound for \(N\) does not explicitly depend on \(\{n_{k}\}_{k \in [d]}, \{m_{e}, n_{e}\}_{e \in E}\) except in the log factor.
\end{theorem}

We give a summary of organization of this section. In Section \ref{sec: 3-tensor preliminrary}, we introduce notations and conventions which are important for sample complexity analysis. In Section \ref{sec: local error to global}, we prove small perturbations of the cores lead to small perturbations of the obtained TTNS ansatz. In Section \ref{sec: local error and sample complexity}, we prove that small error in the estimator \(\hat{Z}_k\) leads to small perturbation of the cores, leading to an upper bound for the sample complexity of TTNS-Sketch in Theorem \ref{thm:sample-complexity}. In Section \ref{sec: proof of sample complexity results}, we give a proof of all the lemmas and corollaries. In Section \ref{sec: remark to l_1 sample complexity}, we remark how our derived results can be extended sample complexity bounds for total variation distance.

\subsection{Preliminaries}
\label{sec: 3-tensor preliminrary}

In what follows, for a given vector $v$, let $\|v\|$ and $\|v\|_\infty$ denote its Euclidean norm and its supremum norm, respectively. For any matrix $M$, denote its spectral norm, Frobenius norm, and the $r$-th singular value by $\|M\|$, $\|M\|_F$, and $\sigma_r(M)$, respectively. Also, for a generic tensor \(p\), let $\|p\|_\infty$ denote the largest absolute value of the entries of $p$. Lastly, the orthogonal group in dimension $r$ is denoted by $\operatorname {O}(r)$.

A mathematical structure important in this section is 3-tensors. Similar to unfolding matrix in Definition \ref{def: unfolding matrix}, the 3-tensors we typically use come from viewing high-dimensional tensors in terms of \(3\)-tensors by grouping joint variables:
\begin{defn}(Unfolding 3-tensor Notation)
    For a generic \(d\)-dimensional tensor $p \colon [n_1] \times \cdots \times [n_d] \to \bR$ and for three disjoint subsets \(\mathcal{U}, \mathcal{V}, \mathcal{W}\) with $\mathcal{U} \cup \mathcal{V} \cup \mathcal{W} = [d]$, we define the corresponding \emph{unfolding 3-tensor} by $p(x_{\mathcal{U}}; x_{\mathcal{V}}; x_{\mathcal{W}})$. The 3-tensor $p(x_{\mathcal{U}}; x_{\mathcal{V}}; x_{\mathcal{W}})$ is of size \(\left[\prod_{i \in \mathcal{U}}n_{i}\right] \times \left[\prod_{j \in \mathcal{V}}n_{j}\right] \times \left[\prod_{k \in \mathcal{W}}n_{k}\right] \to \bR\). 
\end{defn}

It is helpful to introduce a slice of the 3-tensor. In our convention, we only need to consider taking slice at the second component:
\begin{defn}
\label{def: Matrix slice of 3-tensor} (Middle index slice of 3-tensor)
For any 3-tensor $G \colon [r_{1}] \times [n_{1}] \times [r_{2}] \to \mathbb{R}$, we use $G(\cdot, x, \cdot) \colon [r_{1}] \times [r_{2}] \to \mathbb{R}$ to denote an \(r_{1} \times r_{2}\) matrix obtained by fixing the second slot of \(G\) to be \(x\). 
\end{defn}

In Definition \ref{def: 3 tensors norm}-\ref{def: 3 tensors product}, we introduce a new norm and two operations for $3$-tensors. 
\begin{defn}\label{def: 3 tensors norm}(\(\normi{\cdot}\) norm for 3-tensors)
Define the norm \(\normi{G}\) by
\begin{equation}\label{eqn: 3-tensor norm}
    \normi{G}  :=  \max_{x \in [n_{1}]} \|G(\cdot, x, \cdot)\|.
\end{equation}

\end{defn}
\begin{defn}\label{def: 3 tensors contraction}(contraction operator for 3-tensors)
Let \(G \colon [r_{1}] \times [n_{1}] \times [r_{2}] \to \mathbb{R}, G' \colon [r_{3}] \times [n_{2}] \times [r_{4}] \to \mathbb{R}\) be two 3-tensors. Under the assumption \(r_{2} = r_{3}\), define the 3-tensor \(G \circ G' \colon [r_{1}] \times [n_{1} \times n_{2}] \times [r_{4}] \to \mathbb{R}\) by
\begin{equation}\label{eqn: circ operation}
    G \circ G'(\alpha; (x, y); \gamma) = \sum_{\beta \in [r_{2}]} G(\alpha, x, \beta) G'(\beta, y, \gamma).
\end{equation}


\end{defn}

\begin{defn}\label{def: 3 tensors product}(tensor product operator for 3-tensors)
Let \(G \colon [r_{1}] \times [n_{1}] \times [r_{2}] \to \mathbb{R}, G' \colon [r_{3}] \times [n_{2}] \times [r_{4}] \to \mathbb{R}\) be two 3-tensors. Define the 3-tensor \(G \otimes G' \colon [r_{1} \times r_{3}] \times [n_{1} \times n_{2}] \times [r_{2} \times r_{4}] \to \mathbb{R}\) by
\begin{equation}\label{eqn: prod operation}
    G \otimes G'((\alpha, \beta); (x, y); (\gamma, \theta)) =  G(\alpha, x, \beta) G'(\gamma, y, \theta).
\end{equation}

\end{defn}

We summarize the simple properties of the defined operation in Lemma \ref{lem: 3-tensor basic results}, which will be useful for our derivations:
\begin{lemma}
\label{lem: 3-tensor basic results}
The following results hold:
\begin{enumerate}[(i)]
    \item  Associativity of \(\circ\) holds:
    \begin{equation}
        (G \circ G') \circ G'' = G \circ (G' \circ G'').
    \end{equation}
    \item Associativity of \(\otimes\) holds:
    \begin{equation}
         (G \otimes G') \otimes G'' = G \otimes (G' \otimes G'').
    \end{equation}
    \item Inequality of \(\circ\) under \(\normi{\cdot}\) norm:
    \begin{equation}
    \label{eqn: circ sub-multiplicative}
    \normi{G \circ G'} \leq \normi{G} \cdot \normi{G'} 
    \end{equation}
    \item Equality of \(\otimes\) under \(\normi{\cdot}\) norm:
    \begin{equation}
    \label{eqn: otimes multiplicative}
    \normi{G \otimes G'} = \normi{G} \cdot \normi{G'}
    \end{equation}
    \item 
    \label{enumerate: three tensor spectral norm bound}
    For a three tensor  \(G(\alpha, x, \beta) \colon [r_{1}] \times [n] \times [r_{2}] \to \mathbb{R}\), denote \(G(\alpha, x; \beta)\colon [r_{1}n] \times [r_{2}] \to \mathbb{R}\) as the unfolding matrix by grouping the first and second index of \(G\). One has
    \begin{equation}
        \normi{G} \le \|G(\alpha, x; \beta)\| \le n\normi{G}
    \end{equation}
\end{enumerate}
\end{lemma}
As a consequence of associativity, given any collection of 3-tensors \(\{G_{i}\}_{i=1}^{d}\), one can define \(G_{1} \otimes G_{2} \otimes \ldots \otimes G_{d}\). Moreover, if the collection is such that the size of the third index of \(G_{i}\) coincides with the first index of \(G_{i+1}\), then one can naturally define the 3-tensor \(G_{1} \circ G_{2} \circ \ldots \circ G_{d}\).

\subsection{3-tensor structure for TTNS}
\label{sec: local error to global}

For cleaner analysis, one often gives unfolding matrices a 3-tensor structure:
\begin{defn}\label{def: matrix 3-tensor structure}(3-tensor structure for unfolding matrix)
    Consider a generic \(D\)-dimensional tensor $f \colon [n_1] \times \cdots \times [n_D] \to \bR$. Moreover, suppose one picks a disjoint union \(\mathcal{U} \cup \mathcal{V} = [D]\) and forms an unfolding matrix \(f(x_{\mathcal{U}}; x_{\mathcal{V}})\) in the sense of Definition \ref{def: unfolding matrix}. Define \(f(x_{\mathcal{U}}; 1; x_{\mathcal{V}})\) as the 3-tensor of size \(\left[\prod_{i \in \mathcal{U}}n_{i}\right] \times \left\{1\right\} \times \left[\prod_{j \in \mathcal{V}}n_{j}\right] \to \bR\). One likewise defines 3-tensor structure of \(f(1; x_{\mathcal{U}}; x_{\mathcal{V}})\), whose first index is of size \(1\), and \(f( x_{\mathcal{U}}; x_{\mathcal{V}}; 1)\), whose third index is of size \(1\).
\end{defn}

A tensor core from a TTNS ansatz has a default 3-tensor view, whereby the indices are grouped according to tree topology:
\begin{defn}\label{def: 3-tensor view of TTNS tensor core}
    (3-tensor structure for TTNS tensor cores)
    Suppose a tensor \(p\) is defined by a collection of tensor cores $\{G_i\}_{i = 1}^{d}$ in the sense of Definition \ref{def: TTNS short def}. 
    
    If \(k\) is neither a root node nor a leaf node, then  $G_{k}$ is viewed with the 3-tensor unfolding structure  $$G_k(\alpha_{(k, \child(k))}; x_k; \alpha_{(k, \parent(k))}) \colon \left[\prod_{w \in \child(k)} r_{(w, k)}\right] \times [n_{k}] \times [r_{(k, \parent(k))}] \to \bR.$$
    
    If \(k\) is a leaf node, then 
    $G_{k}$ is viewed with the 3-tensor unfolding structure  $$G_k(1 ; x_k; \alpha_{(k, \parent(k))}) \colon \left\{1\right\} \times [n_{k}] \times [r_{(k, \parent(k))}] \to \bR.$$
    
    If \(k\) is the root node, then 
    $G_{k}$ is viewed with the 3-tensor unfolding structure  $$G_k(\alpha_{(k, \child(k))} ; x_k; 1) \colon \left[\prod_{w \in \child(k)} r_{(w, k)}\right] \times [n_{k}] \times \left\{1\right\} \to \bR.$$
\end{defn}

With this design of norms, one can prove the following result by simple algebra:
\begin{lemma}
\label{lem:three-norm-contraction-bound}
Suppose a generic tensor \(p\colon [n_1] \times \cdots \times [n_d] \to \bR\) is defined by a collection of tensor cores $\{G_i\}_{i = 1}^{d}$ in the sense of Definition \ref{def: TTNS short def}. Moreover, suppose the tensor cores are viewed by 3-tensor structures as in Definition \ref{def: 3-tensor view of TTNS tensor core}. Then
\begin{equation*}
    \|p\|_\infty \le \prod_{i = 1}^{d} \normi{G_i}. 
\end{equation*}
\end{lemma}
Using Lemma \ref{lem:three-norm-contraction-bound}, one can bound global errors by errors in tensor cores:
\begin{lemma}
\label{lem:three-norm-total-contraction-perturbation}
In Lemma \ref{lem:three-norm-contraction-bound}, let $\Delta G_k$ be a perturbation of $G_k$. Define a tensor \(p'\colon [n_1] \times \cdots \times [n_d] \to \bR\) by tensor cores $\{G_k + \Delta G_k\}_{k = 1}^{d}$ in the sense of Definition \ref{def: TTNS short def}, with the tree topology \(T\) and the internal rank \(\{r_{e}\}_{e \in E}\) the same as that of \(p\). Suppose $\normi{\Delta G_k} \le \delta_k \normi{G_k}$ for all $k \in [d]$, and set \(\Delta p  :=  p' - p\). Then,
\begin{equation*}
    \|\Delta p\|_\infty 
    \le
    \normi{G_1} \cdots \normi{G_d} \left(\sum_{i = 1}^{d} \delta_i\right) \exp\left(\sum_{i = 1}^{d} \delta_i\right).
\end{equation*}
If $\max_{k \in [d]} \delta_k \le \epsilon / (3 d)$ for some fixed $\epsilon \in (0, 1)$, 
\begin{equation}
\label{eqn:three-norm-total-contraction-perturbation}
    \frac{\|\Delta p\|_\infty}{\normi{G_1} \cdots \normi{G_d}} \le \epsilon.
\end{equation}
\end{lemma}

\subsection{Derivation for sample complexity of TTNS-Sketch}
\label{sec: local error and sample complexity}

We first give a lemma which bounds the perturbation of solutions of a linear equation $AX = B$, where in particular $X,B$ are two 3-tensors viewed under an unfolding matrix. This result will be the main building block to form our subsequent error analysis:
\begin{lemma}
\label{lem:tensor-eq-perturbation}
Consider a matrix \(A^{\star}(\beta, \alpha) \in \bR^{l \times r}\) with $\mathrm{rank}(A^{\star}) = r \le l$ and a 3-tensor \(B^{\star}(\beta, x, \gamma)\in \mathbb{R}^{l \times n \times m}\) with unfolding matrix structure \(B^{\star}(\beta; (x, \gamma)) \in \mathbb{R}^{l\times (nm)}\). Let \(X^{\star}(\alpha, x, \gamma) \in \bR^{r \times n \times m}\) be the 3-tensor with an unfolding matrix view \(X^{\star}(\alpha; (x, \gamma)) \in \mathbb{R}^{r\times (nm)}\) which uniquely solves the linear equation \(A^{\star}X = B^{\star}\) in the sense of least squares:
\begin{equation*}
    \sum_{\alpha}A^{\star}(\beta, \alpha)X(\alpha, (x, \gamma)) = B^{\star}(\beta, (x, \gamma)).
\end{equation*}

Moreover, let $\Delta B^{\star}\in \mathbb{R}^{l \times n \times m}$ be a perturbation of $B^{\star}$, and let $\Delta A^{\star}\in \mathbb{R}^{l \times n \times m}$ be a perturbation of $A^{\star}$ with $\|\left(A^{\star}\right)^\dagger\| \|\Delta A^{\star}\| < 1$ so that $\mathrm{rank}(A^{\star} + \Delta A^{\star}) = n$. Then, let $\Delta X^{\star}$ be a 3-tensor so that $X^{\star} + \Delta X^{\star}$ which uniquely solves the linear equation \((A^{\star}+\Delta A^{\star})X = (B^{\star}+\Delta B^{\star})\) in the sense of least squares:
\begin{equation*}
    \sum_{\alpha} (A^{\star} + \Delta A^{\star})(\beta, \alpha)X(\alpha, (x, \gamma)) = (B^{\star} + \Delta B^{\star})(\beta, (x, \gamma))
\end{equation*}

Under the unfolding matrix structure, suppose the column space of $B^{\star}$ is contained in that of $A^{\star}$, i.e. $X^{\star}$ solves \(A^{\star}X = B^{\star}\) exactly, one has
\begin{equation}
\label{eqn: Delta X bound tight version}
    \normi{\Delta X^{\star}} \le \frac{\|\left(A^{\star}\right)^\dagger\|}{1 - \|\left(A^{\star}\right)^\dagger\| \|\Delta A^{\star}\|} \left( \|\Delta A^{\star}\| \normi{X^{\star}} +  \normi{\Delta B^{\star}}\right).
\end{equation}

 In particular, if $\normi{X^{\star}}\geq\chi >0$ for some constant $\chi,$ and $\Delta A^{\star}$ satisfies  $\|\left(A^{\star}\right)^\dagger\| \|\Delta A^{\star}\| \le 1 / 2,$ then 
\begin{equation}
\label{eqn: Delta X bound relaxed version}
    \frac{\normi{\Delta X^{\star}}}{\normi{X^{\star}}} \le {2 \|\left(A^{\star}\right)^\dagger\|}\,\left(\|\Delta A^{\star}\| + \chi^{-1}\normi{\Delta B^{\star}}\right).
\end{equation}
\end{lemma}

For our use case of Lemma \ref{lem:tensor-eq-perturbation}, the coefficient matrix is viewed as a Kronecker product of some smaller matrices, e.g. \(A^{\star}_{k}\) is formed by $\{A^{\star}_{w \to k}\}_{w \in \child(k)}$. One can bound the $\|\Delta A^{\star}\|$ in this case, as the following lemma shows:
\begin{lemma}
\label{lem:kronecker_product_bound}
Consider a collection of matrices $\{E_i,C_i\}_{i \in [n]}$ such that $E_{i}, C_{i}$ are of the same shape. Moreover, let $\|C_i\| \le 1, \|E_i\| \le \delta_i$. Then
\begin{equation*}
    \left\|\bigotimes_{i = 1}^{n} (C_i + E_i) - \bigotimes_{i = 1}^{n} C_i\right\| 
    \le 
    \left(\sum_{i = 1}^{n} \delta_i\right) \exp\left(\sum_{i = 1}^{n} \delta_i\right).
\end{equation*}
\end{lemma}

Lemma \ref{lem:tensor-eq-perturbation} and \ref{lem:kronecker_product_bound}
 leads to the proof strategy for obtaining sample complexity. With the particular perturbation \(\Delta p^{\star}  :=  \hp - p^{\star}\), the terms \(\hat A_{k}\) and \(\hat B_{k}\) from the Algorithm \ref{alg:1} satisfies \(B_{k}^{\star} + \Delta B^{\star}_{k} = \hat B_{k}\) and \(A_{k}^{\star} + \Delta A^{\star}_{k} = \hat A_{k}\). The least-squares solution $G^{\star}_k + \Delta G^{\star}_k$ to the perturbed equation is thus the actual output \(\hat G_{k}\) from Algorithm \ref{alg:1}. 

However, due to the SVD step that is involved in obtaining \(\hat A_{k}\) and \(\hat B_{k}\), one can only bound the sample estimation error in terms of the following alternative metric:
\begin{defn}(\(\mathrm{dist}(\cdot, \cdot)\) operator for matrices)
For any matrices $B, B^{\star} \in \mathbb{R}^{n \times m}$, define 
\begin{equation*}
    \mathrm{dist}(B, B^{\star})  :=  \min_{R \in \operatorname {O}(m)} \|B - B^{\star} R\|.
\end{equation*}
\end{defn}

In other words, the finite sample estimate of \(\hat A_{k}, \hat B_{k}\) could be closer to a rotation of \(A^{\star}_{k}, B^{\star}_{k}\), which we will denote as \(A^{\circ}_{k}, B^{\circ}_{k}\). An error bound for of the type \(\mathrm{dist}(B, B^{\star})\) exists through Wedin theorem, and thus the magnitude of \(B^{\circ}_{k} - \hat B_{k}\) can be bounded, and \(A^{\circ}_{k} - \hat A_{k}\) is bounded via \eqref{eqn: A_k under recursive sketching}. In Corollary \ref{cor: cor of Wedin}-\ref{cor: cor of Matrix Bernstein}, we write out relaxed version of Wedin theorem and the Matrix Bernstein inequality, which we will use to analyze error terms of the type \(\|\Delta Z_{k}^{\star}\|, \|\Delta B_{k}^{\circ}\|\), respectively. As a summary of all previous result, in Theorem \ref{thm:rotated-perturbation}, we form a quite technical proof bounding the error on the rotated cores by the sample estimation error in sketching. In Theorem \ref{thm:sample-complexity}, we form the sample complexity of TTNS-Sketch.

\begin{cor}
\label{cor: cor of Wedin}(Corollary to Wedin theorem, cf. Theorem 2.9 in \cite{ccfm_2021}) Let \(Z^{\star} \in \bR^{n \times m}\) be a matrix of rank \(r\) and \(\Delta Z^{\star} \in \bR^{n \times m}\) be its perturbation with \(Z  :=  Z^{\star} + \Delta Z^{\star}\). Moreover, let \(B^{\star}, B \in \bR^{n \times r}\) respectively be the first \(r\) left singular vectors of \(Z^{\star}, Z\). If \( \|\Delta Z^{\star}\| \leq (1 - 1 / \sqrt{2}) \sigma_{r}(Z^{\star})\), then 
\begin{equation*}
    \mathrm{dist}(B, B^{\star}) \leq \frac{2 \|\Delta Z^{\star}\|}{\sigma_{r}(Z^{\star})}
\end{equation*}
\end{cor}

\begin{cor}
\label{cor: cor of Matrix Bernstein}(Corollary to Matrix Bernstein inequality, cf. Corollary 6.2.1 in \cite{tropp2015introduction}) Let \(Z^{\star} \in \bR^{n \times m}\) be a matrix, and let \(\{ Z^{(i)} \in \bR^{n \times m}\}_{i =1}^{N}\) be a sequence of i.i.d. matrices with \(\mathbb{E}\left[Z^{(i)}\right] = Z^{\star}\). Denote \(\hat Z = \frac{1}{N}\sum_{i = 1}^{N}Z^{(i)}\) and \(\Delta Z^{\star} = \hat Z - Z^{\star}\). Let the distribution of \(Z^{(i)}\) be such that there exists a constant \(L\) with \(||Z^{(i)}|| \leq L\). 

Let \(\gamma :=  \max{\left(\left\|\mathbb{E}\left[Z^{(i)}\left(Z^{(i)}\right)^{\top}\right]\right\|,\left\|\mathbb{E}\left[\left(Z^{(i)}\right)^{\top}Z^{(i)}\right]\right\| \right)}\), and then
\begin{equation*}
    \mathbb{P}\left[
        \|\Delta Z^{\star}\| \geq t 
    \right] \leq  (m + n) \exp{\left(\frac{-Nt^2/2}{\gamma + 2Lt/3}     \right)}.
\end{equation*}
Using Jensen's inequality, one has \(\gamma \leq L^2\), and 
\begin{equation}
\label{eqn: Matrix Bernstein relaxed}
    \mathbb{P}\left[
        \|\Delta Z^{\star}\| \geq t 
    \right] \leq  (m + n) \exp{\left(\frac{-Nt^2/2}{L^2 + 2Lt/3}     \right)}.
\end{equation}
\end{cor}


\begin{theorem}
\label{thm:rotated-perturbation}(Error bound over TTNS tensor cores)
Let $p^\star \colon [n_1] \times \cdots \times [n_d] \to \mathbb{R}$ be a density function satisfy the TTNS assumption in Condition \ref{cond: TTNS ansatz condition}. Fix a sketch function $\{T_i, S_i\}_{i=1}^{d}$ which satisfies the recursive sketching assumption in Condition \ref{cond: recurisve sketching}. Let $\{A_i^\star, B_i^\star, G_i^\star, Z_{i}^{\star}\}_{i=1}^{d}$ be as in Theorem \ref{thm:ttns-algorithm-recovery}. Moreover, let $\{\hat A_{i}, \hat B_{i}, \hat G_{i}, \hat{Z}_{i}\}_{i =1}^{d}$ be as in Algorithm \ref{alg:1} with \(\hp\) as input. Suppose further that for some fixed $\delta \in (0, 1)$, one has
\begin{equation}
\label{eqn: sketch bounds}
    \|Z^{\star}_{k}  - \hat Z_{k} \|
    \le   
    \zeta_{k} \delta,
\end{equation}
where \(\zeta_{k}\) is defined by a series of constants as follows:
\begin{equation}
\zeta_{k} := \left(6 \frac{c_{\child}}{c_{k;Z}} \right)^{-1}\xi,\quad \xi := 1 \wedge \min_{i \in [d]} \left(2 c_{i;A} \left(c_{i; S} + c_{i;G}\right)\right)^{-1},
\end{equation}
and the constants are defined as follows:
\begin{itemize}
    \item $c_{\child} = \max_{i \in [d]}|\child(i)|$,
    \item $c_{k;Z} = 1$ when \(k = \text{root}\), and $c_{k;Z} = \sigma_{r_{(k, \parent(k))}}( Z_{k}^{\star}(\beta_{(k,\child(k))}, x_{k}; \gamma_{(k, \parent(k))} ))$ otherwise.
    \item $c_{k;G} = 1/\normi{G^{\star}_k}$,
    \item $c_{k;A} = 1$ when \(k = \text{leaf}\), and $c_{k;A} = \|\left(A^{\star}_{k}\right)^\dagger\|$ otherwise,
    \item $c_{k;S} = 1$ when \(k = \text{leaf}\), and $c_{k; S} = \prod_{w \in \child(k)} ||s_{w}(\beta_{(w,\parent(w))}; \beta_{(w, \child(w))},x_{w})||$ otherwise.
\end{itemize}

Then, there exists a TTNS tensor core \(\{G_{i}^{\circ}\}_{i=1}^{d}\) for \(p^{\star}\) in the sense of Definition \ref{def: TTNS short def}, such that \(\normi {G_{i}^{\circ}} = \normi {G_{i}^{\star}}\), and the following holds:
\begin{equation}
\label{eqn: core wise bound for Delta G_k^ast}
    \frac{\normi{\hat{G}_k - G_k^\circ}}{\normi{G^{\circ}_k}} \le \delta.
\end{equation}

\end{theorem}

We defer the proof of Theorem \ref{thm:rotated-perturbation} to the end of this subsection. As a direct application, one obtain the sample complexity of TTNS-Sketch:
\begin{theorem}
\label{thm:sample-complexity}
(Sample Complexity of TTNS-Sketch)
Assume the setting and notation of Theorem \ref{thm:rotated-perturbation}. Let \(\PTS\) denote the TTNS tensor formed by the TTNS tensor core \(\{\hat G_{i}\}_{i=1}^{d}\). In particular, \(\{\hat G_{i}\}_{i=1}^{d}\) is the output of Algorithm \ref{alg:1} with the empirical distribution \(\hp\) formed by \(N\) i.i.d. samples \((y_1^{(i)}, \ldots, y_d^{(i)})_{i =1}^{N}\). Let \(Z_{k}^{(i)}\) be the \(i\)-th sample estimate of \(Z_{k}^{\star}\), i.e.
\[
Z_{k}^{(i)}(\beta_{(k,\child(k))}, x_{k}, \gamma_{(k, \parent(k))})  :=  S_{k}(\beta_{(k,\child(k))}, y^{(i)}_{\leftside(k)})
 \mathbf{1}(y^{(i)}_k = x_{k}) T_{k}(y^{(i)}_{\rightside(k)}, \gamma_{(k, \parent(k))}),
\]
and set \(L_{k}\) as an upper bound of \( \|Z_{k}^{(i)}(\beta_{(k,\child(k))}, x_{k}; \gamma_{(k, \parent(k))})\|\). Define \(L = \max_{k \in [d]}L_{k}\).

For $\eta \in (0, 1)$ and $\epsilon \in (0, 1)$, suppose
\begin{equation*}
    N \ge \frac{18L^2d^2 + 4L\epsilon \zeta d}{\zeta^2 \epsilon^2}\log{\left(\frac{(ln + m)d}{\eta}\right)},
\end{equation*}
and the constants are defined as follows:
\begin{itemize}
    \item $\zeta = \max_{k \in [d]}\zeta_{k}$, with \(\zeta_k\) as in Theorem \ref{thm:rotated-perturbation}.
    \item $l = \max_{k \in [d]} l_k$, where $l_k = \prod_{w \in \child(k)}l_{(w, k)}$
    \item $m = \max_{k \in [d]} m_k$, where $m_{k} = m_{(k , \parent(k))}$.
    \item $n = \max_{k \in [d]} n_k$.
\end{itemize}

Then with probability at least $1 - \eta$ one has
\begin{equation}
    \frac{\|\PTS - p^{\star}\|_\infty}{\normi{G_1^\star} \cdots \normi{G_d^\star}} \le \epsilon.
\end{equation}
\end{theorem}

\begin{proof}
    (of Theorem \ref{thm:sample-complexity})
    Suppose that the inequality \eqref{eqn: sketch bounds} holds with \(\delta = \frac{\epsilon}{3d}\). In the setting of Theorem \ref{thm:rotated-perturbation}, note that \(p^{\star}\) is formed by \(\{G_{i}^{\circ}\}_{i=1}^{d}\), and \(\PTS\) is formed by \(\{G_{i}^{\circ} + \Delta G_i^\circ\}_{i=1}^{d}\) with $\normi{\Delta G_i^\circ} \le  \frac{\epsilon}{3d} \normi{G_i^\circ}$. Moreover, one has \(\normi {G_{i}^{\star}} = \normi {G_{i}^{\circ}}\). Applying \eqref{eqn:three-norm-total-contraction-perturbation} in Lemma \ref{lem:three-norm-total-contraction-perturbation}, one thus has
    \[
       \frac{\|\PTS - p^{\star}\|_\infty}{\normi{G^{\star}_1} \cdots \normi{G^{\star}_d}} = \frac{\|\PTS - p^{\star}\|_\infty}{\normi{G^{\circ}_1} \cdots \normi{G^{\circ}_d}}   \le \epsilon.
    \]
    
    By a simple union bound argument, it suffices to find a sample size that \eqref{eqn: sketch bounds} is guaranteed for each individual \(k \in [d]\) with \(\delta = \frac{\epsilon}{3d}\) and with probability \(1 - \frac{\eta}{d}\). We apply \eqref{eqn: Matrix Bernstein relaxed} in Corollary \ref{cor: cor of Matrix Bernstein}, where \(Z^{\star}_k\) is a matrix of size \(\bR^{l_{k}n_{k} \times m_{k}}\). With the choice of \((l,n,m,L)\) as set in the theorem statement, one has
    \[
        \mathbb{P}\left[
        \|\Delta Z^{\star}_{k}\| \geq t 
    \right] \leq  (ln + m) \exp{\left(\frac{-Nt^2/2}{L^2 + 2Lt/3}     \right)}.
    \]
    It then suffices for one to find a lower bound for \(N\) so that for \(t = \zeta\frac{\epsilon}{3d}\) one has
    \[
    (ln + m) \exp{\left(\frac{-Nt^2/2}{L^2 + 2Lt/3}     \right)} \leq \eta/d.
    \]
    By simple algebra, it suffices to lower bound \(N\) by the following quantity:
    \[
    N \ge \frac{2L^2 + 4Lt/3}{t^2}\log{\left(\frac{(ln + m)d}{\eta}\right)} = \frac{18L^2d^2 + 4L\epsilon \zeta d}{\zeta^2 \epsilon^2}\log{\left(\frac{(ln + m)d}{\eta}\right)}.
    \]
\end{proof}

As a corollary, for a Markov sketch function, note that each \(Z^{(i)}_{k}\) is a tensor with one entry being of value one, the rest being zero. Under this setting, note that \(\|Z^{(i)}_{k}\| \le \|Z^{(i)}_{k}\|_{F} = 1\), and hence one can set \(L = 1\). Let \(\Delta(T)\) denote the maximal degree of a tree \(T\). One has \(l \le n^{\Delta(T) - 1}\) and \(m = n \le ln\). 
Thus one obtains a sample complexity for TTNS-Sketch under Markov sketching:
\begin{cor}
(Sample Complexity of TTNS-Sketch for Markov Sketch function)
Suppose that \(p^{\star}\) is a graphical model over a tree \(T\), with the sketching function being the Markov sketch function specified in Lemma \ref{lem: ttns-markov-exact}.  Suppose
\begin{equation*}
    N \ge \frac{18d^2 + 4\epsilon \zeta d}{\zeta^2 \epsilon^2}\log{\left(\frac{2n^{\Delta(T)}d}{\eta}\right)}.
\end{equation*}

Then, with probability at least $1 - \eta$, one has
\begin{equation}
    \frac{\|\PTS - p^{\star}\|_\infty}{\normi{G_1^\star} \cdots \normi{G_d^\star}} \le \epsilon.
\end{equation}

\end{cor}

In the remainder of this subsection, we give the proof of Theorem \ref{thm:rotated-perturbation}, which is a culmination of all previous statements, the proof of which are of secondary interest and are included in Section \ref{sec: proof of sample complexity results}. For some intuition of Theorem \ref{thm:rotated-perturbation}, the factors in \(\zeta_{k}\) is set such that \(\xi\) can bound the sample estimation error of the sketched down core determining equation in Algorithm \ref{alg:1}. One then uses Lemma \ref{lem:tensor-eq-perturbation} to derive \eqref{eqn: core wise bound for Delta G_k^ast}. As a sanity check of the defined constants, note that \(\xi_{i} := \left(2 c_{i;A} \left(c_{i; S} + c_{i;G}\right)\right)^{-1}\) can be thought of as a homogeneous constant. That is, for any non-zero scaling constant \(\{q_{i}\}_{i = 1}^{d}\), changing the sketch cores from \(\{s_{i}\}_{i = 1}^{d}\) to \(\{q_{i}s_{i}\}_{i = 1}^{d}\) won't affect \(\xi_{i}\), which is because the resultant multiplicative change to \(\{c_{i; A}, c_{i; S}, c_{i; G}\}\) will be cancelled out in \(\xi_{i}\). One can think of \(\xi = 1 \wedge \min_{i} \xi_{i}\) in Theorem \ref{thm:rotated-perturbation} as serving the role of condition number. Moreover, because \((Z^\star_{k} - \hat Z_{k}) \propto c_{k; Z}\) by definition, it follows the condition in \eqref{eqn: sketch bounds} will not be affected if a scaling constant is applied to sketch cores.

\begin{proof}
(of Theorem \ref{thm:rotated-perturbation})
Following the short-hand in Algorithm \ref{alg:1-3}, for the joint variables we write \(\beta_{k} \gets \beta_{(k,\child(k))}, \gamma_{k} \gets \gamma_{\parent(k)}, \alpha_{k} \gets \alpha_{(k,\parent(k))}\), and for the bond dimensions we write \(r_{k} \gets r_{(\parent(k), k)}, l_{k} \gets \prod_{w \in \child(k)}l_{(w, k)}, m_{k} \gets m_{(k , \parent(k))}\). Moreover, if \(k\) is leaf, then we understand \(\beta_{k}\) as a joint variable taking value in \(\{1\}\), and \(l_{k} = 1\). Likewise, if \(k\) is root, then we understand \(\alpha_{k}, \gamma_{k}\) respectively as a joint variable taking value in \(\{1\}\), and \(r_{k} = m_{k} = 1\). In this notation, when \(k\) is leaf or root, the joint variables sketch \(Z_{k}\) is conveniently written as $Z_{k}(\beta_{k}, x_{k}; \gamma_{k} )$.

For this proof, we will fix a canonical unfolding matrix structure for the tensors used. For \(Z_{k}(\beta_{k}, x_{k},  \gamma_{k})\) being one of \(\{Z^{\star}_{k}, \hat Z_{k}, \Delta Z^{\star}_{k}\}\), we reshape it as \(Z_{k}(\beta_{k}, x_{k};  \gamma_{k})\). For \(B_{k}(\beta_{k}, x_{k},  \alpha_{k})\) being one of \(\{B^{\star}_{k}, B^{\circ}_{k}, \hat B_{k}, \Delta B^{\star}_{k}, \Delta B^{\circ}_{k}\}\), we reshape it as \(B_{k}(\beta_{k}, x_{k};  \alpha_{k})\). For \(A_{k}(\beta_{k}, \alpha_{(k, \child(k))})\) being one of \(\{A^{\star}_{k}, A^{\circ}_{k}, \hat A_{k}, \Delta A^{\star}_{k}, \Delta A^{\circ}_{k}\}\), we reshape it as \(A_{k}(\beta_{k}; \alpha_{(k, \child(k))})\). For \(s_{k}(\beta_{(k,\parent(k))}, \beta_{k},x_{k} )\), we reshape it as \(s_{k}(\beta_{(k,\parent(k))}; \beta_{k},x_{k} )\). For \(G_{k}(\alpha_{(k, \child(k))}, x_k, \alpha_{k})\) being one of \(\{G^{\star}_{k}, G^{\circ}_{k}, \hat G_{k}, \Delta G^{\star}_{k}, \Delta G^{\circ}_{k}\}\), we reshape it as \(G_{k}(\alpha_{(k, \child(k))}; x_k, \alpha_{k})\).

Fix a non-root $k$, recall that $B_k^{\star}$ and $\hat{B}_k$ are the first $r_{k}$ left singular vectors of $ Z^{\star}_{k}$ and $\hat Z_{k}$, respectively. One applies Corollary \ref{cor: cor of Wedin}: if $\|\Delta Z_{k}^{\star}\| \le (1 - 1 / \sqrt{2}) \sigma_{r_{k}}(Z^{\star}_{k})$, then one can find $R_{k}\in \operatorname{O}(r_{k})$ such that one can define \(B^{\circ}_{k}  :=  B^{\star}_k R_{k}\) so that
\begin{equation*}
    \hat{B}_k = B_k^\circ + \Delta B^{\circ}_{k}, \quad \|\Delta B^{\circ}_{k}\| 
    \le \frac{2 \|\Delta Z_{k}^{\star}\|}{\sigma_{r_{k}}(Z^{\star}_{k})}
\end{equation*}
and by \eqref{enumerate: three tensor spectral norm bound} in Lemma \ref{lem: 3-tensor basic results}, one has
\begin{equation}
\label{eq:wedin-error}
    \normi{\Delta B^{\circ}_{k}}
    \le \frac{2 \|\Delta Z_{k}^{\star}\|}{\sigma_{r_{k}}(Z^{\star}_{k})}.
\end{equation}

Meanwhile, if $k$ is the root, there is no SVD step. In this case, the perturbation $\Delta B_k^{\star}$ is simply $\Delta Z_{k}^{\star}$. For consistency, when \(k\) is the root, we set $B_k^\circ = B_{k}^{\star}$, and the corresponding perturbation $\Delta B^{\circ}_{k}$ is just $\Delta Z_{k}^{\star}$. 

In summary, $B_k^\circ$ is a rotation of $B_k^{\star}$, and $\hat B_k$ differs from $B_k^\circ$ by a perturbation $\Delta B^{\circ}_{k}$, for which one has an error bound. For a ``rotated" version of \(A_k^{\star}\), define
\begin{equation}
\label{eqn: def of A^ast}
     A^{\circ}_{k}(\beta_{(k,\child(k))}, \alpha_{(k, \child(k))})
     :=   \prod_{w \in \child(k)}  \sum_{(\beta_{w}, x_{w})} s_{w}(\beta_{(w,k)}, \beta_{w},x_{w})
    B^{\circ}_{w}(\beta_{w}, x_{w},\alpha_{(w,k)})
\end{equation}

Viewed in the unfolding matrix structure fixed in the beginning of proof, one can write $A_k^\circ = \bigotimes_{w \in \child(k)} s_w B_w^\circ$. Likewise, one has $\hat A_k= \bigotimes_{w \in \child(k)} s_w \hat B_w = \bigotimes_{w \in \child(k)} \left(s_w B_w^\circ + s_{w} \Delta B^{\circ}_{w}\right)$.

Now, with the chosen unfolding matrix structure, consider the following ``rotated'' versions of \eqref{eq:alg-CDEs}:
\begin{equation}
    \label{eq:alg-CDEs-rotated}
    \begin{aligned}
        G_{k}^\circ & = B_{k}^\circ(\beta_{k}; x_{k},  \alpha_{k}) \quad \text{if} ~ k ~ \text{is a leaf}, \\
        \textstyle A_{k}^\circ G_{k}^\circ & = B_{k}^\circ(\beta_{k}; x_{k},  \alpha_{k}) \quad \text{otherwise}.
    \end{aligned}
\end{equation}

We will first prove that \(\{G_{i}^{\circ}\}_{i=1}^{d}\) forms a TTNS tensor core for \(p^{\star}\) in the sense of Definition \ref{def: TTNS short def}. Suppose one has \(\{\Phi^{\star}_{k \to \parent(k)}\}_{\text{\(k \not =\) root}}\) defined according to Condition \ref{cond: ttns gauge choice}. Then, Theorem \ref{thm:ttns-algorithm-recovery} proves that
\(\{G_{i}^{\star}\}_{i=1}^{d}\) solves the CDE \eqref{eq:ttn-determining} in Theorem  \ref{thm:ttns-existence} for the choice of gauge as \(\{\Phi^{\star}_{k \to \parent(k)}\}_{\text{\(k \not =\) root}}\). Now consider \eqref{eq:ttn-determining} for a rotated choice of gauge \(\{\Phi^{\circ}_{k \to \parent(k)}  :=  \Phi^{\star}_{k \to \parent(k)}R_{k}\}_{\text{\(k \not =\) root}}\). One can directly check that the sketched down equation coincides with \eqref{eq:alg-CDEs-rotated}, and Theorem \ref{thm:ttns-algorithm-recovery} ensures that the solution \(\{G_{i}^{\circ}\}_{i=1}^{d}\) is unique and forms a TTNS tensor core for \(p^{\star}\).

Next, we prove that $\normi{G_k^\circ} = \normi{G^{\star}_k}$, with the 3-tensor view for $G_k^\circ, G^{\star}_k$ as in Definition \ref{def: 3-tensor view of TTNS tensor core}. As the coefficients $A_k^\circ$'s and right-hand sides $B_k^\circ$'s are simply the rotations of $A_k^{\star}$'s and $B_k^{\star}$'s of \eqref{eq:alg-CDEs}, one can verify that \(G_k^\circ\) is a rotation of \(G_k^{\star}\). 
If \(k\) is not leaf nor root, the equation for \(G^{\circ}_k\) can be rewritten as 
\begin{equation}
\label{eqn: core-wise equation for G_k_ast}
    \left(\bigotimes_{w \in \child(k)} s_w B_w^{\star} R_{w}\right) G_{k}^\circ = B^{\star}_{k}R_{k},
\end{equation}

whereas the equation for \(G_{k}^{\star}\) is 
\[
\left(\bigotimes_{w \in \child(k)} s_w B_w^{\star} \right) G_{k}^\star = B^{\star}_{k}.
\]

For the rotation matrix \(R_{k}(\alpha; \beta) \in \operatorname{O}(r_{k})\), one gives it a 3-tensor view as \(R_{k}(\alpha; 1;\beta)\) in the sense of Definition \ref{def: matrix 3-tensor structure}.  One can directly verify that the following equation for \(G_k^\circ\) solves \eqref{eqn: core-wise equation for G_k_ast}:
\begin{equation}
\label{eq:rotated-cores}
    G_k^\circ = \left(\bigotimes_{w \in \child(k)} R_{w}^\top \right) \circ G^{\star}_k \circ R_{k}.
\end{equation}

Likewise, \(G_k^\circ =  G^{\star}_k \circ R_{k}\) if \(k\) is leaf, and \(G_k^\circ = \left(\bigotimes_{w \in \child(k)} R_{w}^\top \right) \circ G^{\star}_k\) if \(k\) is root. Then $\normi{G_k^\circ} = \normi{G^{\star}_k}$ as a consequence. The constructive form in \eqref{eq:rotated-cores} also gives a more intuitive sense of why \(\{G_{i}^{\circ}\}_{i=1}^{d}\) forms a TTNS tensor core in the same way as \(\{G_{i}^{\star}\}_{i=1}^{d}\). The reason is that each \(R_{k}\) and \(R_{k}^{\top}\) comes in pairs, which does not change the formed TTNS tensor itself.

Next, we prove that, for \(\Delta B_k^\circ  :=  \hat{B}_k - B_k^\circ\) and \(\Delta A_k^\circ  :=  \hat{A}_k - A_k^\circ\), the assumption \eqref{eqn: sketch bounds} leads to the the following bound:
\begin{equation}
\label{eq:condition-for-lemma}
      \|\Delta B_k^\circ\|
    \le \xi_{} \delta , \quad \|\Delta A_k^\circ\| \le c_{k; S}\xi \delta.
\end{equation}

First, for $\Delta B_k^\circ$'s, we will derive a tighter bound
\begin{equation}
\label{eq:tighter}
     \|\Delta B_k^\circ\| \le \frac{\xi_{}\delta}{3 c_{\child}},
\end{equation}
which implies $ \|\Delta B_k^\circ\| \le \xi_{} \delta$ as $3 c_{\child} \ge 1$. To see this, using \eqref{eq:wedin-error}, one has for any non-root $k$,
\begin{equation*}
    \|\Delta B_k^\circ\|
    \le \frac{2 \|\Delta Z_{k}^{\star}\|}{\sigma_{r_{k}}(Z^{\star}_{k})}
    \le \frac{2\zeta_{k} \delta}{c_{k;Z}} 
    = \frac{2}{c_{k;Z}} \frac{c_{k;Z}}{6 c_{\child}} \xi_{}\delta
    \le \frac{\xi_{}\delta}{3 c_{\child}}.
\end{equation*}
If $k$ is the root, recall that $\Delta B_k^\circ = \Delta Z_k^\star$, hence using $c_{k;Z} = 1$, 
\begin{equation*}
    \|\Delta B_k^\circ\| 
    = \|\Delta Z_k^\star\| 
    \le \zeta_k \delta 
    = \frac{c_{k;Z}}{6 c_{\child}} \xi_{}\delta 
    \le \frac{\xi_{}\delta}{3 c_{\child}}.
\end{equation*}
Therefore, \eqref{eq:tighter} holds.

Next, for a non-leaf $k$, recall that
\begin{align*}
    \Delta A_k^\circ &= \bigotimes_{w \in \child(k)} (s_w B_w^\circ + s_w \Delta B^{\circ}_{w}) - \bigotimes_{w \in \child(k)} s_w B_w^\circ\\
    &=\left(\bigotimes_{w \in \child(k)}{s_{w}}\right)\left(\bigotimes_{w \in \child(k)} (B_w^\circ + \Delta B^{\circ}_{w}) - \bigotimes_{w \in \child(k)} B_w^\circ\right)
\end{align*}

By definition, one has \(\|\bigotimes_{w \in \child(k)}{s_{w}}\| = c_{k; S}\). Note that $\|B_w^\circ\| = 1$ and $\|\Delta B_w^\circ\| \leq  \frac{\xi_{}}{3 c_{\child}} \delta $. Hence, one can apply Lemma \ref{lem:kronecker_product_bound}, which shows
\begin{equation*}
    \|\Delta A_k^\circ\|
    \le c_{k; S}\left(\sum_{w \in \child(k)} \|\Delta B^{\circ}_{w}\|\right) \exp\left(\sum_{w \in \child(k)} \|\Delta B^{\circ}_{w}\|\right),
\end{equation*}
Using \eqref{eq:tighter},
\begin{equation*}
    \sum_{w \in \child(k)} \|\Delta B^{\circ}_{w}\| 
    \le c_{\child} \cdot \max_{w \in [d]} \|\Delta B^{\circ}_{w}\| 
    \le c_{\child} \frac{\xi\delta}{3 c_{\child}}
    \le \frac{\xi\delta}{3}.
\end{equation*}
Hence,
\begin{align*}
    \|\Delta A_k^\circ\|
    &\le c_{k; S}\left(\sum_{w \in \child(k)} \|\Delta B^{\circ}_{w}\|\right) \exp\left(\sum_{w \in \child(k)} \|\Delta B^{\circ}_{w}\|\right)\\
    &\le c_{k; S}\frac{\xi_{} \delta}{3} \exp(1)\\
    &\le c_{k; S}\xi_{} \delta,
\end{align*}
where the last two steps hold because $\frac{\xi\delta}{3} < 1$ and $\exp(1) < 3$. 

It remains to show how \eqref{eq:condition-for-lemma} lead to \eqref{eqn: core wise bound for Delta G_k^ast}. 

If $k$ is a leaf, 
\begin{equation*}
    \frac{\normi{\Delta G_k^\circ}}{\normi{G^{\circ}_k}} 
    = \frac{\normi{\Delta B^{\circ}_k}}{\normi{G^{\circ}_k}}
    \le \frac{ \|\Delta B^{\circ}_k\|}{\normi{G^{\circ}_k}}
    \le c_{k;G}\xi_{}  \delta 
    \le \delta,
\end{equation*}
where the first equation follows from \(\Delta G_k^\circ = \Delta B_k^\circ\) in \eqref{eq:rotated-cores}, the first inequality comes from \eqref{enumerate: three tensor spectral norm bound} in Lemma \ref{lem: 3-tensor basic results}, and the last inequality uses \(c_{k;A} = c_{k;S} = 1\) and \(\xi \leq \left(2 c_{k;A} \left(c_{k; S} + c_{k;G}\right)\right)^{-1} = \frac{1}{2}\left(1 + c_{k;G}\right)^{-1}\).

Importantly, note that
$$A_k^\circ = \left(\bigotimes_{w \in \child(k)} s_w B_w^\star\right)\left(\bigotimes_{w \in \child(k)} R_w \right),$$
and so the fact that each \(R_w\) is orthogonal implies \(\|\left(A_k^{\circ}\right)^\dagger\| = \|\left(A_k^{\star}\right)^\dagger\| = c_{k;A}\). For any non-leaf $k$, note that \( \xi,\delta \le 1\) leads to $\|\left(A_k^{\circ}\right)^\dagger\| \|\Delta A_k^{\circ}\| \le c_{k;A}c_{k;S}\xi \delta  \le 1 / 2$. From Lemma \ref{lem:tensor-eq-perturbation} and \eqref{enumerate: three tensor spectral norm bound} in Lemma \ref{lem: 3-tensor basic results}, it follows that
\begin{equation*}
\begin{split}
    \frac{\normi{\Delta G_k^\circ}}{\normi{G^{\circ}_k}} 
    & \le 2 \|\left(A^{\circ}_{k}\right)^\dagger\| \left(\|\Delta A^{\circ}_k\| + \|\Delta B^{\circ}_k\|c_{k;G}\right) \\
    & \le 2 c_{k;A} \left(c_{k;S} + c_{k;G}\right) \xi \delta \\
    & \le \delta,
\end{split}
\end{equation*}
and so we are done.


\end{proof}

\subsection{Remarks on sample complexity bound for total variation distance}\label{sec: remark to l_1 sample complexity}
Using the proof technique as outline before, one can derive a sample complexity upper bound on the total variation norm via the \(l_{1}\) distance between \(p^\star\) and \(\PTS\). Note that one can define a new norm \(\normi{\cdot}_{1}\) by
\[
\normi{G(\alpha;x;\beta)}_{1} := \sum_{x} \|G(\cdot, x, \cdot)\|,
\]
which is a similar definition to \(\normi{\cdot}\).

The proofs in Section \ref{sec: proof of sample complexity results} are also written such that the adaptation to \(\normi{\cdot}_{1}\) is straightforward. First, all of the results in Lemma \ref{lem: 3-tensor basic results} will hold in this new norm, with only a change in the constant in (v). Second, from the \(\normi{\cdot}_{1}\) version of Lemma \ref{lem: 3-tensor basic results}, one can bound the global \(\|\cdot\|_{1}\) error by the core-wise \(\normi{\cdot}_{1}\) error by an adaptation of Lemma \ref{lem:three-norm-total-contraction-perturbation}. Finally, for local \(\normi{\cdot}_{1}\) error on cores, the proof of Lemma \ref{lem:tensor-eq-perturbation} also proves that Lemma \ref{lem:tensor-eq-perturbation} holds if one replaces \(\normi{\cdot}\) by the new \(\normi{\cdot}_{1}\) norm. Importantly, the \(N = O(d^{2})\) scaling will still hold under the \(l_{1}\)-norm. 


\subsection{Proof of results}
\label{sec: proof of sample complexity results}
\begin{proof}
(of Lemma \ref{lem: 3-tensor basic results})

In the notation of Definition \ref{def: Matrix slice of 3-tensor}, one can write the definition of \(\circ\) by
\begin{equation*}\label{eqn: circ operation with dot}
    G \circ G'(\cdot, (x, y), \cdot) =  G(\cdot, x, \cdot) G'(\cdot, y, \cdot).
\end{equation*}
 
Associativity of \(\circ\) thus follows from associativity of matrix product, and likewise the inequality for \(\circ\) comes from submultiplicativity of matrix product under spectral norm:
\begin{equation*}
    \max_{(x,y)} \|G \circ G'(\cdot, (x, y), \cdot)\| = \max_{(x,y)}  \|G(\cdot, x, \cdot) G'(\cdot, y, \cdot)\| \leq \left(\max_{x}  \|G(\cdot, x, \cdot)\| \right)  \left(\max_{y} \| G'(\cdot, y, \cdot)\| \right)
\end{equation*}

Likewise, by abuse of notation, also use \(\otimes\) as the Kronecker product operation over matrices. Then one can simplify and write the definition of \(\otimes\) by
\begin{equation*}\label{eqn: prod operation with dot}
    G \otimes G'(\cdot, (x, y), \cdot) =  G(\cdot, x, \cdot) \otimes G'(\cdot, y, \cdot).
\end{equation*}

Associativity of \(\otimes\) likewise follows from associativity of Kronecker product over matrices. Likewise, the equality for \(\otimes\) comes from multiplicativity of matrix product under spectral norm:
\begin{equation*}
    \|G \otimes G'(\cdot, (x, y), \cdot)\| =  \|G(\cdot, x, \cdot)\| \cdot \|G'(\cdot, y, \cdot)\|
\end{equation*}

We now prove (\ref{enumerate: three tensor spectral norm bound}). For a vector \(v \in \bR^{r_{2}}\), one can view the vector \(G(\alpha, x; \beta)v\) as the concatenation of \(n\) smaller vectors of the form \(G(\cdot, x; \cdot)v\). For the upper bound, one has
    \begin{align*}
        \|G(\alpha, x; \beta)v\| = \sum_{x \in [n]}\|G(\cdot, x; \cdot)v\| \leq n\max_{x \in [n]}\|G(\cdot, x; \cdot)\| \|v\| = n \normi{G}\|v\|,
    \end{align*}
where is done after taking supremum over \(v\) with \(\|v\| = 1\).

For the lower bound, one has
    \begin{align*}
        \|G(\alpha, x; \beta)v\| = \sum_{x \in [n]}\|G(\cdot, x; \cdot)v\| \geq \max_{x \in [n]}\|G(\cdot, x; \cdot)\| \|v\| = \normi{G}\|v\|,
    \end{align*}
and likewise one is done after taking supremum over \(v\) with \(\|v\| = 1\).

\end{proof}

\begin{proof} (of Lemma \ref{lem:three-norm-contraction-bound})
Suppose that in \(T\), the maximum distance from the root node is \(L\). At a level \(l \in \{1, \ldots, L\}\), suppose there are \(d_{l}\) nodes in \(T\) which are of distance \(l\) to the root, denoted by the set \(\{v^{l}_{i}\}_{i = 1}^{d_{l}}\). Then, if one views \(p\) as a 3-tensor of size \(\{1\} \times \left[\prod_{i = 1}^{d}n_{i}\right] \times \{1\} \to \bR\), then one has
\begin{equation}
    p =  G_{\text{root}(T)} \circ \bigotimes_{i \in [d_{1}]}G_{v^{1}_{i}}  \circ \bigotimes_{k \in [d_{2}]}G_{v^{2}_{j}}  \circ \ldots \circ \bigotimes_{k \in [d_{L}]}G_{v^{L}_{k}},
\end{equation}
which is only a consequence of the structure of the TTNS ansatz and the 3-tensor structure of TTNS tensor core in Definition \ref{def: 3-tensor view of TTNS tensor core}. Then, by the chosen 3-tensor structure of \(p\), one has \(\|p\|_{\infty} = \normi{p}\). By Lemma \ref{lem: 3-tensor basic results}, one has 
\begin{equation}
    \|p\|_{\infty} = \normi{p} \leq  \normi{G_{\text{root}(T)}} \prod_{l = 1}^{L}\normi{ \bigotimes_{i \in [d_{l}]}G_{v^{l}_{i}}}  = \normi{G_{\text{root}(T)}} \prod_{l = 1}^{L}\prod_{i \in [d_{l}]}\normi{ G_{v^{l}_{i}}} = \prod_{i \in [d]}\normi{ G_{i}},
\end{equation}
where the first inequality and the second equality follows from \eqref{eqn: circ sub-multiplicative} and \eqref{eqn: otimes multiplicative} in Lemma \ref{lem: 3-tensor basic results}.

\end{proof}

\begin{proof} (of Lemma \ref{lem:three-norm-total-contraction-perturbation})
Let \(p_{0}, \ldots, p_{d}\) be a sequence of tensors such that \(p_{0} = p\), and \(p_{k}\) is the tensor formed by the TTNS tensor core \(\{G_{i} + \Delta G_{i}\}_{i = 1}^{k} \cup \{G_{j}\}_{j \not \in [k]}\). One is interested in the error \(\|p_{d} - p_{0}\|_{\infty}\), and one can bound by 
\begin{equation*}
    \|p_{d} - p_{0}\|_{\infty} \leq \sum_{k = 1}^{d}\|p_{k} - p_{k-1}\|_{\infty}.
\end{equation*}

One can then bound the magnitude of each term in this telescoping sum. Note that \(p_{k} - p_{k-1}\) is a TTNS ansatz formed by cores \(\{G_{i} + \Delta G_{i}\}_{i = 1}^{k-1} \cup \{\Delta G_{k}\} \cup \{G_{j}\}_{j = k+1}^{d}\), and thus by Lemma \ref{lem:three-norm-contraction-bound}
\begin{equation*}
    \|p_{k} - p_{k-1}\|_{\infty} \leq \prod_{i = 1}^{k-1}\normi{G_{i} + \Delta G_{i}} \normi{\Delta G_{k}} \prod_{j = k+1}^{d}\normi{G_{j}} \leq \delta_{k}\prod_{i = 1}^{d} (1 + \delta_i)\prod_{i = 1}^{d}\normi{\Delta G_{i}}.
\end{equation*}
 
Therefore, using $1 + x \le \exp(x)$, one has
\begin{equation*}
    \|p_{d} - p_{0}\|_{\infty} \leq \left(\sum_{i = 1}^{d}\delta_{i}\right)\prod_{i = 1}^{d} (1 + \delta_i)\prod_{i = 1}^{d}\normi{\Delta G_{i}} \leq \left(\sum_{i = 1}^{d}\delta_{i}\right)\exp\left(\sum_{i = 1}^{d} \delta_i\right)\prod_{i = 1}^{d}\normi{\Delta G_{i}}
\end{equation*}

\end{proof}

\begin{proof} (of Lemma \ref{lem:tensor-eq-perturbation})
Note that \eqref{eqn: Delta X bound relaxed version} is only a corollary of \eqref{eqn: Delta X bound tight version}. To prove \eqref{eqn: Delta X bound tight version}, it suffices to prove that for any \(i \in [n]\), one has
\begin{equation}
\label{eqn: Delta X bound slice version}
    \|\Delta X^{\star}(\cdot, i, \cdot)\| \le \frac{\|\left(A^{\star}\right)^\dagger\|}{1 - \|\left(A^{\star}\right)^\dagger\| \|\Delta A^{\star}\|} \left( \|\Delta A^{\star}\| \|X^{\star}(\cdot, i, \cdot)\| +  \|\Delta B^{\star}(\cdot, i, \cdot)\|\right),
\end{equation}
whereby \eqref{eqn: Delta X bound tight version} is obtained by taking maximum over \(i \in [n]\) on both sides. 

Based on the above observation, one can simplify notation and reduce argument over \(\normi{\cdot}\) to regular spectral norm over matrices. For a fixed \(i \in [n]\), define \(C^\star := B^{\star}(\cdot, i, \cdot)\). Let \(Y^\star\) be the matrix which is the unique \emph{exact} solution the linear equation \(A^{\star}Y = C^{\star}\). Naturally, one has \(Y^\star = X^{\star}(\cdot, i, \cdot)\). 

Likewise, define \(\Delta C^{\star} = \Delta B^{\star}(\cdot, i, \cdot)\) as the corresponding perturbation to \(C^{\star}\), and let \(Y^\star + \Delta Y^{\star}\) be the matrix which is the unique solution the linear equation \((A^{\star}+\Delta A^{\star})Y = (C^{\star}+\Delta C^{\star})\) in the sense of least squares. As before, one has \(Y^\star + \Delta Y^\star  = X^{\star}(\cdot, i, \cdot) +\Delta X^{\star}(\cdot, i, \cdot)\). Then, \eqref{eqn: Delta X bound slice version} is equivalent to the following inequality:
\begin{equation}
\label{eqn: Delta X bound matrix version}
    \|\Delta Y^{\star}\| \le \frac{\|\left(A^{\star}\right)^\dagger\|}{1 - \|\left(A^{\star}\right)^\dagger\| \|\Delta A^{\star}\|} \left( \|\Delta A^{\star}\| \|X^{\star}\| +  \|\Delta C^{\star}\|\right).
\end{equation}

To reduce further, for an arbitrary \(v \in \bR^m\), note that it suffices to prove the following result 
\begin{equation}
\label{eqn: Delta X bound vector version}
    \|\Delta Y^{\star}v\| \le \frac{\|\left(A^{\star}\right)^\dagger\|}{1 - \|\left(A^{\star}\right)^\dagger\| \|\Delta A^{\star}\|} \left( \|\Delta A^{\star}\| \|Y^{\star}v\| +  \|\Delta C^{\star}v\|\right),
\end{equation}
and \eqref{eqn: Delta X bound matrix version} follows by taking supremum over \(v\) with \(\|v\| = 1\). 

To simplify further, define \(b^{\star}  :=  C^{\star}v \in \bR^{l}\), and let \(x^{\star}  := Y^{\star}v \in \bR^{r}\) be the unique \emph{exact} solution to \(A^{\star}x = b^{\star}\).
Moreover, let $\Delta b^{\star}  :=  \Delta C^{\star}v$ and let $\Delta x^{\star}  :=  \Delta Y^{\star}v$, and then \(x^{\star}+ \Delta x^{\star}\) solves the linear equation \((A^{\star}+\Delta A^{\star})x = (b^{\star}+\Delta b^{\star})\) in the sense of least squares. This is exactly the setting of Theorem 3.48 in \cite{wendland_2017}, because of which \eqref{eqn: Delta X bound vector version} holds as a corollary. Thus we are done. 
\end{proof}

\begin{proof} (of Lemma \ref{lem:kronecker_product_bound}
)

Let $C'_i = C_i + E_i$, then
\begin{equation}
\label{eq:kronecker_error_split}
    \begin{split}
        \bigotimes_{i = 1}^{n} (C_i + E_i) - \bigotimes_{i = 1}^{n} C_i
        & = (C'_1 \otimes \cdots \otimes C'_n) - (C_1 \otimes C'_2 \otimes \cdots \otimes C'_n) \\
        & \quad + (C_1 \otimes C'_2 \otimes \cdots \otimes C'_n) - (C_1 \otimes C_2 \otimes C'_3 \otimes \cdots \otimes C'_n) \\
        & \quad + \cdots \\
        & \quad + (C_1 \otimes \cdots \otimes C_{n - 1} \otimes C'_n) - (C_1 \otimes \cdots \otimes C_n).
    \end{split}
\end{equation}

The first line on the right-hand side of \eqref{eq:kronecker_error_split} reduces to $E_1 \otimes C'_2 \otimes \cdots \otimes C'_n$. Since $\|C'_i\| \le \|C_i\| + \|E_i\| \le 1 + \delta_i$, 
\begin{equation*}
    \|E_1 \otimes C'_2 \otimes \cdots \otimes C'_n\|
    =
    \|E_1\| \|C'_2\| \cdots \|C'_n\|
    \le 
    \delta_1 (1 + \delta_2) \cdots (1 + \delta_n)
    \le
    \delta_1 \cdot \prod_{i = 1}^{n} (1 + \delta_i).
\end{equation*}

The norm of the $j$-th line on the right-hand side of \eqref{eq:kronecker_error_split} is upper bounded by $\delta_j \cdot \prod_{i = 1}^{n} (1 + \delta_i)$. Therefore, using $1 + x \le \exp(x)$, one has
\begin{equation*}
    \left\|\bigotimes_{i = 1}^{n} (C_i + E_i) - \bigotimes_{i = 1}^{n} C_i\right\| 
    \le \left(\sum_{i = 1}^{n} \delta_i\right) \cdot \prod_{i = 1}^{n} (1 + \delta_i)
    \le \left(\sum_{i = 1}^{n} \delta_i\right) \exp\left(\sum_{i = 1}^{n} \delta_i\right).
\end{equation*}
\end{proof}

\begin{proof}
(of Corollary \ref{cor: cor of Wedin} and Corollary \ref{cor: cor of Matrix Bernstein}) 

For Corollary \ref{cor: cor of Wedin}, we apply Theorem 2.9, (2.26a) in \cite{ccfm_2021}: if \( \|\Delta Z^{\star}\| \leq (1 - 1 / \sqrt{2}) \sigma_{r}(Z^{\star})\), then
    \[\mathrm{dist}(B, B^{\star}) \leq \frac{2 \|\left(B^{\star}\right)^{\top}\Delta Z^{\star}\|}{\sigma_{r}( Z^{\star}) - \sigma_{r+1}(Z^{\star})} \le  \frac{2 \|\left(B^{\star}\right)^{\top}\|\|\Delta Z^{\star}\|}{\sigma_{r}( Z^{\star}) - \sigma_{r+1}(Z^{\star})},\]
and we are done by applying \(\sigma_{r+1}(Z^{\star}) = 0\) and \(\|B^{\star}\| = 1\).

For Corollary \ref{cor: cor of Matrix Bernstein}, only \eqref{eqn: Matrix Bernstein relaxed} is new, and one only needs to justify \(\gamma \leq L^2\). By Jensen's theorem and sub-multiplicativity of spectral norm, one has
\begin{align*}
    \left\|\mathbb{E}\left[Z^{(i)}\left(Z^{(i)}\right)^{\top}\right]\right\| \leq
     \mathbb{E}\left[\left\|Z^{(i)}\left(Z^{(i)}\right)^{\top}\right\|\right]
    \leq \mathbb{E}\left[\left\|Z^{(i)}\right\|^2\right] \leq L^2.
\end{align*}

\end{proof}

\section{Numerical result}\label{sec: numerical result}
\newcommand{\Tpa}{T_{\mathrm{path}}}
\newcommand{\Epa}{E_{\mathrm{path}}}

In this section, we perform comparison of different modeling methods. There are four models of interest. The symbol \(\PTS\) stands for the model obtained from the TTNS-Sketch method. The symbol \(\hat p_{\mathrm{GM}}\) stands for the model one obtains from direct graphical modeling over a given tree structure \(T\). Specifically, the \(\hat p_{\mathrm{GM}}\) model with tree structure \(T\) refers to the graphical model over \(T\) where the parameters are chosen by maximum likelihood estimation. The symbol \(\hat p_{\mathrm{CL}}\) stands for the Chow-Liu model, which is obtained by direct graphical modeling with the Chow-Liu tree \(\TCL\). The symbol \(\hat p_{\mathrm{BM}}\) stands for the model one obtains from modeling with Born Machine (BM). The training of BM is done by optimizing Negative Log Likelihood (NLL), with details of the training following from that of \cite{han2018unsupervised}.

In what follows, the \emph{error} of a model \(p\) refers to the relative \(l_{2}\) error:
\[
\mathrm{Error}(p) := \frac{\|p - p^\star\|}{\|p\|}.
\]

For BM training, we will use Negative log likelihood level of the model as a performance metric:
\[
\mathrm{NLL}(p) := \sum_{i = 1}^{N}p(y_{1}^{(i)}, \ldots, y_{d}^{(i)}).
\]

\subsection{Numerical case study: tree graphical model with different input tree topology}

\begin{figure}
     \centering
     \begin{subfigure}[b]{\textwidth}
         \centering
         \includegraphics[width=0.9\textwidth]{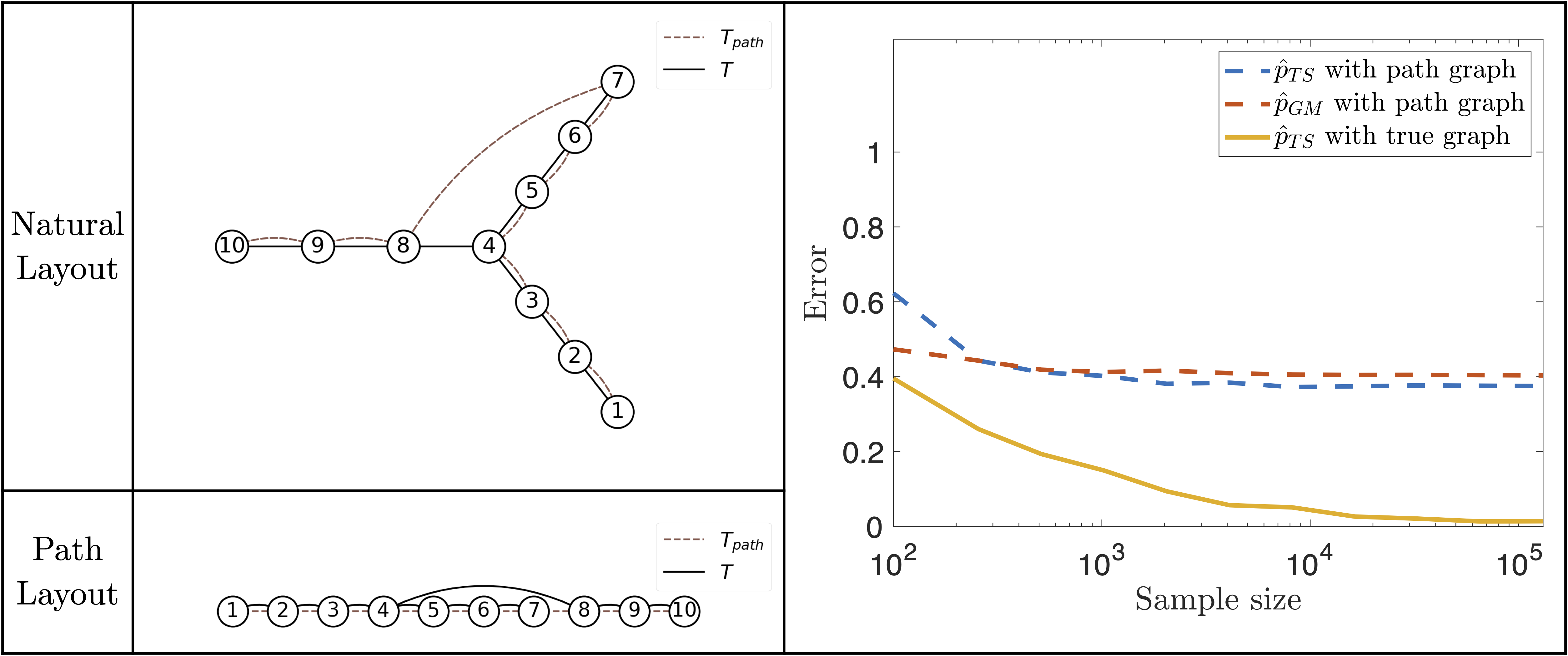}
         \caption{Trident Graph, $d = 10$}
         \label{fig:Trident Graph}
     \end{subfigure}
     \\
      \begin{subfigure}[b]{\textwidth}
         \centering
         \includegraphics[width=0.9\textwidth]{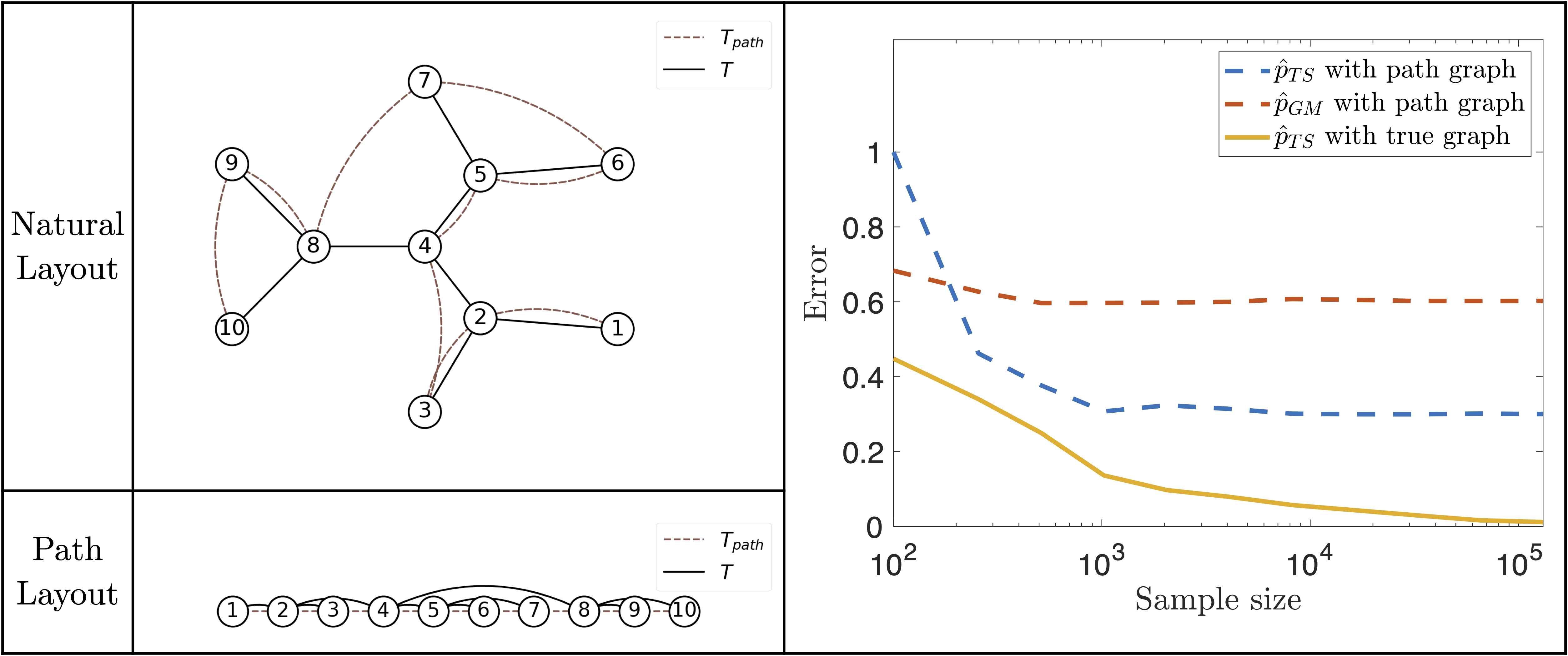}
         \caption{Dendrimer Graph, $d = 10$}
         \label{fig:Dendrimer Graph}
     \end{subfigure}
     \\
     \begin{subfigure}[b]{\textwidth}
         \centering
         \includegraphics[width=0.9\textwidth]{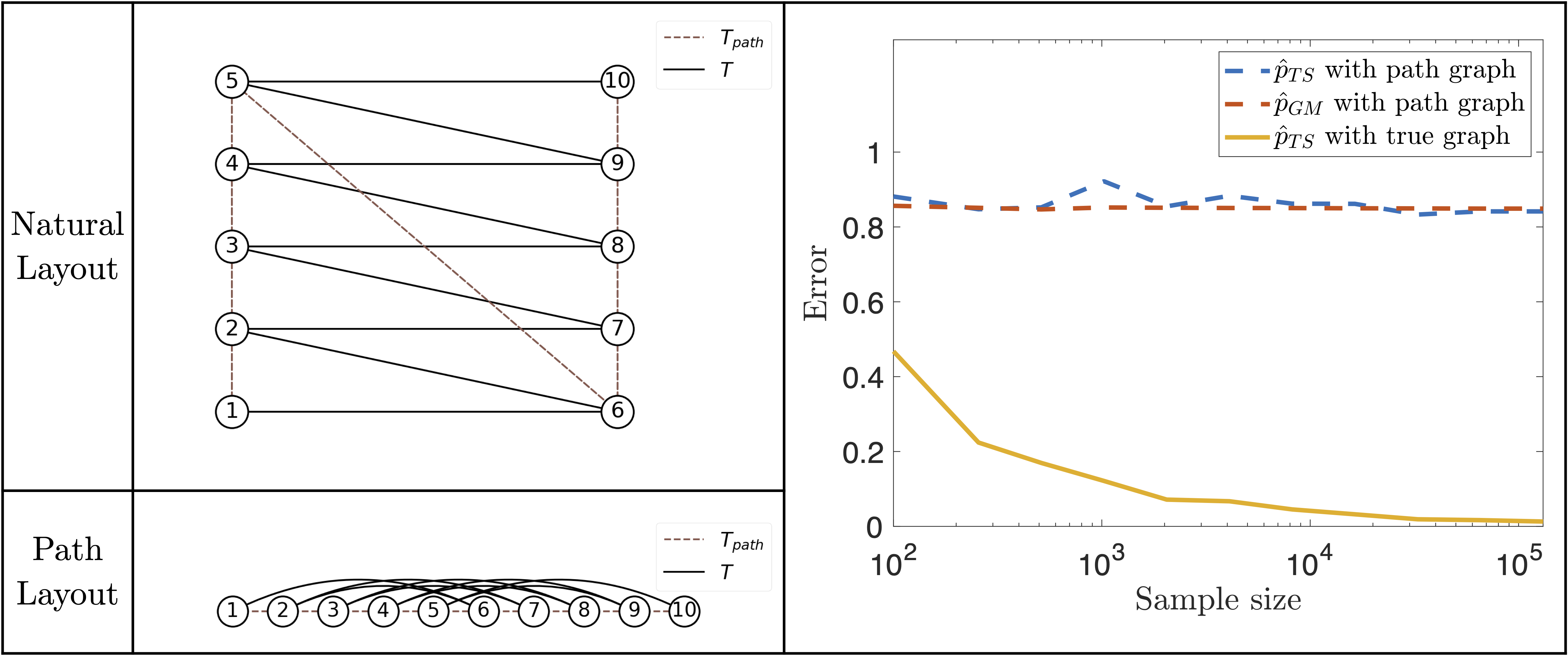}
         \caption{Bipartite Graph, $d = 10$}
         \label{fig:Bipartite Graph}
     \end{subfigure}
    \caption{The black solid line stands for edges on \(T\) and the dashed brown line stands for edges on \(\Tpa\). The graph is plotted in layout that is respectively natural for \(T\) and for \(\Tpa\). Error plot for each case of \(T\) are also included. One can see that convergence to true model only occurs with true tree specification. }
    \label{fig:natural vs path layout}
\end{figure}

In this numerical experiment, we aim to exhibit the importance of using a correct tree topology. We will focus on testing the TTNS-Sketch algorithm on a tree-based graphical model under input tree topology mis-specification. Specifically, given any fixed tree structure \(T = (V, E)\), we consider the following graphical model over \(T\):
\begin{equation}\label{eqn: graphical model over tree}
    p^\star(x_{1},\ldots, x_{d}) = \exp{\left(-\beta\sum_{(i,j) \in E}f_{i,j}(x_{i}, x_{j})\right)},
\end{equation}
where \(\beta > 0\) is the temperature parameter. We test a simple binary model where \(\beta = 1/2\) and \(f_{i,j} (x_{i}, x_{j}) = -x_{i}x_{j}\) with each \(x_{i} \in \{-1, 1\}\), which is the setting of standard ferromagnetic Ising model.

First, consider the case where \(T\) is a 10-node trident, see illustration in Figure \ref{fig:Trident Graph}. Even in such a simple example, one can see that there is not a good way to choose a path to fit the tree model. Suggested by our deliberate choice in ordering the variables, one reasonable candidate is a path graph \(\Tpa\) that traverses from node 1 to 10 in numerical order. Indeed, with only one exception of the edge \((4,8)\), the tree \(\Tpa\) is almost made up of edges from \(T\).  

Likewise, we include the 10-node dendrimer graph in Figure \ref{fig:Dendrimer Graph} and a bipartite graph in Figure \ref{fig:Bipartite Graph}, where we also use the numerical order to indicate the chosen \(\Tpa\) structure. While the dendrimer case also uses a reasonable path structure, one can see that the bipartite graph case uses a \(\Tpa\) structure very different from \(T\). In Figure \ref{fig:natural vs path layout}, the natural layout draws \(T\) and \(\Tpa\) in a layout natural to \(T\), where one can see how often \(\Tpa\) uses edge from \(T\). 

For the three tree structures, we will compare three methods (i)-(iii): (i) TTNS-Sketch over \(\Tpa\), (ii) Graphical modeling over \(\Tpa\), (iii) TTNS-Sketch over \(T\). In (i) and (iii), we choose the Markov sketch function. The results are listed in Figure \ref{fig:natural vs path layout}. One can see that the error for (iii) always converge to zero with large \(N\), which is consistent with Lemma \ref{lem: ttns-markov-exact}. However, both (i) and (ii) does not converge to \(p^{\star}\), with the performance being worst in the bipartite graph case.

In the path layout, one can more naturally ``count" the internal bond dimension necessary to let \(p^\star\) admit a TTNS ansatz under \(\Tpa\). Let \(\Epa\) stand for the edge set for \(\Tpa\), which is essentially \(\Epa = \{(i, i+1)\}_{i = 1}^{d-1}\). In this case, for any edge \(e = (i, i+1)\), one can calculate the internal bond \(r_{e}\) by counting the number of edges one would ``cut" if one places a vertical line in between node \(i\) and \(i+1\). If one counts \(q_i\) edges, then one has the upper bound \(r_{e} \leq 2^{q_i}\) for the Ising model. More generally, suppose \(X \sim p^{\star}\), and let each entry of \(X\) be a discrete variable over \(\{1, \ldots, n\}\), then \(r_{e}\) is upper bounded by \(n^{q_i}\). The description of \(q_i\) exactly coincides the number of edges across the partition \([d] = \{1,\ldots, i\} \cup \{i+1, \ldots, d\}\), i.e. the cardinality of the cut from the partition.

Thus, one can easily extend Figure \ref{fig:natural vs path layout} to cases with more nodes, and then the three models will have quite different behavior. The cut number \(q_i\) for the trident case is upper bounded by \(q_i \leq 2\), but largest \(q_i\) for the bipartite graph is \(d-1\) when \(d\) is even. Hence, in the bipartite case, the TTNS ansatz of \(p^\star\) under \(\Tpa\) is not practical to compute, while the TTNS ansatz of \(p^{\star}\) under \(T\) is simple. Importantly, the bipartite graph \(T\) is itself a path graph, and so the \(p^{\star}\) is not complicated. 

Thus, while modest tree structure mis-specification can be treated with higher bond dimension, a large structural deviation may lead to an intractable bond dimension penalty. We also note that in all of these cases, one has \(T_{\mathrm{CL}} = T\) with overwhelmingly high probability. Due to the \(O(d)\) sample complexity to recover \(T\), we will henceforth assume that one has access to the true tree model \(T\).

\subsection{Numerical case study: 1D spin system with non-local interactions}\label{sec: non-local}

In this example, we consider a more complicated Markov random field model. Given a fixed tree structure \(T = (V, E)\), we denote the shortest-path distance on \(T\) as \(\dist_{T}\). For any fixed positive integer \(l\), we propose the following model
\begin{equation}\label{eqn: general markov model}
    p^\star(x_{1},\ldots, x_{d}) = \exp{\left(-\beta\sum_{\dist_{T}{(i,j)} \leq l}f_{i,j}(x_{i}, x_{j})\right)}.
\end{equation}

This is also a graphical model over \(G_{T} = (V,E')\) by letting \(E' = \{(i,j) \mid \dist_{T}{(i,j)} \leq l\}\). In particular, we set \(l := 2\). Moreover, we consider the case where \(T\) is a path graph with \(d = 32\). See the illustration in Figure \ref{fig:non-local tree}.

Moreover, to demonstrate our algorithm under more general situations than binary data, we consider the 4-state clock model with \(f_{i,j}(x_{i}, x_{j}) =  \cos{(x_i - x_j)}\) and \(x_i \in \{0, \frac{1}{2}\pi, \pi, \frac{3 }{2}\pi\} \). In this case, we set \(\beta = 1/4\) to ensure the spin-spin correlation strength is at an appropriate level. 

\begin{figure}
     \centering
     \includegraphics[width=\textwidth]{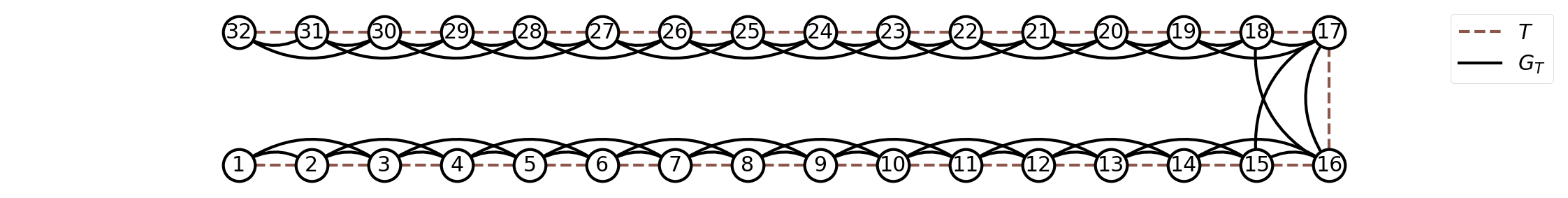}
     \caption{Illustration of the graphical model in Section \ref{sec: non-local}. The solid black line indicates edges for true graphical model \(G_{T}\). The dashed brown line indicates edges for \(T\), the tree model to be used for TTNS-Sketch.}
      \label{fig:non-local tree}
     \centering
     \includegraphics[width=0.6\textwidth]{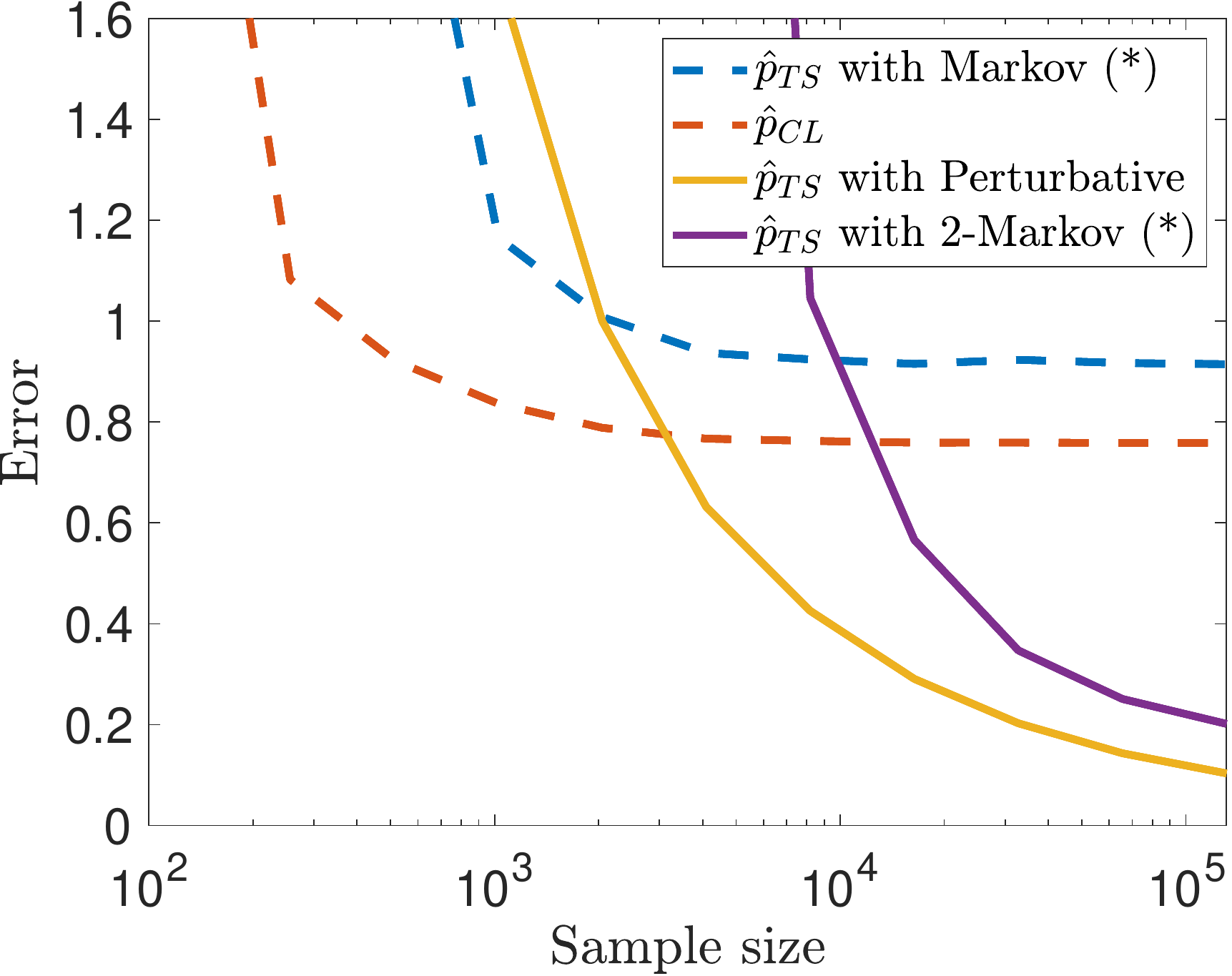}
    \caption{Error comparison for the graphical model in Figure \ref{fig:non-local tree}. The model \(\hat p_{\mathrm{CL}}\) stands for the result of direct graphical modeling with the tree structure \(T\) in Figure \ref{fig:non-local tree}. TTNS-Sketch with perturbative sketch function converges to the true model. The sign (*) stands for a truncation of the internal bond \(r_{e}\). One can see that, even with truncation, TTNS-Sketch with 2-Markov still performs worse than perturbative sketching without truncation. Moreover, due to the pre-selected internal bond truncation, the 2-Markov model cannot converge to the true model in the limit of samples.}
     \label{fig:2_nbhd_markov_path}
\end{figure}

If one applies Chow-Liu algorithm to the model in \eqref{eqn: general markov model}, one could obtain any one of the spanning tree of \(G_{T}\). Hence, for the sake of fair comparison, we fix the path graph \(T\) as the input tree graph to TTNS-Sketch. We test TTNS-Sketch under the following sketch function (i) perturbative sketching function, (ii) Markov sketch function, and (iii) 2-Markov sketch function (defined in Section \ref{sec: high-order markov sketch}). For the perturbative sketch function, we pick \(\epsilon = 0.05\) and a sketch core size of \(l_{e} = 20\) for any \(e \in E\).

We observe convergence for TTNS-Sketch under the perturbative sketching function. On the other hand, TTNS-Sketch with Markov and 2-Markov sketch function both encounter numerical blow-up for sample size up to \(N = 2^{17}\). To provide further comparison, we will include a ``truncated" version of both the Markov and 2-Markov model. For Markov model, we truncate to the target internal bond dimension \(r_{e} = 2\), down from \(r_{e} = n = 4\). For 2-Markov model, we truncate to the target internal bond dimension \(r_{e} = 6\), down from \(r_{e} = n^2 = 16\). In particular, the 2-Markov sketch function with \(r_{e} = n^2\) actually satisfies the exact recovery property, but the truncated \(r_{e} = 6\) version performs better and converges to a reasonable error. For benchmark, we also include direct graphical modeling. The result is shown in Figure \ref{fig:2_nbhd_markov_path}.

Therefore, one can see that the error analysis in Theorem \ref{thm:sample-complexity} is sensitive to the condition number provided by the given sketch function. In the case of large condition number, one might actually obtain a better result from truncating \(r_{e}\).Error analysis for TTNS-Sketch with truncation is out of the scope for this paper, but one can see that the perturbative sketch function actually performs quite well without any truncation.


\subsection{Numerical case study: tree graphical model with large variable dimension}\label{sec: large variables}

In this subsection, we will test the performance of TTNS-Sketch in the setting of large \(d\). We consider tree-based graphical models of the form \eqref{eqn: graphical model over tree} with \(\beta = 1/2\), \(f_{i,j} (x_{i}, x_{j}) = -x_{i}x_{j}\), and \(x_{i} \in \{-1, 1\}\), which is the case of Ising model. As a remark, cases under more general \(\{f_{i,j}\}_{(i,j) \in E}\) leads to the same conclusion, and so they are excluded for the sake of brevity. For the candidate graph \(T\), we consider a path graph with \(d = 100\) nodes and the 3-fractal dendrimer graph with \(d = 94\) nodes, see Figure \ref{fig:large variable graph}.

\begin{figure}
     \centering
     \begin{subfigure}[b]{0.495\textwidth}
         \centering
         \includegraphics[width=\textwidth]{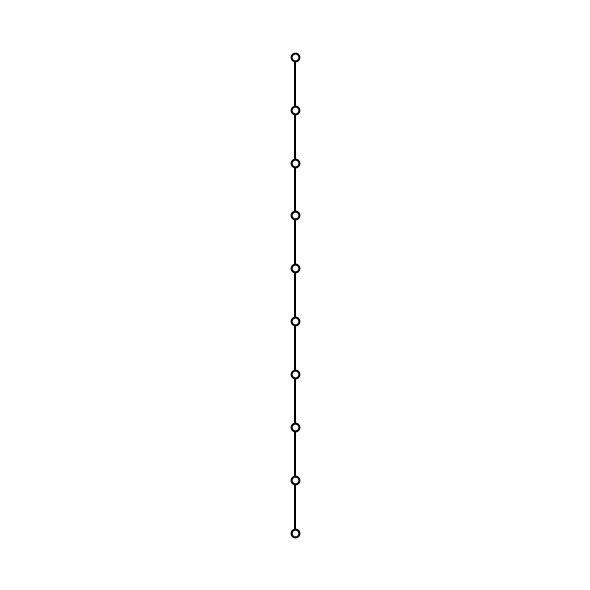}
         \caption{Path Graph (segment)}
         \label{fig:line}
     \end{subfigure}
     \begin{subfigure}[b]{0.495\textwidth}
         \centering
         \includegraphics[width=\textwidth]{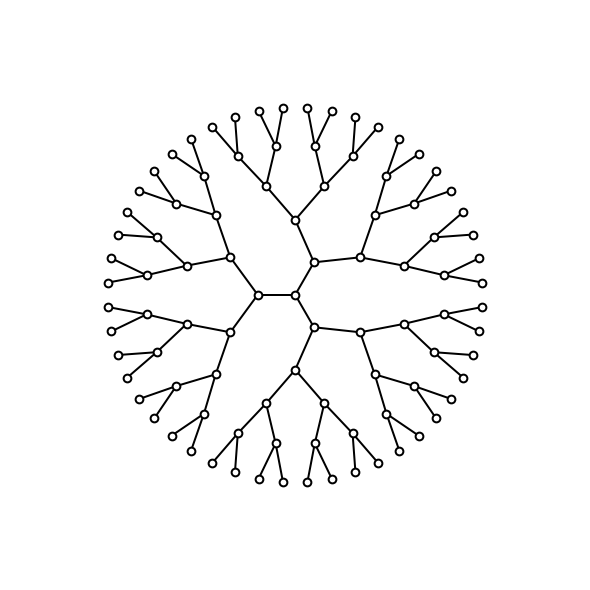}
         \caption{Dendrimer Graph with \(d = 94\)}
         \label{fig:dendrimer}
     \end{subfigure}
    \centering
    \caption{Plot of graphs considered in Section \ref{sec: large variables}}
    \label{fig:large variable graph}
     \begin{subfigure}[b]{0.495\textwidth}
         \centering
         \includegraphics[width=\textwidth]{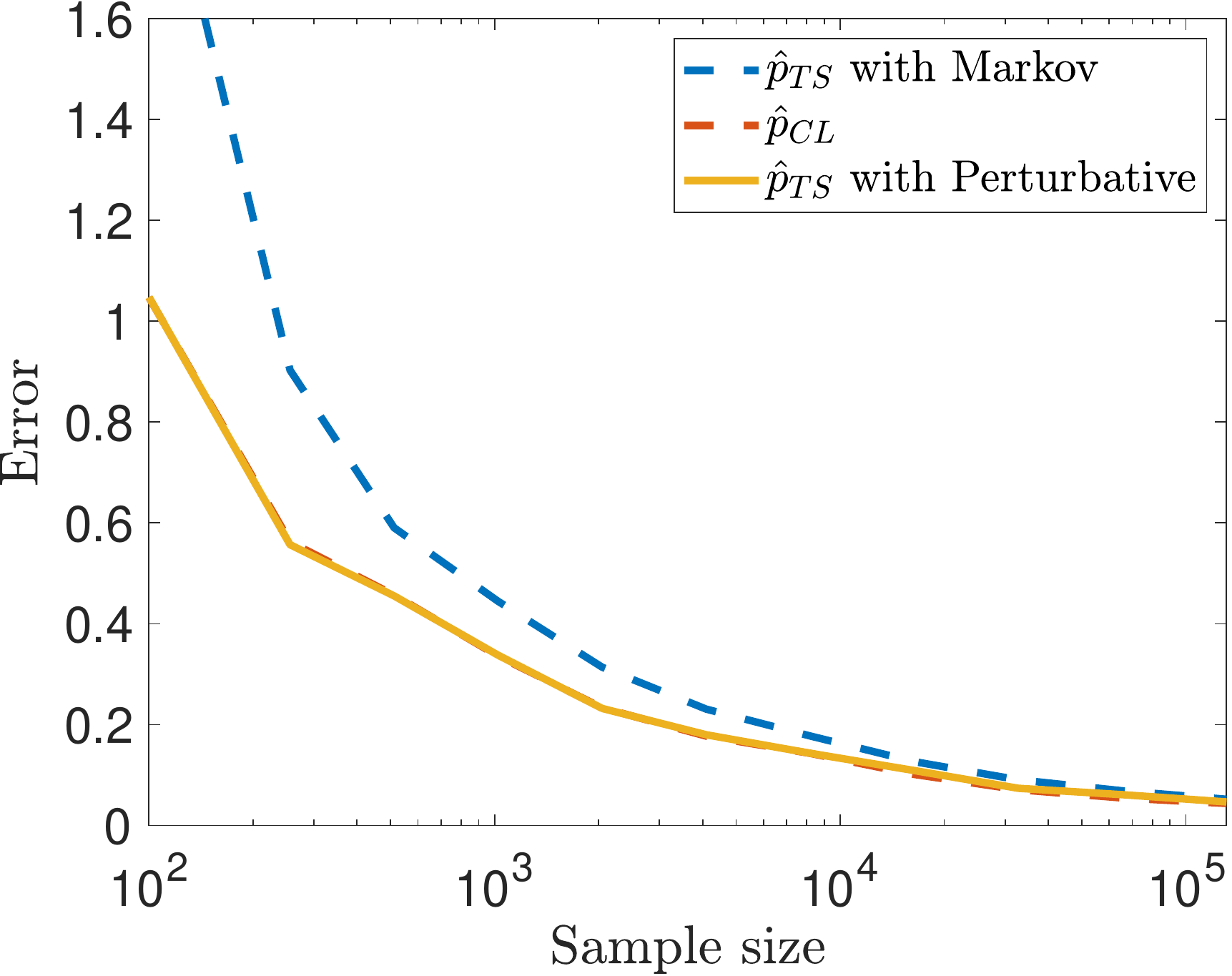}
         \caption{Path graph \(d = 100\) }
         \label{fig:line result}
     \end{subfigure}
     \begin{subfigure}[b]{0.495\textwidth}
         \centering
         \includegraphics[width=\textwidth]{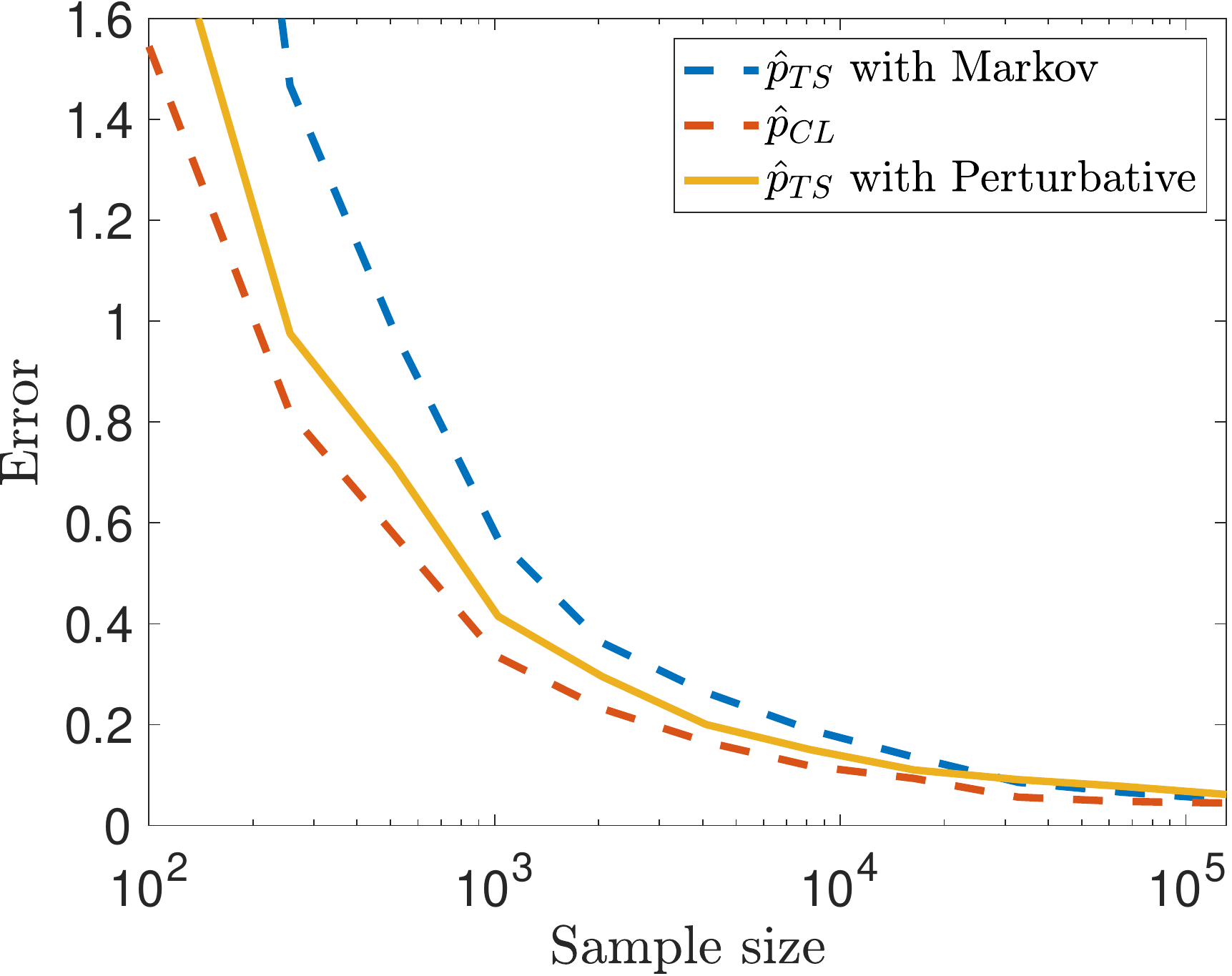}
         \caption{Dendrimer graph with \(d = 94\)}
         \label{fig:dendrimer result}
     \end{subfigure}
    \centering

    \caption{Error comparison for path graph with \(d = 100\) and dendrimer graph with \(d = 94\). The model \(\hat p_{\mathrm{CL}}\) stands for the result of direct graphical modeling with the true Tree structure \(T\). One can see that all methods converge to the true model. Among TTNS-Sketch based models, TTNS-Sketch with perturbative sketch function has the best performance. In Figure \ref{fig:line result}, \(\PTS\) with perturbative sketching has an error that is quite close to the benchmark \(\hat p_{\mathrm{CL}}\).}
    \label{fig:large variable result}
\end{figure}

For samples generated from the underlying model, direct graphical modeling with the underlying tree model \(T\) provides the best tree-based graphical model in the sense of MLE, and hence we put it as the benchmark. 
We compare the following sketch functions: (i) Markov sketch function, (ii) perturbative sketching function with \(\epsilon = 0.05\). The result can be seen in Figure \ref{fig:large variable result}. One can see that perturbative sketching significantly outperforms (i). In the path case, perturbative sketching even has an almost equivalent performance to direct graphical modeling. Due to the \(N = O(d^{2})\) scaling in Theorem \ref{thm:sample-complexity} and the \(O(d)\) computational scaling, one can likewise perform the same procedure up to very large \(d\).

\subsection{Numerical case study: spin system with long range interactions}\label{sec: long range}

\begin{figure}
    \centering
     \begin{subfigure}[b]{0.4\textwidth}
         \centering
         \includegraphics[width=\textwidth]{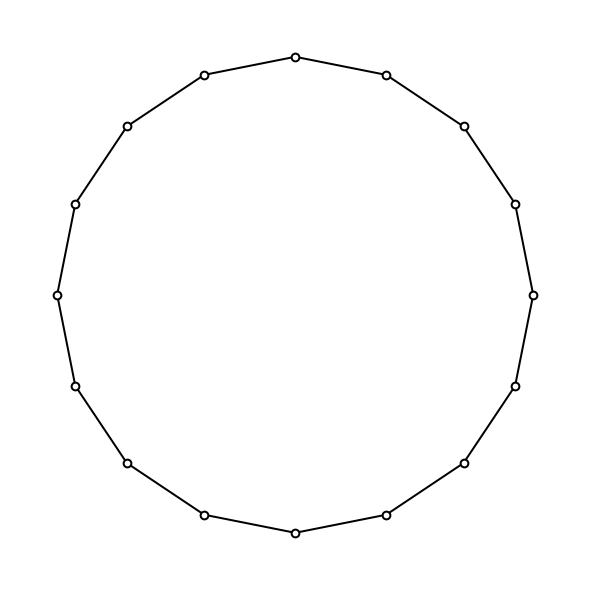}
         \caption{Ring model with \(d = 16\)}
         \label{fig:circle}
     \end{subfigure}
     \centering
     \begin{subfigure}[b]{0.4\textwidth}
         \centering
         \includegraphics[width=\textwidth]{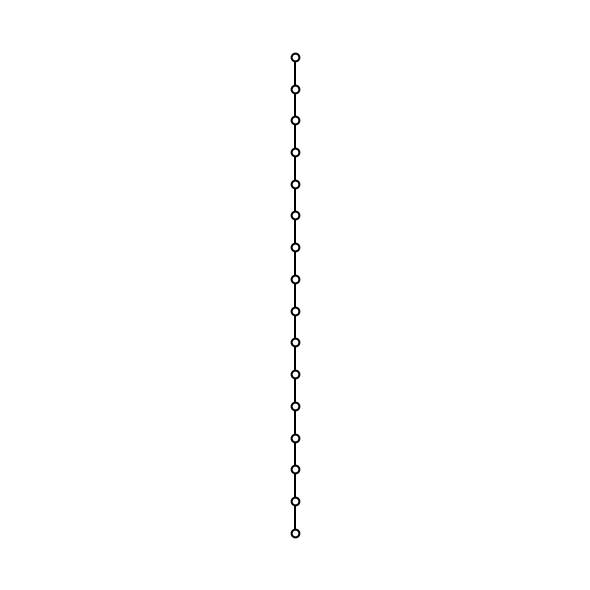}
         \caption{Linear path model with \(d = 16\)}
         \label{fig:line_16}
     \end{subfigure}
    \caption{Graph representations of the true model in Section \ref{sec: long range} and the path model learned by path-based graphical modeling.}
    \label{fig:graph plot}
\end{figure}
\begin{figure}
     \centering
     \includegraphics[width=0.6\textwidth]{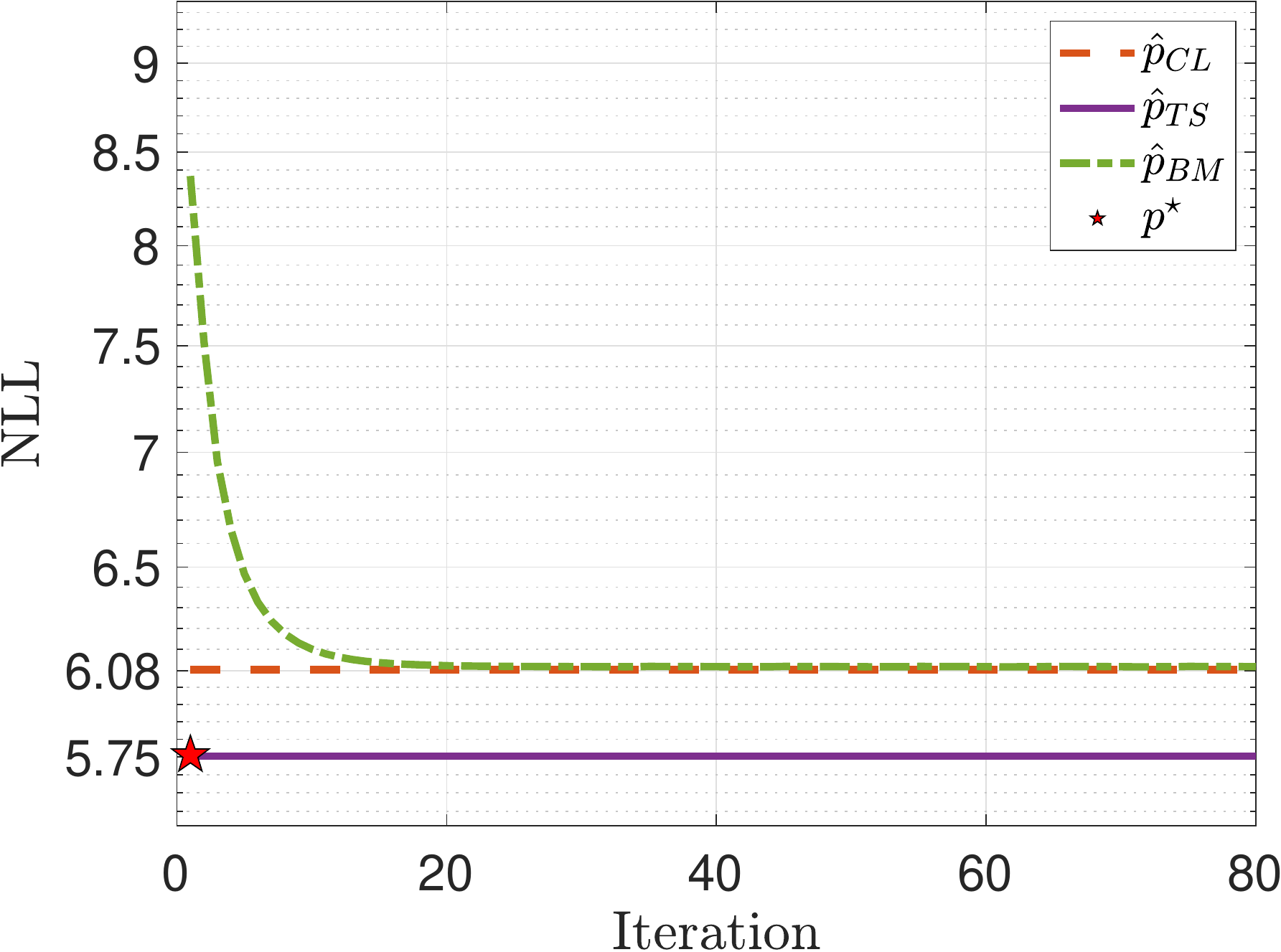}
    \caption{Negative Log Likelihood (NLL) level comparison for the ring graph with \(d = 16\) in Figure \ref{fig:circle}. Here, \(N = 2^{15} = 32,768\). The model \(\hat p_{\mathrm{CL}}\) stands for the result of direct graphical modeling with the path structure in Figure \ref{fig:line_16}. The model \(\hat p_{\mathrm{BM}}\) stands for the result of Born Machine with \(r_{\mathrm{max} = 4}\). The model \(\hat p_{\mathrm{TS}}\) stands for the result of TTNS-Sketch with the special high-order Markov sketch function in \eqref{eqn: choice of neighbor long range}. One can see that the output of TTNS-Sketch is close to the NLL level of the true model, while the output of BM is close to the NLL level of the path-based graphical model \(\hat p_{\mathrm{CL}}\).}
     \label{fig:BM_ring_example}

     \centering
     \includegraphics[width=0.6\textwidth]{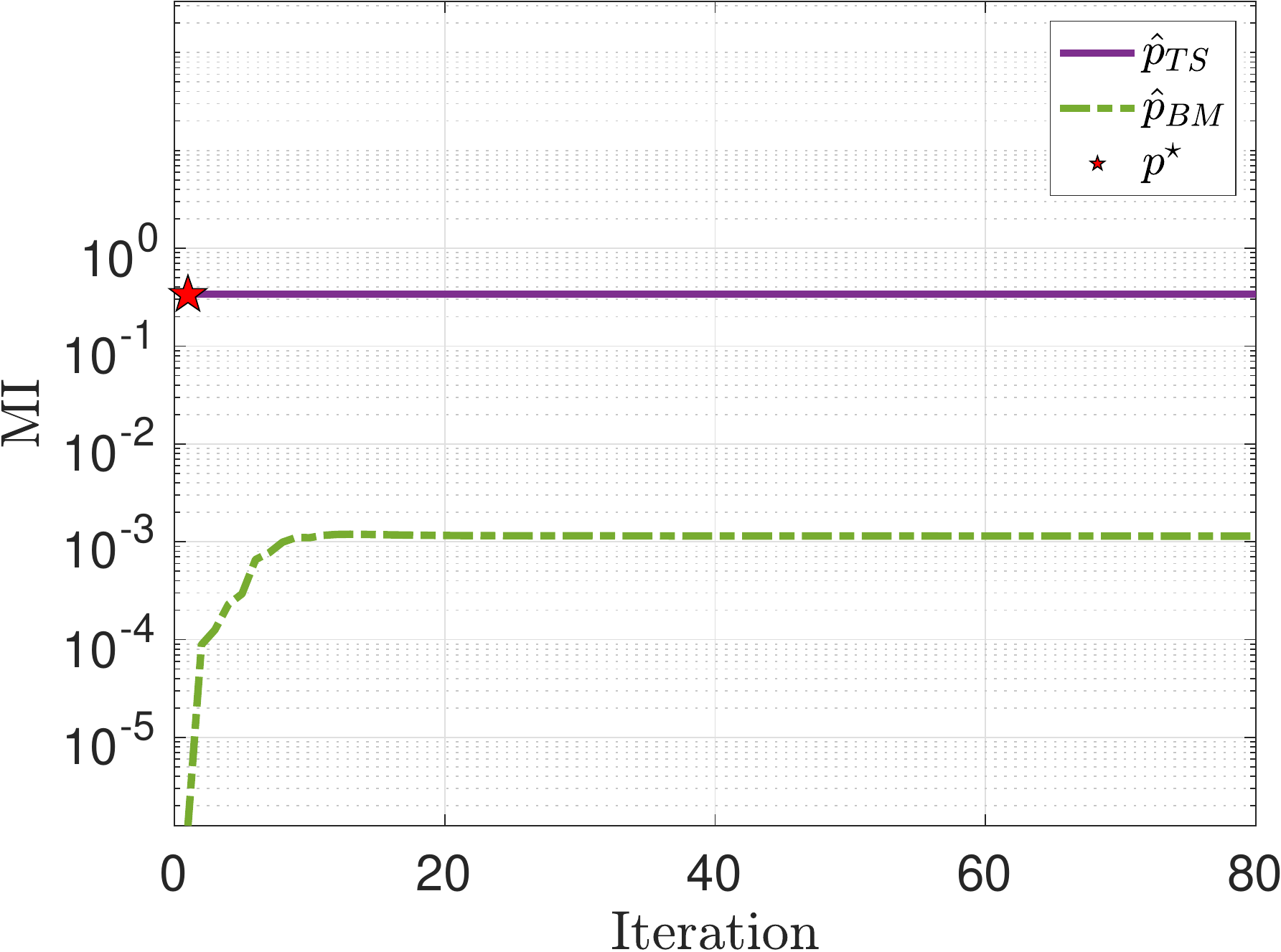}
    \caption{Mutual Information (MI) level comparison for the ring graph with \(d = 16\) in Figure \ref{fig:circle}, with experiment setting described in the caption of Figure \ref{fig:BM_ring_example}. The MI level of a distribution \(p\) refers to \(I(X_{1}, X_{d})\) where \(X \sim p\). In other words, the MI level is the mutual information between the last and first entry of the multivariate distribution according to \(p\). One can see that the mutual information level of BM is close to the zero throughout training, which is a further evidence that BM effectively learns a path graphical model.}
     \label{fig:BM_ring_exampleMI_plot}
\end{figure}

Quite importantly, in this numerical experiment, we introduce another important benchmark called Born Machine (BM). As BM is a method over the tensor train format, we restrict TTNS-Sketch to the path graph \(\Tpa\) to ensure fair comparison.

There are two important differences between BM and TTNS-Sketch. First, BM assures positivity of the trained model. Second, BM is based on Negative Log Likelihood (NLL) training, which is a more conventional error metric for statistical modeling. While TTNS-Sketch does not ensure the trained model is positive, the \(\|\cdot\|_{\infty}\) norm error bound obtained in Section \ref{sec:sample-comlexity} ensures that positivity is not a major issue if the correct tree structure and sketch function are provided.

During numerical experiment with BM, we have discovered a significant failure mode of BM, which is that it fails to converge to the true model for periodic spin systems. Similar observation has been independently made in \cite{gomez2021born}. To show the failure mode in the simplest possible setting, we discuss the case of a Markov random field with a ring graph \(G\) with \(d = 16\) nodes, and its graphical model representation of \eqref{eqn: periodic ising model} is illustrated in Figure \ref{fig:circle}:
\begin{equation}\label{eqn: periodic ising model}
    p^\star(x_{1},\ldots, x_{d}) = \exp{\left(-\sum_{i - j = 0 \mod d}f_{i,j}(x_{i}, x_{j})\right)}
\end{equation}



For TTNS-Sketch, we pick a special high order Markov sketch function, with the following neighborhood of interest: 
\begin{equation}
    \label{eqn: choice of neighbor long range}
    \mathcal{S}_{k} = \{k-1, k, k+1\} \cup \{1, d\},
\end{equation}
which is chosen to ensure convergence in the limit of sample size. Importantly, we include direct graphical modeling over the path graph \(\Tpa\) as a benchmark. 

We now discuss the numerical specification for the BM modeling. We pick the sample size \(N = 2^{15} = 32,768,\) which is sufficiently large so that the generalization error should be small. For the BM method, there is a parameter \(r_{\mathrm{max}}\), which is the largest allowed internal bond dimension, i.e. \(\max_{e \in E}r_e \leq r_{\mathrm{max}}\). Importantly, \(p^\star\) admits a tensor train ansatz with \(\max_{e \in E}r_e \leq 4\), and so we pick \(r_{\mathrm{max}} = 4\) to ensure that the BM method in this instance has no approximation error. The only training parameter is the learning rate and is picked according to cross-validation. 

In Figure \ref{fig:BM_ring_example}, we plot the result of the comparison between BM, TTNS-Sketch, and direct graphical modeling. Quite surprisingly, TTNS-Sketch can converge, while the NLL level of BM stays at a sub-optimal level. 
Interestingly, the NLL level of the BM model \(\hat p_{\mathrm{BM}}\) converges to that of the graphical model \(\hat p_{\mathrm{CL}}\). Moreover, one can also show that the KL-divergence between \(\hat p_{\mathrm{BM}}\) and \(\hat p_{\mathrm{CL}}\) is quite small. In the converged model plotted in Figure \ref{fig:BM_ring_example}, one has \( \KL{\hat p_{\mathrm{BM}}}{p^\star} \approx 0.4357\), while \( \KL{\hat p_{\mathrm{BM}}}{\hat p_{\mathrm{CL}}} \approx 0.0266\). 

Numerically, we can then conclude that the BM training implicitly leads to a sub-optimal model, and the gap to true model can be explained by the sub-optimal model's closeness to the path-based graphical model \(\hat p_{\mathrm{CL}}\), see Figure \ref{fig:graph plot} for illustration. As an additional evidence, during in the training dynamics, the learned model fails to capture the correlation between node \(1\) and node \(d\). In Figure \ref{fig:BM_ring_exampleMI_plot}, the mutual information level between node \(1\) and node \(d\) is plotted throughout training, and one can see the level is consistently close to zero, despite the strong correlation in \(p^\star\).



In practice, when \(d = 16\), this phenomenon of BM training failure persists up to \(r_{\mathrm{max}} = 7\) , and is resolved by setting \(r_{\mathrm{max}} \geq 8\). However, if \(d\) is increased, we have observed that the smallest \(r_{\mathrm{max}}\) for successful BM training also increases. This coupling of maximal bond dimension with \(d\) is problematic, as one could no longer expect a linear dependency between parameter size and the dimension \(d\). 

On the other hand, for TTNS-Sketch under \eqref{eqn: choice of neighbor long range}, one is theoretically guaranteed convergence to true model with \(r_{e} \leq 4\), the proof of which can be obtained by adapting the proof for Lemma \ref{lem: ttns-markov-exact}. Moreover, for TTNS-Sketch with sketch function in \ref{eqn: choice of neighbor long range}, increasing \(d\) does not affect performance beyond the sample complexity scaling in \(d\) as discussed in Section \ref{sec:sample-comlexity}.



\section{Conclusion}
We describe an algorithm TTNS-Sketch, which obtains a Tree Tensor Network States representation of a probability density from a collection of its samples. This is done by formulating a sequence of equations,
one for each core, which can be solved independently. This is a general framework which allows for arbitrary tree structures to be used. We have also compared this algorithm with similar training-based regimes in the tensor train format, in which we have shown much better performance from TTNS-Sketch even in a simple case of periodic spin systems. For models which have interaction beyond immediate neighbors, we have shown that TTNS-Sketch with perturbative sketching greatly outperforms Chow-Liu. Theoretically, we have provided condition for TTNS-Sketch to be a consistent estimator, as well as a reasonable sample complexity upper bound. 

While TTNS-Sketch might not necessarily be superior at approximating to an arbitrary density, the numerical and theoretical evidence gathered point to the conclusion that it is good at the inverse problem of solving for the tensor component of a density with a TTNS ansatz. In other words, a deterministic linear algebraic subroutine is sufficient to approximate a \(p^{\star}\) with a low-rank TTNS format. While this point has been made for tensor completion problems, it is quite remarkable that statistical inference of models with a TTNS ansatz can be shown to reduce to a linear algebraic problem.
\bibliographystyle{apalike}
\bibliography{ref}

\appendix
\section{Proof of Theorem \ref{thm:ttns-existence}}\label{appendix: proof of basic ttns existence}

\begin{proof} (Of Theorem \ref{thm:ttns-existence})
    For simplicity, for the remainder of the proof, we fix a structure for all of the high-order tensors we use. For \(\Phi_{w \to k}\), one reshapes it to the unfolding matrix \(\Phi_{w \to k}(x_{\leftside(w) \cup w}; \alpha_{(w, k)})\). For \(\Psi_{w \to k}\), one reshapes it to the unfolding matrix \(\Psi_{w \to k}(\alpha_{(w, k)}; x_{\rightside(w)})\). For \(\Phi_{\child(k) \to k}\), one reshapes it to the unfolding matrix \(\Phi_{\child(k) \to k}( x_{\leftside(k)}; \alpha_{(k, \child(k))})\). For \(G_{k}\), one reshapes it to the unfolding matrix \(G_{k}(\alpha_{(k, \child(k))}; x_{k}, \alpha_{(k, \parent(k))})\). For \(p\), we will explicitly write out the unfolding matrix structure to avoid ambiguity.

    According to Condition \ref{cond: TTNS ansatz condition}, for any edge $w \to k$, one has 
    \begin{equation}
    \label{eq:unfolding-decomposition}
        p(x_{\leftside(w)\cup w}; x_{\rightside(w)}) =  \Phi_{w \to k} \Psi_{w \to k}.
    \end{equation}
    
    Then, define $\Phi_{w \to k}^\dagger \colon [r_{(w, k)}] \times \left[\prod_{i \in \leftside(w) \cup w} n_i\right] \to \bR$ and $\Psi_{w \to k}^\dagger \colon \left[\prod_{i \in \rightside(w)} n_i\right] \times [r_{(w, k)}] \to \bR$
    so that $\Phi_{w \to k}^\dagger (\alpha_{(w, k)}; x_{\leftside(w) \cup w})$ denotes the pseudoinverse of $\Phi_{w \to k}$, and $\Psi_{w \to k}^\dagger(x_{\rightside(w)}; \alpha_{(w, k)})$ denotes the pseudoinverse of $\Psi_{w \to k}(\alpha_{(w, k)}; x_{\rightside(w)})$. Then, 
    \begin{equation*}
        \Phi_{w \to k}^\dagger\Phi_{w \to k} = \Psi_{w \to k}\Psi_{w \to k}^\dagger = \mathbb{I}_{r_{(w, k)}}
    \end{equation*}
    

    First, we prove uniqueness of each equation of \eqref{eq:ttn-determining} in the sense of least squares. Note that an exact solution is guaranteed when \(k\) is leaf, and so one only needs to consider when \(k\) is non-leaf. By assumption, $\Phi_{w \to k}$ has full column rank of $ r_{w \to k}$. In particular, the Kronecker product structure ensures that $\Phi_{\child(k) \to k}$ has full column rank of $ \prod_{w \in \child(k)} r_{w \to k}$. Therefore, a unique of solution to \eqref{eq:ttn-determining} exists in the sense of least squares.
    
    Moreover, when \(k\) is non-leaf and non-root, the pseudoinverse \(\Phi_{\child(k) \to k}^\dagger(\alpha_{(k, \child(k))}; x_{\leftside(k)})\) leads to the following explicit construction of \(G_k\):
    \begin{equation}
    \label{eq: G_k construction by unfolding}
        G_k = \Phi_{\child(k) \to k}^\dagger\Phi_{k \to \parent(k)},
    \end{equation}
    and likewise when \(k\) is root, one has
    \begin{equation}
        \label{eq: G_k construction by unfolding root case}
        G_k = \Phi_{\child(k) \to k}^\dagger p(x_{\leftside(k)};  x_{k}).
    \end{equation}
    
    
    To verify that \eqref{eq:ttn-determining} holds exactly for the construction of \(G_{k}\) in \eqref{eq: G_k construction by unfolding}, one can argue it suffices to check that
    \begin{equation}
        \label{eq: G_k equation wts}
        \Phi_{\child(k) \to k}\Phi_{\child(k) \to k}^\dagger p(x_{\leftside(k)};  x_{k \cup \rightside(k)}) = p(x_{\leftside(k)};  x_{k \cup \rightside(k)}),
    \end{equation}
    for which we give a brief explanation.
    When \(k\) is root, \eqref{eq: G_k construction by unfolding root case} implies that \eqref{eq: G_k equation wts} coincides with \eqref{eq:ttn-determining} for when \(k\) is root. When \(k\) is non-root and non-leaf, one can multiply both sides of \eqref{eq: G_k equation wts} by \(\Psi_{k \to \parent(k)}^\dagger\) and sum over \(x_{\rightside(k)}\). According to \eqref{eq: G_k construction by unfolding}, the obtained equation coincides with \eqref{eq:ttn-determining} for when \(k\) is non-root and non-leaf.
    
    It remains to show that \eqref{eq: G_k equation wts} holds. For an edge \(w \to k \in E\), define an intermediate term \(Q_{w \to k}\) as follows 
    \[
    Q_{w \to k}(x_{\leftside(w) \cup w}; y_{\leftside(w) \cup w})  :=  \sum_{\alpha_{(w,k)}}\Phi_{w \to k}(x_{\leftside(w) \cup w}; \alpha_{(w, k)})\Phi_{w \to k}^{\dagger}(\alpha_{(w, k)}; y_{\leftside(w) \cup w}).
    \]
    Then, for a generic tensor \(f \colon [n_{1}] \times \ldots \times [n_{d}] \to \bR\), one can define a projection operator \(P_{w \to k}\) as follows
    \[
    (P_{w \to k}f)(x_{1}, \ldots, x_{d}) = \sum_{y_{\leftside(w) \cup w}}Q_{w \to k}(x_{\leftside(w) \cup w}, y_{\leftside(w) \cup w})f(y_{\leftside(w) \cup w}, x_{\rightside(w)}).
    \]
    
    By commutativity of the sum operations involved, one has
    \begin{align*}
            \Phi_{\child(k) \to k}\Phi_{\child(k) \to k}^\dagger f &= \sum_{y_{\leftside(k)}}\left(\prod_{w \in \child(k)}{Q_{w \to k}(x_{\leftside(w) \cup w}; y_{\leftside(w) \cup w})}\right)f(x_{1}, \ldots, x_{d})\\
            &= \sum_{\substack{y_{\leftside(w) \cup w} \\ w \in \child(k)}}\left(\prod_{w \in \child(k)}{Q_{w \to k}(x_{\leftside(w) \cup w}; y_{\leftside(w) \cup w})}\right)f(x_{1}, \ldots, x_{d})
            \\
            &= \left(\prod_{w \in \child(k)}P_{w \to k}\right)f
    \end{align*}
    Thus, \eqref{eq: G_k equation wts} holds if one can show that \(P_{w \to k}p = p\) for any \(w \in \child(k)\), but this fact is straightforward:
    \[
    P_{w \to k}p
    =\Phi_{w \to k}\Phi_{w \to k}^\dagger p(x_{\leftside(w)\cup w}; x_{\rightside(w)}) = \Phi_{w \to k}\Phi_{w \to k}^\dagger \Phi_{w \to k}\Psi_{w \to k} =\Phi_{w \to k}\Psi_{w \to k} = p,
    \]
    and thus \eqref{eq:ttn-determining} exactly holds for the constructed $\{G_i\}_{i =1}^{d}$.

    Lastly, we prove that the solution $\{G_i\}_{i =1}^{d}$ forms a TTNS tensor core of \(p\). To show this result, it will be much more convenient to use the notion of subgraph TTNS function in Definition \ref{defn: TTNS on subgraphs}. We remark that the construction in Definition \ref{defn: TTNS on subgraphs} is only arithmetic and has no dependency on this theorem. For every node \(k \in [d]\), define a subset \(\mathcal{S}_{k}  :=  \leftside(k) \cup \{k\}\). Then, for non-root \(k\), we prove that \(\Phi_{k \to \parent(k)}\) is the subgraph TTNS function over \(\{G_{i}\}_{i = 1}^{d}\) and \(T_{\mathcal{S}_{k}}\), i.e. one wishes to show
    \begin{equation}
    \label{eqn: thm existence induction hypothesis}
            \Phi_{k \to \parent(k)}(x_{\leftside(k) \cup k}, \alpha_{k \to \parent(k)}) = \sum_{\substack{\alpha_{e} \\ e \not = (k, \parent(k))}} \prod_{i \in \leftside(k) \cup k} G_{i}\left(x_{i}, \alpha_{(i, \neighbor(i))}\right).
    \end{equation}
    
    We prove \eqref{eqn: thm existence induction hypothesis} by mathematical induction. Notice that \eqref{eq:ttn-determining} proves \eqref{eqn: thm existence induction hypothesis} when \(k\) is leaf node. Then, suppose that \(k\) is non-leaf and suppose by mathematical induction that \(\Phi_{w \to k}\) satisfies \eqref{eqn: thm existence induction hypothesis} for all \(w \in \child(k)\). Then, one can rewrite \eqref{eq:ttn-determining} by plugging in \(\Phi_{\child(k) \to k}\) by the form of each \(\Phi_{w \to k}\) according to \eqref{eqn: thm existence induction hypothesis}, and the obtained equation is exactly \eqref{eqn: thm existence induction hypothesis} for \(\Phi_{k \to \parent(k)}\). By induction over nodes by topological order, \eqref{eqn: thm existence induction hypothesis} then holds for every non-root \(k\). 
    
    By the same logic, now consider \eqref{eq:ttn-determining} when \(k\) is root, and one plugs in \(\Phi_{\child(k) \to k}\) by the form of each \(\Phi_{w \to k}\) according to \eqref{eqn: thm existence induction hypothesis}, whereby the obtained equation is exactly \eqref{eq:ttn-contraction} in Definition \ref{def: TTNS short def}, and thus $\{G_i\}_{i =1}^{d}$ does form a TTNS tensor core of \(p\).
\end{proof}

\section{Proof of Theorem \ref{thm:ttns-algorithm-recovery}}\label{appendix: proof of basic ttns recovery}

\begin{proof} (of Theorem \ref{thm:ttns-algorithm-recovery})
    For any non-root $k$, note that $Z^{\star}_k$ is assumed to be of rank \(r_{(k, \parent(k))}\) by (ii) in Condition \ref{cond: Condition for sketch functions}. Let $Q^{\star}_k$ be as in \eqref{eqn: def of Q_k^T star}. In other words, $Q^{\star}_k(\gamma_{(k, \parent(k))}; \alpha_{(k, \parent(k))})$ is formed by the rank-\(r_{(k, \parent(k))}\) SVD decomposition of $Z^{\star}_k$ in the \textsc{SystemForming} step of Algorithm \ref{alg:1}. Thus $Q^{\star}_k(\gamma_{(k, \parent(k))}; \alpha_{(k, \parent(k))})$ is of rank \(r_{(k, \parent(k))}\), which means it has full column rank. We define
    \begin{equation}
    \label{eqn: temporary def of unfolding matrix}
        \Phi^{\star}_{k \to \parent(k)}(x_{\leftside(k) \cup k}, \alpha_{(k, \parent(k))})  :=  \sum_{\gamma_{(k, \parent(k))}} \bar{\Phi}^{\star}_k(x_{\leftside(k) \cup k}, \gamma_{(k, \parent(k))}) Q^{\star}_k(\gamma_{(k, \parent(k))}, \alpha_{(k, \parent(k))}).
    \end{equation}
    
    Due to (i) in Condition \ref{cond: recurisve sketching} and \(Q_{k}\) having full rank, one can conclude that $\Phi^{\star}_{k \to \parent(k)}(x_{\leftside(k) \cup k}; \alpha_{(k, \parent(k))})$ and $\Phi^{\Delta}_{k \to \parent(k)}(x_{\leftside(k) \cup k}; \alpha_{(k, \parent(k))})$ have the same column space. Thus, there exists $\Psi^{\star}_{(w, k)}$'s such that $\{\Phi^{\star}_{(w, k)}, \Psi^{\star}_{(w, k)}\}_{(w, k) \in E}$ forms a collection of the low-rank decomposition of \(p^{\star}\) in the sense of Condition \ref{cond: TTNS ansatz condition}. We make the following claim, which also justifies \(\star\) upper-index in \eqref{eqn: temporary def of unfolding matrix}:
    \begin{center}
        Claim: $\{\Phi^{\star}_{(w, k)}\}_{(w, k) \in E}$ as defined in \eqref{eqn: temporary def of unfolding matrix} satisfies Condition \ref{cond: ttns gauge choice}
    \end{center}

    The proof of the claim is somewhat technical and we will defer the proof after stating how it proves this theorem.
    
    Assume the claim is correct and $\{\Phi^{\star}_{(w, k)}\}_{(w, k) \in E}$ satisfies Condition \ref{cond: ttns gauge choice}. As a consequence, if one defines \(\{A^{\star}_i, B^{\star}_i\}\) by
    \[\{A^{\star}_i, B^{\star}_i\}_{i = 1}^{d} \leftarrow \textsc{SystemForming}(\{Z^{\star}_{w \to k}\}_{w \to k \in E}, \{Z^{\star}_i\}_{i = 1}^{d}),\]
    then one can alternatively define \(\{A^{\star}_i, B^{\star}_i\}_{i = 1}^{d}\) by \eqref{eqn: def of A^ast and B^ast} with \[\{\Phi_{(w, k)}^\Delta\}_{(w, k) \in E} \gets \{\Phi_{(w, k)}^\star\}_{(w, k) \in E}.\]
    Thus, with the alternative definition in \eqref{eqn: def of A^ast and B^ast}, it follows that \eqref{eq:alg-CDEs noiseless} is a (possibly over-determined) linear system formed by a linear projection of the linear system in \eqref{eq:ttn-determining}, where chosen gauge is $\{\Phi^{\star}_{(w, k)}\}_{(w,k) \in E}$. 
    
    Due to Theorem \ref{thm:ttns-existence}, \eqref{eq:ttn-determining} is an over-determined linear system with a unique and exact solution. 
    Theorem \ref{thm:ttns-existence} guarantees an exact solution $\{G^{\star}_i\}_{i = 1}^{d}$ to \eqref{eq:ttn-determining}, which is then necessarily an exact solution to \eqref{eq:alg-CDEs noiseless}. If \eqref{eq:alg-CDEs noiseless} satisfies uniqueness of solution, then the solution to \eqref{eq:alg-CDEs noiseless} is a solution to \eqref{eq:ttn-determining}, which by Theorem \ref{thm:ttns-existence} forms a TTNS tensor core of $p^\star$. Therefore, it suffices to check uniqueness. Uniqueness of solution to \eqref{eq:alg-CDEs noiseless} when \(k\) is leaf is trivial. When \(k\) is non-leaf, note that one can apply (iii) in Condition \ref{cond: Condition for sketch functions} with \(\{\Phi_{(w, k)}^\star\}_{(w, k) \in E}\) as the chosen gauge, which guarantees that $A^{\star}_k(\beta_{(k, \child(k))}, \alpha_{(k, \child(k))})$ has the full column rank for every non-leaf $k$. In other words, one is guaranteed uniqueness of the solution to \eqref{eq:alg-CDEs noiseless}, as desired. For the assertion on the consistency of \(\hat G_{k}\), note that \(\lim_{N \to \infty} \hat G_{k} = G^{\star}_{k}\) follows from the fact that \(\lim_{N \to \infty} \hat A_{k} = A^{\star}_{k}\) and \(\lim_{N \to \infty} \hat B_{k} = B^{\star}_{k}\).

    We now prove that $\{\Phi^{\star}_{(w, k)}\}_{(w, k) \in E}$ satisfies Condition \ref{cond: ttns gauge choice}. For a clear exposition, we adopt the unfolding 3-tensor structure developed in Section \ref{sec: 3-tensor preliminrary}-\ref{sec: local error to global}. We remark that the 3-tensor construction is only arithmetic and does not depend on validity of this theorem. 
    
    For \(Z^{\star}_k\), we reshape it as \(Z^{\star}_{k}(\beta_{(k,\child(k))}; x_{k}; \gamma_{(k, \parent(k))})\). For \(U^{\star}_{k}\), we reshape it as \(U^{\star}_{k}(\beta_{(k,\child(k))}; x_{k}; \alpha_{(k, \parent(k))})\). For \(Q^{\star}_{k}\), we reshape it as \(Q^{\star}_{k}(\gamma_{(w, k)}; 1;\alpha_{(w, k)} )\).  
    For \(S_{k}, T_{k}\), we reshape as \(S_{k}(\beta_{(k,\child(k))};1; x_{\leftside(k)})\) and \( T_{k}(x_{\rightside(k)}; 1; \gamma_{(k, \parent(k))})\). For \(\Phi^{\star}_{k \to \parent(k)}\), we reshape it as \(\Phi^{\star}_{k \to \parent(k)}(x_{\leftside(k)};  x_{k}; \alpha_{(k, \parent(k))})\). For \(\bar{\Phi}^{\star}_k(x_{\leftside(k) \cup k}, \gamma_{(k, \parent(k))})\), we reshape it as \(\bar{\Phi}^{\star}_k(x_{\leftside(k)}; x_{k}; \gamma_{(k, \parent(k))})\). For \(p^{\star}\), we reshape it as \(p^{\star}(x_{\leftside(k)}; x_{k}; x_{\rightside(k)})\).
    
    Then, by the construction of \(Q^{\star}_{k}\) in the \textsc{SystemForming} step of Algorithm \ref{alg:1}, one has \(U_{k}^{\star}  = Z_k^{\star} \circ Q^{\star}_{k}\). By \eqref{eqn: temporary def of unfolding matrix}, it follows \(\Phi^{\star}_{k \to \parent(k)} :=  \bar{\Phi}^{\star}_k \circ Q^{\star}_{k}\). Condition \ref{cond: ttns gauge choice} is satisfied if \(
    U_{k}^{\star}  = S_{k} \circ \Phi^{\star}_{k \to \parent(k)}.\) With such a choice of unfolding 3-tensor, one obtains a simple proof as follows
    \[
    U_{k}^\star = Z_k^{\star} \circ Q^{\star}_{k} = S_k \circ p^\star \circ T_k \circ Q^{\star}_{k} = S_{k} \circ \bar{\Phi}^{\star}_k \circ Q^{\star}_{k} = S_{k} \circ \Phi^{\star}_{k \to \parent(k)},
    \]
    where the first equality comes from \(Z_k^{\star} =S_k \circ p^\star \circ T_k\), the second equality comes from \(\bar{\Phi}^{\star}_k =  p^\star \circ T_k\), and the third equality comes from \(\Phi^{\star}_{k \to \parent(k)}=  \bar{\Phi}^{\star}_k \circ Q^{\star}_{k}\). Thus the claim holds and we are done.

\end{proof}

\section{Proof of Lemma \ref{lem: ttns-markov-exact}}

\begin{lemma}
\label{lem:markov-col/row-spaces}
    Suppose $p$ satisfies the Markov property given a rooted tree $([d], E)$.
    For any subsets $\mathcal{S}_1 \subset \leftside(k) \cup k$ and $\mathcal{S}_2 \subset \rightside(k)$,
    \begin{itemize}
        \item[(i)] $\mathcal{M}_{\mathcal{S}_1 \cup \mathcal{S}_2} p(x_{\mathcal{S}_1}; x_{\mathcal{S}_2})$ and $\mathcal{M}_{\mathcal{S}_1 \cup \parent(k)} p(x_{\mathcal{S}_1} ; x_{\parent(k)})$ have the same column space if $\parent(k) \in \mathcal{S}_2$,
        \item[(ii)] $\mathcal{M}_{\mathcal{S}_1 \cup \mathcal{S}_2} p(x_{\mathcal{S}_1} ; x_{\mathcal{S}_2})$ and $\mathcal{M}_{k \cup \mathcal{S}_2} p(x_k; x_{\mathcal{S}_2})$ have the same row space if $k \in \mathcal{S}_1$.
    \end{itemize}
\end{lemma}
\begin{proof}
Define a conditional probability tensor as follows:
\begin{equation*}
    \mathcal{M}_{\mathcal{S}_1 | \mathcal{S}_2} p(x_{\mathcal{S}_1}, x_{\mathcal{S}_2})  :=  \mathbb{P}_{X \sim p}\left[X_{\mathcal{S}_1} = x_{\mathcal{S}_1} | X_{\mathcal{S}_2} = x_{\mathcal{S}_2}\right].
\end{equation*}

Due to the conditional independence property for graphical models, one can write
\begin{align*}
    \mathcal{M}_{\mathcal{S}_1 \cup \mathcal{S}_2} p(x_{\mathcal{S}_1}, x_{\mathcal{S}_2}) &=
    \mathcal{M}_{\mathcal{S}_1 | \parent(k)} p(x_{\mathcal{S}_1}, x_{\parent(k)})
    \mathcal{M}_{\parent(k)} p(x_{\parent(k)})
    \mathcal{M}_{\mathcal{S}_2 | \parent(k)} p(x_{\mathcal{S}_2 \backslash \parent(k)}, x_{\parent(k)}) \\
    &=
    \mathcal{M}_{\mathcal{S}_1 \cup \parent(k)} p(x_{\mathcal{S}_1}, x_{\parent(k)})
    \mathcal{M}_{\mathcal{S}_2 | \parent(k)} p(x_{\mathcal{S}_2 \backslash \parent(k)}, x_{\parent(k)})
\end{align*}

Thus, the column space of $\mathcal{M}_{\mathcal{S}_1 \cup \mathcal{S}_2} p(x_{\mathcal{S}_1}; x_{\mathcal{S}_2})$ depends solely on $\mathcal{M}_{\mathcal{S}_1 \cup \parent(k)}(x_{\mathcal{S}_1}; x_{\parent(k)})$. Therefore, (i) holds. 

Similarly, 
\begin{equation*}
    \mathcal{M}_{\mathcal{S}_1 \cup \mathcal{S}_2} p(x_{\mathcal{S}_1}, x_{\mathcal{S}_2}) =
    \mathcal{M}_{\mathcal{S}_1 | k} p(x_{\mathcal{S}_1 \backslash k} , x_{k})
    \mathcal{M}_{k \cup \mathcal{S}_2 } p(x_{k}, x_{\mathcal{S}_2}),
\end{equation*}
which shows that the row space of $\mathcal{M}_{\mathcal{S}_1 \cup \mathcal{S}_2} p(x_{\mathcal{S}_1}; x_{\mathcal{S}_2})$ depends solely on $\mathcal{M}_{k \cup \mathcal{S}_2 } p(x_k ;x_{\mathcal{S}_2})$. Therefore, (ii) holds.
\end{proof}

\begin{proof} (of Lemma \ref{lem: ttns-markov-exact})
We will verify that Condition \ref{cond: Condition for sketch functions} holds. The Markov sketch function is quite special, and we often refer to a concept of natural identification. To make this concept rigorous, if two matrices \(A(x;y)\) and \(A'(z;w)\) are said to have a \emph{natural identification}, it then means that \(A(x;y) = A'(z;w)\) entry-wise as matrices. In particular, if one has a natural identification \(A(x;y) = A'(z;w)\), then \(A\) and \(A'\) share column space, row space, and rank.

By the property of right sketch function in Markov sketch function, one has the natural identification $\bar{\Phi}^\star_k(x_{\leftside(k) \cup k}; \gamma_{(k, \parent(k))}) = \mathcal{M}_{\leftside(k) \cup k \cup \parent(k)}p^\star(x_{\leftside(k) \cup k}; x_{\parent(k)})$. Lemma \ref{lem:markov-col/row-spaces} then shows that the column space of $\bar{\Phi}^\star_k(x_{\leftside(k) \cup k}; \gamma_{(k, \parent(k))})$ equals to that of $p^\star(x_{\leftside(k) \cup k}; x_{\rightside(k)})$. By Condition \ref{cond: TTNS ansatz condition}, the column space of $p^\star(x_{\leftside(k) \cup k}; x_{\rightside(k)})$ equals to that of any $\Phi^{\Delta}_{(w, k)}(x_{\leftside(k) \cup k}; \alpha_{(k, \parent(k))})$, and so (i) holds.

Similarly, due to Markov sketch function, one has the natural identification $$Z_{k}^\star(\beta_{(k,\child(k))}, x_{k}; \gamma_{(k, \parent(k))}) = \mathcal{M}_{\child(k) \cup k \cup \parent(k)}p(x_{\child(k) \cup k}; x_{\parent(k)}).$$

By Lemma \ref{lem:markov-col/row-spaces} and the natural identification of \(Z_{k}^\star\), for any every non-leaf and non-root $k$, it follows that $Z_{k}^\star$ has the same row space as that of $$\bar{\Phi}^\star_k(x_{\leftside(k) \cup k}; \gamma_{(k, \parent(k))}) = \mathcal{M}_{\leftside(k) \cup k \cup \parent(k)}p^\star(x_{\leftside(k) \cup k}; x_{\parent(k)}).$$

Hence, the rank of \(Z_{k}^\star\) equals to the rank of $\bar{\Phi}^\star_k$. By Lemma \ref{lem:markov-col/row-spaces}, the column space of $\bar{\Phi}^\star_k$ equals to the column space of $p^\star(x_{\leftside(k) \cup k}; x_{\rightside(k)})$. Thus, the rank of \(Z_{k}^\star\) equals to \(r_{(k, \parent(k))}\) and (ii) holds. 

Because (i) and (ii) hold, the proof of Theorem \ref{thm:ttns-algorithm-recovery} actually shows that there exists a gauge $\{\Phi^{\star}_{(w, k)}\}_{(w, k) \in E}$ which satisfies Condition \ref{cond: ttns gauge choice}. To verify (iii), it suffices to check (iii) for the gauge $\{\Phi^{\star}_{(w, k)}\}_{(w, k) \in E}$, because \(A^{\star}_{k}\) having full column rank leads to any \(A^{\Delta}_k\) having full column rank. 

Moreover, it suffices to show that each \(A^{\star}_{w \to k}(\beta_{(w, k)}; \alpha_{(w, k)})\)
has full column rank of $r_{w \to k}$. If this holds, then it follows that \(A^{\star}_k = \bigotimes_{w \in \child(k)} A^{\star}_{w \to k}\) has full column rank of $\prod_{w \in \child(k)} r_{w \to k}$. 
By the SVD step in \textsc{SystemForming}, recall that the column space of $Q_{w}^\star(\gamma_{(w,k)}; \alpha_{(w,k)})$ is the same as the column space of $\left(Z^\star_{w}\right)^\top(\gamma_{(w, k)}; \beta_{(w,\child(w))}, x_{w})$. By the natural identification of
\[\left(Z^\star_{w}\right)^\top(\gamma_{(w, k)};x_{w},  \beta_{(w,\child(w))}) = \mathcal{M}_{k \cup w \cup \child(w)}p(x_{k}; x_{w \cup \child(w)}),\]
we know that the column space of $Q_{w}^\star(\gamma_{(w,k)}; \alpha_{(w,k)})$ is the same as that of $\mathcal{M}_{k \cup w \cup \child(w)}p(x_{k}; x_{w \cup \child(w)})$. By Lemma \ref{lem:markov-col/row-spaces}, it then follows that the column space of $Q_{w}^\star(\gamma_{(w,k)}; \alpha_{(w,k)})$ coincides with that of $\mathcal{M}_{k \cup w}p(x_k; x_w)$.

Moreover, \(Z^{\star}_{w \to k}\) has the natural identification \(Z^{\star}_{w \to k}(\beta_{(w,k)}; \gamma_{(w, k)}) = \mathcal{M}_{w \cup k}p(x_w; x_k)\), and so the column space of \(Z^{\star}_{w \to k}(\beta_{(w,k)}; \gamma_{(w, k)})\) coincides with that of \( \mathcal{M}_{w \cup k}p(x_w; x_k)\).

By \eqref{eqn: A w to k and phi w to k def}, one has 
\begin{equation*}
    A^{\star}_{w \to k}(\beta_{(w,k)}; \alpha_{(w, k)}) = Z^{\star}_{w \to k}(\beta_{(w,k)}; \gamma_{(w, k)})Q^{\star}_{w}(\gamma_{(w, k)}; \alpha_{(w, k)}),
\end{equation*}
and so the column space of \(A^{\star}_{w \to k}\) coincides with that of \[\mathcal{M}_{w \cup k}p(x_w; x_k)\mathcal{M}_{k \cup w}p(x_k; x_w) = \mathcal{M}_{w \cup k}p(x_w; x_k)\left(\mathcal{M}_{w \cup k}p(x_w; x_k)\right)^\top.\]

Thus, the rank of \(A^{\star}_{w \to k}\) coincides with that of \(\mathcal{M}_{w \cup k}p(x_w; x_k)\left(\mathcal{M}_{w \cup k}p(x_w; x_k)\right)^\top\), which in turn coincides with the rank of \(\mathcal{M}_{w \cup k}p(x_w; x_k)\). By applying Lemma \ref{lem:markov-col/row-spaces}, the rank of \(\mathcal{M}_{w \cup k}p(x_w; x_k)\) equals to \(r_{(w,k)}\), and so (iii) holds.

\end{proof}

\section{Proof of Theorem \ref{thm: perturbative sketching form}}
After applying the left and right sketching, one has the following form on \(Z^\star_{k}\):
\begin{equation}\label{eq: tildephi in perturbative sketching}
      Z^\star_{k}(x_{k}, \beta_{(k, \neighbor(k))}) = \sum_{\substack{\beta_{e}\\ k \not \in e}} \sum_{\substack{x_{i} \\ i \not = k}}p^\star(x_{1}, \ldots, x_{d})\prod_{i \not = k}s_{i}(x_{i}, \beta_{(i, \neighbor(i))}).
\end{equation}

Let \(\mathcal{S}_{k}  :=  [d] - \{k\}\), and let \(T_{\mathcal{S}_{k}}\) be the subgraph of \(T\) with vertex set being \(\mathcal{S}_{k}\). Using the definition of subgraph TTNS function in Definition \ref{defn: TTNS on subgraphs}, define a tensor \[H_{k} \colon \prod_{i \in [d], i \not = k}[n_{i}] \times \prod_{w \in \neighbor(k)}[\beta_{(w, k)}] \to \bR\] as the subgraph TTNS function over \(\{s_{i}\}_{i \not = k}\) and \(T_{\mathcal{S}_{k}}\), i.e.
\begin{equation}
\label{eqn: definition of perturbative projection operator}
    H_{k}(x_{[d] - \{k\}}, \beta_{(k, \neighbor(k))}) = \sum_{\substack{\alpha_{e} \\ k \not \in e}} \prod_{i \not = k} s_{i}\left(x_{i}, \beta_{(i, \neighbor(i))}\right).
\end{equation}

Then \eqref{eq: tildephi in perturbative sketching} is equivalent to the following equation:
\begin{equation}\label{eq: Z_k^star in perturbative sketching, alterante form}
      Z^\star_{k}(x_{k}, \beta_{(k, \neighbor(k))}) =  \sum_{\substack{x_{w} \\ w \not = k}}p^\star(x_{1}, \ldots, x_{d})H_{k}(x_{[d] - \{k\}}, \beta_{(k, \neighbor(k))}).
\end{equation}

From \eqref{eqn: definition of perturbative projection operator}, one sees that \(H_{k}\) is multi-linear in \(\{s_{i}\}_{i \not = k}\). We thus can apply the binomial theorem to derive a structural form on \(H_{k}\) as a sum of secondary terms. To do so, let \(\mathcal{S}\) be an arbitrary subset of \(\mathcal{S}_{k}\), and define a tensor \[H_{k; \mathcal{S}} \colon \prod_{i \in [d], i \not = k}[n_{i}] \times \prod_{w \in \neighbor(k)}[\beta_{(w, k)}] \to \bR\] as the subgraph TTNS function over \(T_{\mathcal{S}_{k}}\) and \(\{\Delta_{i}\}_{i \in \mathcal{S}_{k}} \cup \{O_{j}\}_{j \in \mathcal{S}_{k} - \mathcal{S}}\), i.e.
\begin{equation*}
    H_{k; \mathcal{S}}(x_{[d] - \{k\}}, \beta_{(k, \neighbor(k))}) = \sum_{\substack{\alpha_{e} \\ k \not \in e}} \prod_{i \in \mathcal{S}} \Delta_{i}\left(x_{i}, \beta_{(i, \neighbor(i))}\right)\prod_{j \in \mathcal{S}_{k} - \mathcal{S}}  O_{j}\left(x_{j}, \beta_{(j, \neighbor(j))}\right).
\end{equation*}

We now use the fact that \(O_{j}\left(x_{j}, \beta_{(j, \neighbor(j))}\right) = 1\) in Condition \ref{cond: perturbative sketching cores}, and so 
\begin{equation}
\label{eqn: part of perturbative projection operator}
    H_{k; \mathcal{S}}(x_{[d] - \{k\}}, \beta_{(k, \neighbor(k))}) = \sum_{\substack{\alpha_{e} \\ k \not \in e}} \prod_{i \in \mathcal{S}} \Delta_{i}\left(x_{i}, \beta_{(i, \neighbor(i))}\right) = \sum_{\beta_{e}, k \not \in e}\Delta_{\mathcal{S}}(x_{\mathcal{S}}, \beta_{\partial \mathcal{S}}),
\end{equation}
where the second equality follows from the Definition of \(\Delta_{\mathcal{S}}\) in \eqref{eqn: def of Delta S}.

By applying the binomial theorem over the fact that \(s_{i} = \epsilon\Delta_{i} + O_{i}\), one sees that \(H_{k}\) is a sum of \(2^{d - 1}\) terms, each of which formed by corresponding to one \(H_{k; \mathcal{S}}\), i.e.
\begin{equation}\label{eq: power series for perturbative projection}
    H_{k}(x_{[d] - \{k\}}, \beta_{(k, \neighbor(k))}) = \sum_{l = 0}^{d-1} \epsilon^{l} \sum_{\mathcal{S} \subset [d] - \{k\}, |\mathcal{S}| = l}H_{k; \mathcal{S}}(x_{[d] - \{k\}}, \beta_{(k, \neighbor(k))}).
\end{equation}

Define \(Z^{\star}_{k; \mathcal{S}}\) as the following tensor:
\begin{equation}
    \label{eqn: alternate def of Z^star_S}
    Z^\star_{k}(x_{k}, \beta_{(k, \neighbor(k))})  :=   \sum_{\substack{x_{w} \\ w \not = k}}p^\star(x_{1}, \ldots, x_{d})H_{k; \mathcal{S}}(x_{[d] - \{k\}}, \beta_{(k, \neighbor(k))}).
\end{equation}

The proof that \(Z^{\star}_{k; \mathcal{S}}\) satisfies \eqref{eq: structural form of  Z star kS} is a simple result of exchanging summation order:
\begin{align*}
     &\sum_{\substack{x_{w} \\ w \not = k}}p^\star(x_{1}, \ldots, x_{d})H_{k; \mathcal{S}}(x_{[d] - \{k\}}, \beta_{(k, \neighbor(k))})\\ 
    = &\sum_{\substack{x_{w} \\ w \in \mathcal{S}}} \left(\sum_{\substack{x_{w} \\ w \in  \mathcal{S}_{k} - \mathcal{S}}}p^{\star}(x_{1}, \ldots, x_{d})\right)\left(\sum_{\beta_{e}, k \not \in e}\Delta_{\mathcal{S}}(x_{\mathcal{S}}, \beta_{\partial \mathcal{S}})\right)\\
    = &\sum_{\beta_{e}, k \not \in e}\left(\sum_{x_{\mathcal{S}}}\mathcal{M}_{\mathcal{S} \cup \{k\}}  p^{\star}(x_{k},x_{\mathcal{S}})\Delta_{\mathcal{S}}(x_{\mathcal{S}}, \beta_{\partial \mathcal{S}})\right).
\end{align*}

Due to the linear relationship between \(H_k\) and \(Z^{\star}_{k}\) in \eqref{eq: Z_k^star in perturbative sketching, alterante form}, it follows that the structural form of \(H_{k}\) in \eqref{eq: power series for perturbative projection} leads to the structural form for \(Z^{\star}_{k}\) in \eqref{eq: structural form}, as desired.

\end{document}